\documentclass[10pt]{article} 
\usepackage[preprint]{tmlr}

\usepackage{amsmath,amsfonts,bm}

\def\eqref#1{equation~\ref{#1}}

\def\1{\bm{1}}

\DeclareMathAlphabet{\mathsfit}{\encodingdefault}{\sfdefault}{m}{sl}
\SetMathAlphabet{\mathsfit}{bold}{\encodingdefault}{\sfdefault}{bx}{n}

\usepackage{hyperref}
\usepackage{url}
\usepackage{xspace}
\usepackage{subcaption}
\usepackage{graphicx}
\usepackage{float}
\usepackage{booktabs}
\usepackage{algorithm}
\usepackage{algpseudocode}

\usepackage{diagbox}
\usepackage{wrapfig}

\usepackage{pifont}
\newcommand{\cmark}{\ding{51}}
\newcommand{\xmark}{\ding{55}}

\usepackage{listings}
\usepackage{multicol}
\usepackage{multirow}
\usepackage{enumitem}
\usepackage{amsthm}
\usepackage{amsmath, amssymb}

\newcommand{\ie}{\textit{i}.\textit{e}.}

\newcommand{\etc}{\textit{etc.}}
\newcommand{\cae}{CAE\xspace}
\newcommand{\caep}{CAE$+$\xspace}

\theoremstyle{plain}
\newtheorem{theorem}{Theorem}

\newtheorem{lemma}{Lemma}

\newtheorem{definition}{Definition}

\newtheorem{assumption}{Assumption}

\title{\cae: Repurposing the Critic as an Explorer in Deep Reinforcement Learning}

\author{\name Yexin Li \email liyexin@bigai.ai \\
      State Key Laboratory of General Artificial Intelligence, BIGAI}

\def\month{02}  
\def\year{2026} 
\def\openreview{\url{https://openreview.net/forum?id=54MOD02xC2}} 

\begin{document}

\maketitle

\begin{abstract}
Exploration remains a fundamental challenge in reinforcement learning, as many existing methods either lack theoretical guarantees or fall short in practical effectiveness. In this paper, we propose \cae, \ie, the Critic as an Explorer, a lightweight approach that repurposes the value networks in standard deep RL algorithms to drive exploration, without introducing additional parameters. \cae leverages multi-armed bandit techniques combined with a tailored scaling strategy, enabling efficient exploration with provable sub-linear regret bounds and strong empirical stability. Remarkably, it is simple to implement, requiring only about 10 lines of code. For complex tasks where learning reliable value networks is difficult, we introduce \caep, an extension of \cae that incorporates an auxiliary network. \caep increases the parameter count by less than 1\% while preserving implementation simplicity, adding roughly 10 additional lines of code. Extensive experiments on MuJoCo, MiniHack, and Habitat validate the effectiveness of \cae and \caep, highlighting their ability to unify theoretical rigor with practical efficiency.
\end{abstract}

\section{Introduction}
\label{introduction}

Exploration in reinforcement learning (RL) remains a fundamental challenge, particularly in sparse-reward environments. Although algorithms such as DQN~\citep{DQN}, PPO~\citep{PPO}, SAC~\citep{SAC}, DDPG~\citep{DDPG}, TD3~\citep{TD3}, IMPALA~\citep{IMPALA}, and DSAC~\citep{DSAC, DSAC-T} have achieved impressive performance on tasks like Atari games~\citep{atari_dqn, DQN}, StarCraft~\citep{StarCraft}, Go~\citep{Go}, \etc, they often depend on rudimentary exploration strategies. Common approaches, such as $\epsilon$-greedy or injecting noise into actions, are often sample-inefficient and struggle in environments with delayed or sparse rewards.

For decades, exploration strategies with provable optimality guarantees have been well established in tabular RL settings~\citep{tabular_rl}. More recently, methods with provable regret bounds have been progressively extended to RL with function approximation, including linear functions~\citep{provable_random_value_function, provable_random_value_function_, provable_q_learning, provable_linear_function, bayesian_q_net, PC-PG}, kernel-based models~\citep{provable_kernel_nn}, and neural networks~\citep{provable_kernel_nn}. While linear and kernel-based approaches make strong assumptions about the RL functions, the provable method based on neural networks suffers from prohibitive computational costs, specifically $O(n^{3})$, where $n$ is the number of parameters in the network. Subsequently, some studies~\citep{ACB, LMCDQN} propose algorithms with theoretical optimality guarantees under the linearity assumption and extend them directly to deep RL without further proof. Other works~\citep{LSVI-PHE, General_LMCDQN} derive provable regret bounds for their deep RL methods, but these approaches are either practically burdensome or rely on unknown sampling errors.

More practical exploration approaches typically rely on heuristic mechanisms, giving rise to a variety of empirically successful methods, including Pseudocount~\citep{pseudocount}, ICM~\citep{ICM}, RND~\citep{RND}, RIDE~\citep{RIDE}, NovelD~\citep{NovelD}, AGAC~\citep{AGAC}, and E3B~\citep{E3B, E3B_study}. These methods introduce intrinsic bonuses to encourage agents to visit novel states. For instance, RND computes the exploration bonus using the prediction error of a randomly initialized target network. Despite their empirical effectiveness, such approaches often distort the extrinsic reward signal and generally lack formal theoretical guarantees. In contrast, potential-based reward shaping methods, such as Liberty~\citep{Liberty} and EME~\citep{EME}, provide stronger theoretical justification but are often challenging to implement in practice. Moreover, all of the above methods require training auxiliary networks beyond the standard policy and value networks used in deep RL, resulting in substantially increased computational overhead.

\begin{figure}[t]
    \centering
    \includegraphics[width=1.0\textwidth]{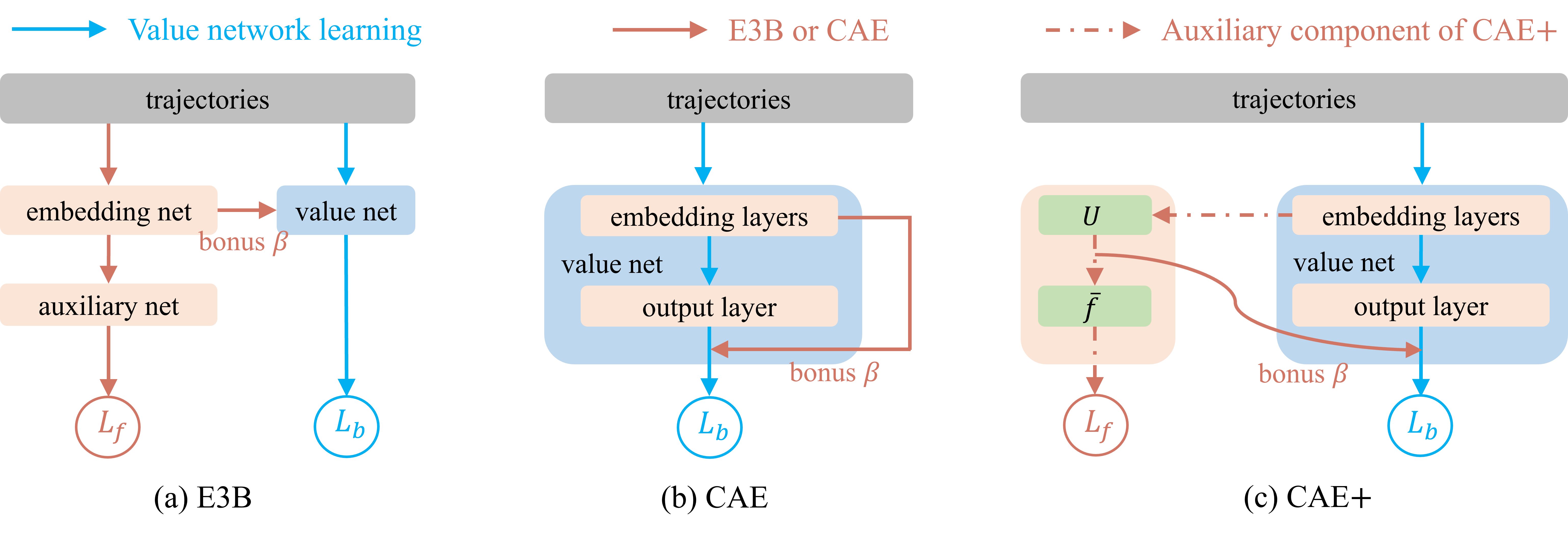}
    \caption{Comparison of baselines, such as E3B, with \cae and \caep. $L_{b}$ denotes the Bellman loss used to update the value function, while $L_{f}$ refers to the loss of the auxiliary network. Unlike E3B, which requires a separate network to compute exploration bonuses, \cae exploits the embedding layers of the value network for this purpose. \caep further enhances \cae by incorporating a compact, lightweight auxiliary network, $f = \bar{f} \circ \boldsymbol{U}$, improving performance in sparse-reward environments with only a minor increase in parameters.}
    \label{fig:comparison}
    \vspace{-10pt}
\end{figure}

In this work, we aim to combine the strengths of both theoretically grounded and empirically effective exploration methods. Provably efficient approaches are fundamentally rooted in the theory of Multi-Armed Bandits (MAB)~\citep{LinUCB, LinUCB_proof, LinTS, ComLinUCB, NeuralTS, NeuralUCB}. Building on this foundation, we hypothesize that advanced techniques from \textbf{neural MAB} can be effectively adapted for exploration in \textbf{deep RL}. Recent studies~\citep{deep_neural_lin_bandit, Neural-LinTS, shallow_exploration} suggest that decoupling deep representation learning from exploration strategies holds promise for achieving efficient exploration in neural MAB settings.

Motivated by these insights, we propose \cae. Unlike existing methods that train additional embedding networks to generate exploration bonuses, \cae leverages the embedding layers of the value networks in the RL algorithms and employs MAB techniques to produce exploration bonuses. To ensure the practical stability of \cae, we adopt an appropriate scaling strategy~\citep{sample-mean-var, streaming-rl} to process the bonuses. Consequently, \cae introduces no additional parameters beyond those in the original algorithms, showcasing that RL algorithms inherently possess strong exploration capabilities if their learned networks are effectively leveraged. Moreover, \cae is simple to implement, requiring only about 10 lines of code. A comparison between \cae and existing methods is in Figure~\ref{fig:comparison}.

For tasks with complex dynamics and very sparse rewards, learning value networks is challenging, impeding exploration based on them. Accordingly, we propose an extended version \caep, as illustrated in Figure~\ref{fig:comparison}. \caep integrates a compact, lightweight auxiliary network to facilitate the learning process. The structure of the auxiliary network is carefully designed to prevent severe coupling between the environment dynamics and the returns, thereby further enhancing the performance of \caep. Remarkably, this addition increases the parameter count by less than 1\% and requires only about 10 extra lines of code, thus preserving the simplicity and lightweight nature.

Our experiments cover a diverse set of benchmarks, \ie, MuJoCo, MiniHack, and Habitat, representing dense-reward, sparse-reward, and reward-free environments. \cae improves the performance of state-of-the-art RL baselines, including PPO~\citep{PPO}, SAC~\citep{SAC}, TD3~\citep{TD3}, and DSAC~\citep{DSAC, DSAC-T}. Additionally, \caep demonstrates robust performance, consistently outperforming E3B~\cite{E3B, E3B_study}, the state-of-the-art exploration method for MiniHack and Habitat, across all evaluated tasks. These results highlight the superior reliability and effectiveness of \cae and \caep in diverse RL scenarios.

In summary, we make three key contributions. First, we propose lightweight \cae and \caep, which enable the use of linear MAB techniques for exploration in deep RL. By adopting a scaling strategy and carefully designing the small auxiliary network for \caep, we ensure both practical stability and functionality in environments with dense and sparse rewards. Second, our theoretical analysis demonstrates that any deep RL algorithm with \cae or \caep achieves a sub-linear regret bound over episodes. Finally, experiments on MuJoCo, MiniHack, and Habitat validate the effectiveness of \cae and \caep, showcasing their superior performance. Our code is available at 
\href{https://github.com/liyexn/Critic-as-an-Explorer}{\textcolor{magenta}{https://github.com/liyexn/Critic-as-an-Explorer}}.

\section{Related Work}
\label{sec:related_work}

\paragraph{Multi-armed bandits.} MAB algorithms address the exploration-exploitation dilemma by making sequential decisions under uncertainty. LinUCB~\citep{LinUCB} assumes a linear relationship between arm contexts and rewards, enabling efficient exploration via uncertainty quantification in the estimated parameter space and ensuring a sub-linear regret bound~\citep{LinUCB_proof}. 
To relax the linearity assumption, KernelUCB~\citep{KernelUCB, GP_UCB} and NegUCB~\citep{NegUCB} transform contexts into high-dimensional spaces and apply LinUCB to the mapped contexts. Neural-UCB~\citep{NeuralUCB} and Neural-TS~\citep{NeuralTS} leverage neural networks to model the complex relationships between contexts and rewards. However, their computational complexity of $O(n^{3})$, where $n$ denotes the number of network parameters, limits their scalability in real-world applications. Neural-LinTS~\citep{Neural-LinTS} and Neural-LinUCB~\citep{shallow_exploration} mitigate this limitation by decoupling representation learning from exploration strategies, improving the practicality of neural MAB.

\paragraph{Provable exploration in RL.} Provably efficient exploration methods~\citep{tabular_rl, provable_random_value_function, provable_random_value_function_, provable_q_learning, provable_linear_function, PC-PG, OPPO, Bayes-UCBVI, RandQL} often face empirical limitations or remain primarily theoretical, with limited applicability to deep RL. Some studies~\citep{MR-NaS, bayesian_q_net, ACB, LMCDQN} propose methods with theoretical guarantees under tabular or linearity assumptions and extend them directly to deep RL settings. Other works~\citep{LSVI-PHE, General_LMCDQN} provide provable bounds for deep RL, but these approaches may suffer from any of the following limitations: computationally prohibitive, practically burdensome, or reliance on unknown sampling errors. A comparative analysis is presented in Table~\ref{tab:provable_algorithm_comparison}.

\begin{table}[t]
\caption{Method comparison. \textbf{Linear function} indicates whether the theoretical guarantees apply to RL algorithms with linear function approximation. 
\textbf{Network-based function} specifies whether the guarantees extend to RL algorithms using network-based function approximation. \textbf{Toy task} denotes whether the method can be effectively applied to simple deep RL tasks. \textbf{Practical complex task} indicates whether the method remains practically feasible when deployed in realistic, large-scale deep RL environments.}
\label{tab:provable_algorithm_comparison}
\begin{center}
    \scalebox{0.85}{
    \begin{tabular}{l|cc|cc}
        \toprule
        \multirow{2}{*}{Method} & \multicolumn{2}{c|}{Provable} & \multicolumn{2}{c}{Empirical with Deep RL} \\
        \noalign{\vskip 2pt}
        \cline{2-5}
        \noalign{\vskip 2pt}
        & Linear function & Network-based function & Toy task & Practical complex task \\
        \midrule
        LSVI-UCB~\citep{provable_linear_function} & \cmark & \xmark & \xmark & \xmark \\
        NN-UCB~\citep{provable_kernel_nn} & \cmark & \cmark & \xmark & \xmark \\
        OPT-RLSVI~\citep{OPT-RLSVI} & \cmark & \xmark & \xmark & \xmark \\
        LSVI-PHE~\citep{LSVI-PHE} & \cmark & \cmark & \cmark & \xmark \\
        BDQN~\citep{bayesian_q_net}  & \cmark & \xmark & \cmark & \cmark \\
        ACB~\citep{ACB} & \cmark & \xmark & \cmark & \cmark \\
        LMCDQN~\citep{LMCDQN}  & \cmark & \xmark & \cmark & \cmark \\
        \midrule
        \cae & \cmark & \cmark & \cmark & \cmark \\
        \caep & \cmark & \cmark & \cmark & \cmark \\
        \bottomrule
    \end{tabular}
    }

\end{center}
\end{table}

\paragraph{Practical exploration in deep RL.} Empirically successful methods~\citep{pseudocount, ICM, study_curiosity, RIDE, RND, NovelD, AGAC, E3B, E3B_study, curiosity_in_hindsight, Automatic_ir} typically generate exploration bonuses to encourage agents to visit novel states. However, these approaches often lack rigorous theoretical foundations. In contrast, methods inspired by potential-based reward shaping~\citep{PBRS}, such as Liberty~\citep{Liberty} and EME~\citep{EME}, offer stronger theoretical grounding, but are challenging to implement in practice. Moreover, both classes of methods generally require training a substantial number of additional parameters. In comparison, \cae and \caep utilize MAB techniques, assisted by embedding layers within the RL value networks, providing empirical benefits with minimal additional parameters. Figure~\ref{fig:comparison} illustrates the differences among various exploration methods, while Table~\ref{tab:empirical_algorithm_comparison} summarizes their additional networks and parameters.

\begin{table}[H]
\caption{Comparison of exploration methods on MiniHack. \textbf{Networks}: additional networks beyond those in the base RL algorithm IMPALA~\citep{IMPALA}, which contains $25,466,652$ parameters; \textbf{\# Params}: the number of additional parameters introduced by the exploration method. Networks in \textbf{bold} represent those with significant parameters, while those in \textcolor{gray}{gray} indicate substantially fewer parameters.}
\label{tab:empirical_algorithm_comparison}
\begin{center}
\resizebox{\textwidth}{!}{
\begin{tabular}{lccr}
        \toprule
        Method & Networks & \# Params & Params $\uparrow$ \\
        \midrule
        ICM~\citep{ICM} & \textbf{Embedding} \textcolor{gray}{$+$ Forward dynamics $+$ Inverse dynamics} & \textbf{$16, 074, 512$} \textcolor{gray}{$ + 2, 110, 464 + 527, 371$} & $73\%$ \\
        RND~\citep{RND} & \textbf{Embedding} & \textbf{$16, 074, 512$} & $63\%$ \\
        RIDE~\citep{RND} & \textbf{Embedding} \textcolor{gray}{$+$ Forward dynamics $+$ Inverse dynamics} & \textbf{$16, 074, 512$} \textcolor{gray}{$ + 2, 110, 464 + 527, 371$} & $73\%$ \\
        NovelD~\citep{NovelD} & \textbf{Embedding} & \textbf{$16, 074, 512$} & $63\%$ \\
        E3B~\citep{E3B, E3B_study} & \textbf{Embedding} \textcolor{gray}{$+$ Inverse dynamics} & \textbf{$16, 074, 512$} \textcolor{gray}{$ + 527, 371$} & $65\%$ \\

        \midrule
        \cae & \textcolor{gray}{-} & \textcolor{gray}{-} & $0$ \\
        \caep & \textcolor{gray}{Inverse dynamics} & \textcolor{gray}{$199, 819$} & $0.8\%$ \\
        
        \bottomrule
    \end{tabular}
}
\end{center}
\end{table}
\vspace{-10pt}

\section{Methodology}
\label{sec:methodology}

Unless otherwise specified, bold uppercase symbols denote matrices, while bold lowercase symbols represent vectors. $\boldsymbol{I}$ refers to an identity matrix. Frobenius norm for matrices and $l_{2}$ norm for vectors are both denoted by $\left \| \cdot \right \|_{2}$. Mahalanobis norm of a vector $\boldsymbol{x}$ with respect to $\boldsymbol{A}$ is defined as $\left \| \boldsymbol{x} \right \|_{\boldsymbol{A}} = \sqrt{\boldsymbol{x}^{\mathsf{T}} \boldsymbol{A} \boldsymbol{x}}$. For any integer $K > 0$, the set of integers $\{ 1, 2, ..., K \}$ is denoted by $\left[ K \right]$.

\subsection{Preliminary}
\label{subsec:preliminary}

An episodic Markov Decision Process (MDP) is formally defined as a tuple $(\mathcal{S}, \mathcal{A}, H, \mathbb{P}, r)$, where $\mathcal{S}$ denotes the state space and $\mathcal{A}$ is the action space. Integer $H > 0$ indicates the duration of each episode. Functions $\mathbb{P}: \mathcal{S} \times \mathcal{A} \times \mathcal{S} \rightarrow [0, 1]$ and $r: \mathcal{S} \times \mathcal{A} \rightarrow [0, 1]$ are the Markov transition and reward functions, respectively. During an episode, the agent follows a policy $\pi: \mathcal{S} \times \mathcal{A} \rightarrow [0, 1]$. At each time step $h \in \left[ H \right]$ in the episode, the agent observes the current state $s_{h} \in \mathcal{S}$ and selects an action $a_{h} \sim \pi (\cdot | s_{h} )$ to execute, then the environment transits to the next state $s_{h+1} \sim \mathbb{P}(\cdot|s_{h}, a_{h} )$, yielding an immediate reward $r_{h} = r(s_{h}, a_{h})$. At time step $h$, the action-value function $Q (s_{h}, a_{h})$ measures the cumulative return obtained by taking action $a_{h}$ in state $s_{h}$ and subsequently following policy $\pi$:

\begin{equation}
    Q (s_{h}, a_{h}) = \sum_{t=h}^{H} \gamma^{t-h} r_{t}
\end{equation}
where $0 \leq \gamma \leq 1$ is the discount parameter.

Many algorithms have been developed to learn the optimal policy $\pi^{*}$ for the agent to select and execute actions at each time step $h$ in the episode, thus ultimately maximizing the cumulative return $\sum_{h=1}^{H} \gamma^{h-1} r_{h}$. Notable algorithms include DQN~\citep{DQN}, PPO~\citep{PPO}, SAC~\citep{SAC}, TD3~\citep{TD3}, IMPALA~\citep{IMPALA}, DSAC~\citep{DSAC, DSAC-T}, and others. A common component of these algorithms is a neural network that parameterizes the action-value function\footnote{In some algorithms, the state-value function is learned instead of the action-value function. However, this does not affect the implementation and conclusion of our method, as will be seen in Subsection~\ref{subsec:cae}.} $Q(\cdot, \cdot)$ under a specific policy as Equation~\ref{eq:q_fun}, where $\phi(\cdot, \cdot|\boldsymbol{W})$ is the embedding layers, $\boldsymbol{\theta}$ and $\boldsymbol{W}$ are trainable parameters of the network.

\begin{equation}
    Q (s, a) = \boldsymbol{\theta}^{\mathsf{T}} \phi (s, a|\boldsymbol{W})
\label{eq:q_fun}
\end{equation}

Bellman loss as Equation~\ref{eq:bellman} is often employed to learn the action-value function. Using the most recent action-value function, the policy can be updated in various ways, depending on the specific algorithm. Since \cae focuses on leveraging Equation~\ref{eq:q_fun} for efficient exploration while preserving the core techniques of existing RL algorithms, we introduce \cae within the context of DQN for simplicity. However, it can be easily adapted to other RL algorithms.

\begin{equation}
    L_{B} = \left( Q (s_h, a_h) - \mathbb{E}_{s_{h+1} \sim \mathbb{P} (\cdot|s_h,a_{h})} \left[ r_{h} + \gamma \cdot \max_{a_{h+1}} Q (s_{h+1}, a_{h+1}) \right] \right)^{2}
\label{eq:bellman}
\end{equation}

\subsection{\cae: the Critic as an Explorer}
\label{subsec:cae}
For a state-action pair $(s, a)$, the estimated action-value $Q(s, a)$ is subject to an uncertainty term $\beta (s, a)$, arising from the novelty or limited experience with the particular state-action pair. Similar to MAB problems, explicitly accounting for such uncertainty is crucial. Incorporating this uncertainty term promotes exploration and can lead to improved long-term returns. Accordingly, we adjust the action-value function with an uncertainty bonus, as shown in Equation~\ref{eq:q_value_update}, where $\alpha \geq 0$ denotes the exploration coefficient. In the literature, \textit{uncertainty} is often also referred to as a \textit{bonus}, and we use the two terms interchangeably when no ambiguity arises.

\begin{equation}
\label{eq:q_value_update}
    Q (s, a) = \boldsymbol{\theta}^{\mathsf{T}} \phi(s, a|\boldsymbol{W}) + \alpha \beta (s, a)
\end{equation}

However, defining $\beta (s, a)$ remains challenging. Provably efficient methods often attempt to address this by either assuming a linear value function~\citep{provable_q_learning} or requiring $O(n^{3})$ computation time~\citep{provable_kernel_nn} in terms of the number of parameters $n$ in the value network. Both of these approaches have drawbacks, \ie, linearity fails to capture the complexity of tasks while the $O(n^{3})$ computational cost is impractical.

To overcome these limitations, we draw inspiration from Neural-LinUCB~\citep{shallow_exploration} and Neural-LinTS~\citep{Neural-LinTS}, which effectively decouple representation learning from exploration strategies. Building on this idea and following the standard value network structure in Equation~\ref{eq:q_fun}, \cae decomposes the action-value function into two distinct components.

\begin{itemize}[leftmargin=2em]
    \item Network $\phi (s, a|\boldsymbol{W})$ extracts the embedding of the state-action pair $(s, a)$;
    \item $Q (s, a) = \boldsymbol{\theta}^{\mathsf{T}} \phi (s, a|\boldsymbol{W} )$ is a linear function of the embedding $\phi (s, a|\boldsymbol{W} )$ with parameter $\boldsymbol{\theta}$.
\end{itemize}

Consequently, after appropriate modifications, MAB theory under the linearity assumption can be adapted to work with the embeddings $\phi(s, a)$ for $\forall s \in \mathcal{S}$ and $\forall a \in \mathcal{A}$. Simultaneously, the action-value function retains its representational capacity through the embedding layers $\phi(s, a)$, ensuring promising practical performance. While various MAB techniques can be adapted to the embeddings, we illustrate \cae using the two most representative ones. Other techniques can be utilized similarly, showcasing that \cae is a generalizable framework rather than a fixed method.

\textbf{Upper Confidence Bound (UCB)} is an optimistic exploration strategy in MAB. It defines the uncertainty term as Equation~\ref{eq:beta_ucb}, where $\phi(\cdot, \cdot)$ denotes the latest embedding layers, and $\boldsymbol{A}$ is the Gram matrix, initialized as $\boldsymbol{A} = \lambda \boldsymbol{I}$ with $\lambda$ being the ridge regularization parameter. At each step, $\boldsymbol{A}$ is updated using Equation~\ref{eq:variance_update}.

\begin{equation}
    \label{eq:beta_ucb}
    \beta (s, a) = \sqrt{\phi (s, a)^{\mathsf{T}} \boldsymbol{A}^{-1} \phi (s, a)}
\end{equation}

\vspace{-10pt}
\begin{equation}
\label{eq:variance_update}
    \boldsymbol{A} \leftarrow \boldsymbol{A} + \boldsymbol{\phi} (s, a) \boldsymbol{\phi}^{\mathsf{T}} (s, a)
\end{equation}

\textbf{Thompson Sampling} is a randomized exploration strategy that samples the value function, adjusted for uncertainty, from a posterior distribution. It defines the uncertainty term as Equation~\ref{eq:beta_thompson}, where the Gram matrix $\boldsymbol{A}$ is initialized and updated in the same manner as in UCB.

\vspace{-5pt}
\begin{equation}
\label{eq:beta_thompson}
\begin{aligned}
    \boldsymbol{\Delta \theta} \sim N(0, \boldsymbol{A}^{-1}); \quad \beta (s, a) = (\boldsymbol{\Delta \theta})^{\mathsf{T}} \phi (s, a)
\end{aligned}
\end{equation}

As the value network undergoes continuous updates, exploration based on the ever-changing embedding layers $\phi(\cdot, \cdot)$ can become highly unstable, significantly impairing practical performance. Inspired by existing scaling strategies~\citep{sample-mean-var, streaming-rl}, we adopt an appropriate one for the generated uncertainty at each time step, ensuring both stability and practical functionality, as detailed in Algorithm \ref{alg:running_mean_variance}. This scaling strategy normalizes the generated uncertainty at each time step using the running standard deviation, which, despite its simplicity, has a profound impact on the performance of \cae. The critical importance of this design is further highlighted through ablation studies presented in Appendix~\ref{app:experiment}.

\paragraph{Adapting \cae to General RL Algorithms.} Depending on the RL algorithm employed, we may sometimes learn a state-value network instead of an action-value network. In such cases, the network produces embeddings for states rather than for state-action pairs. Even when an action-value network is learned, it might still output only state embeddings if it is designed to take a state as input and produce values for multiple actions. In these scenarios, we use either the embedding of the next state or the embedding concatenated with the action as a proxy for the current state–action embedding when computing uncertainty.

\begin{algorithm}[H]
\caption{Scaling strategy for the uncertainty}
\label{alg:running_mean_variance}
\begin{algorithmic}[1]
\footnotesize
   \State {\bfseries Input:} Uncertainty $b$, running mean $\mu$, cumulative squared deviation $\nu^{2}$, and running count of samples $\mathcal{N}$
   
   \State Update the sample count $\mathcal{N} \leftarrow \mathcal{N} + 1$
   
   \vspace{2pt}
   \State Compute $\delta = b - \mu$
   \State Update the running mean $\mu \leftarrow \mu + \frac{\delta}{\mathcal{N}}$

   \vspace{2pt}
   \State Update the cumulative squared deviation $\nu^{2} \leftarrow \nu^{2} + \delta \times (b - \mu)$

   \vspace{2pt}
   \State {\bfseries Output:} Scaled uncertainty $\frac{b}{\sqrt{\nu^2 / \mathcal{N}}}$, and the updated $\mu$, $\nu^{2}$, and $\mathcal{N}$
\end{algorithmic}
\end{algorithm}

\subsection{\texorpdfstring{\caep: Enhancing \cae with Minimal Overhead}{CAE+: Enhancing CAE with Minimal Overhead}}
\label{subsec:caep}

For tasks with complex dynamics or very sparse rewards, learning value networks is particularly challenging, which in turn hinders exploration reliant on them. Accordingly, we propose \caep, an extension of \cae that incorporates a lightweight auxiliary network, introducing less than $1\%$ additional parameters. 

Specifically, we utilize an Inverse Dynamics Network (IDN)~\citep{ICM, RIDE, E3B} to enhance the learning of the embedding layers contained in the value networks. This is achieved by a compact network $f$ that infers the distribution $p(a_{h})$ over taken actions given consecutive states $s_{h}$ and $s_{h+1}$, which is trained by maximum likelihood estimation as Equation~\ref{eq:idn_loss}.

\begin{equation}
    \label{eq:idn_loss}
    L_{f} = - \log p(a_h | s_{h}, s_{h+1}) 
\end{equation}

We utilize the embedding layers as follows:
\begin{itemize}[leftmargin=2em]
    \item If the value network is an action-value network with embedding layers $\phi(s, a)$, a fixed default action $\ddot{a}$ is assigned to the action input while the true states are used; the resulting outputs of $\phi(\cdot, \ddot{a})$ are then treated as state embeddings.
    
    \item If the value network is a state-value network with embedding layers $\phi(s)$, the actual states are directly fed in to generate their embeddings. 
\end{itemize}
These state embeddings are subsequently transformed by a linear layer $\boldsymbol{U}$, followed by a small network $\bar{f}$, which processes the transformed consecutive embeddings to infer the action. Equation~\ref{eq:attn} illustrates this procedure for the case of embedding layers $\phi(s, a)$.

\vspace{-10pt}
\begin{align}
    \label{eq:attn}
    p(a_h | s_{h}, s_{h+1}) = f(\phi(s_{h}, \ddot{a}), \phi(s_{h+1}, \ddot{a})) = \bar{f}(\boldsymbol{U} \phi(s_{h}, \ddot{a}), \boldsymbol{U} \phi(s_{h+1}, \ddot{a}))
\end{align}

Although incorporating the IDN loss can accelerate learning of the embedding layers by leveraging knowledge of environment dynamics, it also introduces additional constraints on the shared embedding space, thereby limiting the flexibility of the embeddings for exploration. In \caep, to address this limitation, we adopt the transformed embeddings $\boldsymbol{U} \phi (s, a)$ instead of the original ones $\phi (s, a)$ to calculate the uncertainty. Consequently, by replacing $\phi (s, a)$ with $\boldsymbol{U} \phi (s, a)$, Equations~\ref{eq:beta_ucb} and~\ref{eq:beta_thompson}, which are used in \cae to generate uncertainty, are modified into Equation~\ref{eq:beta}, with the Gram matrix $\boldsymbol{A}$ updated according to Equation~\ref{eq:variance_update_}. Complete procedure of \caep is in Algorithm~\ref{alg:CAE+}.

\begin{align}
\label{eq:beta}
\beta (s, a) \triangleq
\begin{cases}
\sqrt{\phi^{\mathsf{T}} (s, a) \boldsymbol{U}^{\mathsf{T}} \boldsymbol{A}^{-1} \boldsymbol{U} \phi (s, a)} & \textup{UCB} \\[5pt]
\left. (\boldsymbol{\Delta \theta})^{\mathsf{T}} \boldsymbol{U} \phi (s, a) \; \right|_{\boldsymbol{\Delta \theta} \sim \mathcal{N}(0, \boldsymbol{A}^{-1})} & \text{Thompson Sampling}
\end{cases}
\end{align}

\vspace{-5pt}
\begin{align}
\label{eq:variance_update_}
\boldsymbol{A} \leftarrow \boldsymbol{A} + \boldsymbol{U} \phi(s, a) \phi^{\mathsf{T}} (s, a) \boldsymbol{U}^{\mathsf{T}}
\end{align}

In \caep, the structure of the network $f$ offers several advantages. First, the transformation $\boldsymbol{U}$ decouples environment dynamics from returns, mitigating interdependencies that could hinder flexibility and thereby enhancing empirical performance, as evidenced by the ablation studies presented later. Second, since $\boldsymbol{U}$ is a simple linear transformation, it approximately preserves the theoretical guarantees of UCB-based and Thompson Sampling-based exploration strategies, ensuring both rigor and stability in practice. Third, by projecting $\phi(s, a)$ into a lower-dimensional embedding with $\bar{d} < d$, the approach not only reduces the number of additional parameters but also lowers the computational complexity of uncertainty estimation at each time step, \ie, from $O(d^3)$ to $O(\bar{d}^3)$, making the method more efficient.

\textbf{Speed Up \caep with Rank$-1$ Update} \quad According to Algorithm~\ref{alg:CAE+}, the Gram matrix $\boldsymbol{A}$ needs to be inverted at each step, which is cubic in dimension. \textbf{Alternatively}, we can use
the Sherman-Morrison matrix identity~\citep{rank-1, E3B} to perform rank$-1$ updates of $\boldsymbol{A}^{-1}$ in quadratic time as Equation~\ref{eq:rank-1}.

\begin{equation}
\label{eq:rank-1}
    \boldsymbol{A}^{-1} \leftarrow \boldsymbol{A}^{-1} - \frac{\boldsymbol{A}^{-1} \boldsymbol{U} \boldsymbol{\phi} (s, a) \boldsymbol{\phi}^{\mathsf{T}} (s, a) \boldsymbol{U}^{\mathsf{T}} (\boldsymbol{A}^{-1})^{\mathsf{T}}}{1 + \boldsymbol{\phi}^{\mathsf{T}} (s, a) \boldsymbol{U}^{\mathsf{T}} \boldsymbol{A}^{-1} \boldsymbol{U} \boldsymbol{\phi} (s, a)}
\end{equation}

\begin{algorithm}[t]
   \caption{\caep with Action-value Network}
   \label{alg:CAE+}
\begin{algorithmic}[1]
\footnotesize
   \State {\bfseries Input:} Ridge parameter $\lambda > 0$, exploration parameter $\alpha \geq 0$, episode length $H$, episode number $M$, learning rate $\eta$

   \vspace{2pt}
   \State {\bfseries Initialize:} Gram matrix $\boldsymbol{A} = \lambda \boldsymbol{I}$, initial policy $\pi (\cdot)$ and value function $Q(\cdot, \cdot)$, network $f(\cdot|\boldsymbol{U})$, running mean $\mu=0$, cumulative squared deviation $\nu^2=0$, running count $\mathcal{N}=0$

   \vspace{2pt}
   \For{episode $m=1$ {\bfseries to} $M$}
   
   \State Receive the initial state $s_{1}^{m}$ from the environment
   
   \For{step $h=1, 2, ..., H-1$}

   \vspace{2pt}
   \State Execute action $a_{h}^{m} \sim \pi (s_{h}^{m})$, observe the next state $s_{h+1}^{m}$, and receive the immediate reward $r_{h}^{m}$
   
   \vspace{2pt}
   \State Generate bonus $\beta(s_{h}^{m}, a_{h}^{m})$ by Equation~\ref{eq:beta}
   
   \vspace{5pt}
    \State Provide $\beta(s_{h}^{m}, a_{h}^{m})$, $\mu$, $\nu^{2}$, and $\mathcal{N}$ as inputs to Algorithm~\ref{alg:running_mean_variance} to obtain scaled bonus $\beta_{h}^{m}$ and updated $\mu$, $\nu^{2}$, and $\mathcal{N}$

   \vspace{5pt}
    \State Reshape the reward $\tilde{r}_{h}^{m} = r_{h}^{m} + \alpha \beta_{h}^{m}$
    
    \vspace{2pt}
    \State Update the Gram matrix $\boldsymbol{A}$ by Equation~\ref{eq:variance_update_}

   \EndFor

   \vspace{2pt}
   \State Sample a batch $\mathcal{B} = \left\{ s_{h}^t, a_{h}^t, s_{h+1}^t, \tilde{r}_{h}^t \mid h \in [1, H-1], t \leq m \right\}$

   \vspace{2pt}
   \State Calculate the IDN loss $L_{f}$ by Equation~\ref{eq:idn_loss} and the Bellman loss $L_{B}$ by Equation~\ref{eq:bellman} on batch $\mathcal{B}$

   \vspace{2pt}
   \State Update value function $Q(\cdot, \cdot)$ and network $f$ jointly by minimizing the combined loss $\min (L_{f} + L_{B})$ with step size $\eta$

   \State Update the policy $\pi(\cdot)$ based on the latest value function $Q(\cdot, \cdot)$
   \EndFor

\end{algorithmic}
\end{algorithm}

\section{Theoretical Analysis}
\label{sec:theory}

Under the optimal policy $\pi^{*}$, let the corresponding action-value function $Q^{*}$ be structured as in Equation~\ref{eq:q_fun} and parameterized by $\boldsymbol{\theta}^{*}$ and $\boldsymbol{W}^{*}$. In Algorithm~\ref{alg:CAE+}, the policy executed in episode $m \in \left[ M \right]$ is denoted by $\pi_{m}$, and its associated action-value function is $Q^{\pi_{m}}$. Cumulative regret of Algorithm~\ref{alg:CAE+} is given in Definition~\ref{def:regret}.

\begin{definition}
\label{def:regret}
    \textbf{Cumulative Regret}. After $M$ episodes of interactions with the environment, the cumulative regret of \cae or \caep is defined as Equation~\ref{eq:regret}, where $u_{1}^{m}$ is the optimal action at state $s_{1}^{m}$ generated by policy $\pi^{*}$ while $a_{1}^{m}$ is that selected by the executed policy $\pi_{m}$.

\vspace{-5pt}
\begin{small}
    \begin{equation}
    \textup{R}_{M} = \sum_{m=1}^{M} Q^{*} (s_{1}^{m}, u_{1}^{m}) - Q^{\pi_{m}} (s_{1}^{m}, a_{1}^{m})
    \label{eq:regret}
    \end{equation}
\end{small}

\end{definition}

The cumulative regret measures the gap between the optimal return and the actual return accumulated over $M$ episodes. As discussed earlier, \cae draws inspiration from Neural-LinUCB~\citep{shallow_exploration} and Neural-LinTS~\citep{Neural-LinTS}. While Neural-LinUCB is supported by theoretical guarantees, Neural-LinTS has so far only been validated empirically. In this work, we complete the regret analysis for Neural-LinTS and subsequently derive the regret bound of \cae. Before stating Theorem~\ref{theorem:cae_regret}, we introduce the standard assumptions used in the literature on \textit{deep representation and shallow exploration}~\citep{shallow_exploration} as Assumption~\ref{ass:ass_1} and Assumption~\ref{ass:ass_3}. Assumption~\ref{ass:ass_2} is deferred to the Appendix because it requires more detailed notation.

\begin{assumption}
    \label{ass:ass_1}
    Assume that $\left\| (s; a) \right\|_{2} = 1$ for $\forall s \in \mathcal{S}, \forall a \in \mathcal{A}$. The entries of $(s;a)$ satisfy Equation~\ref{eq:ass_1}, where $j=1, 2, \ldots, \frac{D}{2}$ and $D$ represents the dimension of $(s;a)$.
    \begin{align}
        (s;a)_{j} = (s;a)_{j + \frac{D}{2}}
        \label{eq:ass_1}
    \end{align}
\end{assumption}
Note that even if the original state-action pairs $(s; a)$ do not satisfy this assumption, they can be preprocessed by augmenting them to $(s; a; s; a)$ and applying appropriate scaling to ensure the assumption holds.

\begin{assumption}
    \label{ass:ass_3}
    The neural tangent kernel $\boldsymbol{H}$ of the action-value network is positive definite.
\end{assumption}
Neural tangent kernel $\boldsymbol{H}$ is defined in accordance with a recent line of research~\citep{neural_tangent_kernel, wide_net, shallow_exploration} and is essential for the analysis of overparameterized neural networks. A detailed discussion of these assumptions is deferred to Appendix~\ref{app:assumptions}, where we show that they are mild and commonly adopted in the literature. In the following, we present the regret guarantee of \cae.

\begin{theorem}
    \label{theorem:cae_regret}
    Assume that Assumption~\ref{ass:ass_1}, Assumption~\ref{ass:ass_3}, and Assumption~\ref{ass:ass_2} hold, and that $\left \| \boldsymbol{\theta}^{*} \right \|_{2} \leq 1$ as well as $\left \| \phi(s, a) \right \|_{2} \leq 1$ for $\forall s \in \mathcal{S}, \forall a \in \mathcal{A}$. For any $\sigma \in (0, 1)$, assume that the number of parameters $\iota$ in each of the $L$ layers of $\phi(\cdot, \cdot)$ satisfies $\iota = \textup{poly} (L, D, \frac{1}{\sigma}, \log \frac{M \left| \mathcal{A} \right|}{\sigma})$, where $\left| \mathcal{A} \right|$ denotes the size of the action space and $\textup{poly} (\cdot)$ denotes a polynomial function of the listed variables. Set the exploration coefficient and step size as:

\begin{small}
    \begin{align}
    \notag
        & \alpha = \sqrt{2 (D \cdot \log (1 + \frac{M \cdot \textup{log} \left| \mathcal{A} \right|}{\lambda}) - \log \sigma)} + \sqrt{\lambda} \\ \notag
        & \eta \leq C_{1} (\iota \cdot D^{2} M^{\frac{11}{2}} L^{6} \cdot \log \frac{M \left| \mathcal{A} \right|}{\sigma})^{-1}
    \end{align}
\end{small}
    
    then with probability at least $1-\sigma$, it holds that:

    \vspace{-10pt}
    \begin{small}
    \begin{equation}
    \label{eq:regret_bound}
        \textup{R}_{M} \leq C_{2} \alpha H \sqrt{M D \textup{log} (1 + \frac{M}{\lambda D})} + C_{4} H \sqrt{MH \textup{log} \frac{1}{\sigma}} + \frac{C_{3} H L^{3} D^{\frac{5}{2}} M \sqrt{\log \iota \log \left( \frac{1}{\sigma} \right) \log \left( \frac{M \left| \mathcal{A} \right|}{\sigma} \right)} \left \| \boldsymbol{q} - \tilde{\boldsymbol{q}} \right \|_{\boldsymbol{H}^{-1}}}{\iota^{\frac{1}{6}}} \notag
    \end{equation}
    \end{small}
    
    where $C_{1}, C_{2}, C_{3}, C_{4}$ are constants; $\boldsymbol{q}$ and $\tilde{\boldsymbol{q}}$ are the target value vector and the estimated value vector of the action-value network, respectively. More discussions of these notations are in Appendix~\ref{app:long_alg} and Appendix~\ref{app:lemma}.
\end{theorem}

Specifically, we assume $\left \| \boldsymbol{\theta}^{*} \right \|_{2} \leq 1$ and $\left \| \phi(s, a) \right \|_{2} \leq 1$ to make the bound scale-free. From this theorem, we can conclude that the upper bound of the cumulative regret grows sub-linearly with the number of episodes $M$, \ie, $\widetilde{O} (\sqrt{M})$ where $\widetilde{O} (\cdot)$ hide constant and logarithmic dependence of $M$, indicating that the executed policy improves over time. Since MAB techniques are applied to the linear layer on top of the embedding layers, the cumulative regret naturally consists of two components, i.e., the exploration regret from the linear layer and the error induced by the network's estimation, which appears as the last term in the regret bound. It involves a trade-off between $M$ and $\iota$. Moreover, as the estimation error $\left \| \boldsymbol{q} - \tilde{\boldsymbol{q}} \right \|_{\boldsymbol{H}^{-1}}$ decreases over time, this term typically becomes negligible.

\section{Experiment}
\label{sec:experiment}

\paragraph{Benchmarks.} In our experiments, we evaluate \cae and \caep on MuJoCo, MiniHack, and Habitat, which correspond to dense-reward, sparse-reward, and reward-free settings, respectively.

\paragraph{Baselines.} We evaluate our approach against \textbf{nine} baselines, \ie, SAC~\citep{SAC}, PPO~\citep{PPO}, TD3~\citep{TD3}, DSAC~\citep{DSAC, DSAC-T}, ICM~\citep{ICM}, RND~\citep{RND}, RIDE~\citep{RIDE}, NovelD~\citep{NovelD}, and E3B~\citep{E3B, E3B_study}.

For \textbf{MuJoCo} tasks, we evaluate SAC, PPO, TD3, and DSAC, both with and without \cae. Notably, the other baselines are excluded from the MuJoCo experiments, as they are rarely applied to dense reward settings. Since \cae introduces no additional parameters, it is meaningful to assess whether it can improve existing RL algorithms without increasing training overhead.

For \textbf{MiniHack and Habitat} tasks, we adopt IMPALA~\citep{IMPALA} and PPO as the base RL algorithms, respectively, following the standard configurations in the open-source codebases. We compare \cae and \caep against ICM, RND, RIDE, NovelD, and E3B. Since E3B achieves state-of-the-art performance on both MiniHack and Habitat, we report only the results of E3B. For results of ICM, RND, RIDE, and NovelD, refer to the E3B paper~\citep{E3B, E3B_study}.

All the experiments are based on open-source codebases from E3B, CleanRL~\citep{cleanrl}, DSAC, and Habitat-lab~\citep{Habitat, Habitat_2, Habitat_3}. The core code and hyperparameters are provided in Appendices~\ref {app:code} and~\ref{app:experiment}, respectively. All experiments were conducted on an Ubuntu 22.04 LTS system equipped with a 13th Gen Intel Core i9-13900KF CPU and an NVIDIA RTX 4090 GPU.

\subsection{MuJoCo tasks with Dense Rewards}
\label{subsec:mujoco}

\begin{table}[H]
\caption{Experimental results after $1e6$ environment interaction steps on MuJoCo-v4 tasks, except for the \textit{Humanoid} task, which is evaluated after $4e6$ steps. \textit{RPI} represents the \textbf{R}elative \textbf{P}erformance \textbf{I}mprovement achieved by \cae. Refer to Appendix~\ref{app:experiment} for the corresponding experimental curves.}
\label{tab:mujoco_v4_results}

\begin{center}
\resizebox{\textwidth}{!}{
\begin{tabular}{l|ccc|ccc|ccc}
        \toprule
        \diagbox{Env}{Alg.} & PPO & PPO $+$ \cae & \textit{RPI \%} & TD3 & TD3 $+$ \cae & \textit{RPI \%} & SAC & SAC $+$ \cae & \textit{RPI \%} \\
        \midrule
        Swimmer & $99 \pm 11.5$ & $\mathbf{107 \pm 6.47}$ & $8.08 $ & $78 \pm 15.4$ & $\mathbf{130 \pm 14.2}$ & $66.7$ & $61 \pm 35.2$ & $\mathbf{161 \pm 26.9}$ & $164$ \\
        Hopper & $\mathbf{2503 \pm 786.6}$ & $2453 \pm 673.2$ & $-2.00$ & $3044 \pm 574.0$ & $\mathbf{3244 \pm 226.3}$ & $6.57$ & $2908 \pm 600.8$ & $\mathbf{3188 \pm 485.2}$ & $9.63$ \\
        Walker2d & $3405 \pm 842.0$ & $\mathbf{3554 \pm 928.0}$ & $4.38$ & $3764 \pm 234.4$ & $\mathbf{4251 \pm 567.1}$ & $12.9 $ & $4362 \pm 405.5$ & $\mathbf{4742 \pm 484.4}$ & $8.71$ \\
        Ant & $1762 \pm 540.0$ & $\mathbf{2378 \pm 843.4}$ & $35.0$ & $3492 \pm 1745.7$ & $\mathbf{5074 \pm 519.3}$ & $45.3$ & $4846 \pm 1306.4$ & $\mathbf{5482 \pm 511.9}$ & $13.1$ \\
        HalfCheetah & $2636 \pm 1344.3$ & $\mathbf{3104 \pm 926.0}$ & $17.8$ & $10316 \pm 193.8$ & $\mathbf{10473 \pm 563.4}$ & $1.52 $ & $11154 \pm 457.1$ & $\mathbf{11587 \pm 418.3}$ & $3.88$ \\
        Humanoid & $619 \pm 92.1$ & $\mathbf{646 \pm 127.1}$ & $4.36$ & $5973 \pm 257.7$ & $\mathbf{6275 \pm 483.2}$ & $5.06$ & $\mathbf{5261 \pm 186.4}$ & $5218 \pm 228.4$ & $-0.82 $ \\
        \midrule
        \textit{RPI Mean} & - & - & $11.3 $ & - & - & $23.0 $ & - & - & $33.1 $ \\
        \bottomrule
    \end{tabular}
}
\vspace{-10pt}
\end{center}
\end{table}

MuJoCo testbed is a widely used physics-based simulation environment. MuJoCo provides a suite of continuous control tasks where agents must learn to perform various actions, such as locomotion, manipulation, and balancing, within simulated robotic environments. Since comparisons among state-of-the-art RL baselines, such as PPO, SAC, TD3, and DSAC on MuJoCo, have been extensively covered in previous studies, our focus is on investigating how \cae can enhance these algorithms.

\begin{wrapfigure}{r}{0.5\textwidth}
\vspace{-1.1\baselineskip}
  \captionsetup{type=table}
  \caption{Experimental results after $1e6$ interaction steps on MuJoCo-v3, except for the \textit{Humanoid} task, which is evaluated after $2e6$ steps.}
  \label{tab:mujoco_v3_results}
  \centering
  \scalebox{0.8}{
  \begin{tabular}{l|ccc}
        \toprule
        \diagbox{Env}{Alg.} & DSAC & DSAC $+$ \cae & \textit{RPI \%} \\
        \midrule
        Swimmer & $131 \pm 14.8$ & $\mathbf{150 \pm 7.96}$ & $14.5$ \\
        Hopper & $2417 \pm 541.6$ & $\mathbf{2845 \pm 594.5}$ & $17.7$ \\
        Walker2d & $5550 \pm 624.0$ & $\mathbf{6069 \pm 422.1}$ & $9.35$ \\
        Ant & $5912 \pm 809.7$ & $\mathbf{6305 \pm 322.7}$ & $6.65$ \\
        HalfCheetah & $16036 \pm 439.1$ & $\mathbf{16338 \pm 249.}1$ & $1.88$ \\
        Humanoid & $10059 \pm 996.1$ & $\mathbf{10333 \pm 1104.4}$ & $2.72$ \\
        \midrule
        \textit{RPI Mean} & - & - & $8.80$ \\
        \bottomrule
    \end{tabular}
}
\end{wrapfigure}
Results are summarized in Table~\ref{tab:mujoco_v4_results}, averaged over random seeds $\left\{ 1, 2, 3, 4, 5 \right\}$. The results show that \cae consistently improves the performance of PPO, TD3, and SAC across most MuJoCo tasks. Notably, TD3 and SAC augmented with \cae achieve substantially better performance on the \textit{Swimmer} task. Although this task is not considered particularly challenging, the standalone TD3 and SAC baselines obtain relatively limited returns.

In Table~\ref{tab:mujoco_v3_results}, we summarize the performance of DSAC with and without \cae. These experiments are conducted on MuJoCo-v3, rather than MuJoCo-v4, because the open-source DSAC codebase is built upon MuJoCo-v3. As shown, incorporating \cae consistently improves the performance of DSAC across the MuJoCo benchmark tasks.

\subsection{MiniHack tasks with Sparse Rewards}
\label{subsec:minihack}

\begin{wrapfigure}{r}{0.5\textwidth}
\vspace{-1.8\baselineskip}
  \centering
  \subcaptionbox{Horn-Restricted\label{fig:freeze_horn_restricted}}[0.45\linewidth]{
    \includegraphics[width=\linewidth, trim={0.5cm 0.5cm 3.5cm 1.5cm}, clip]{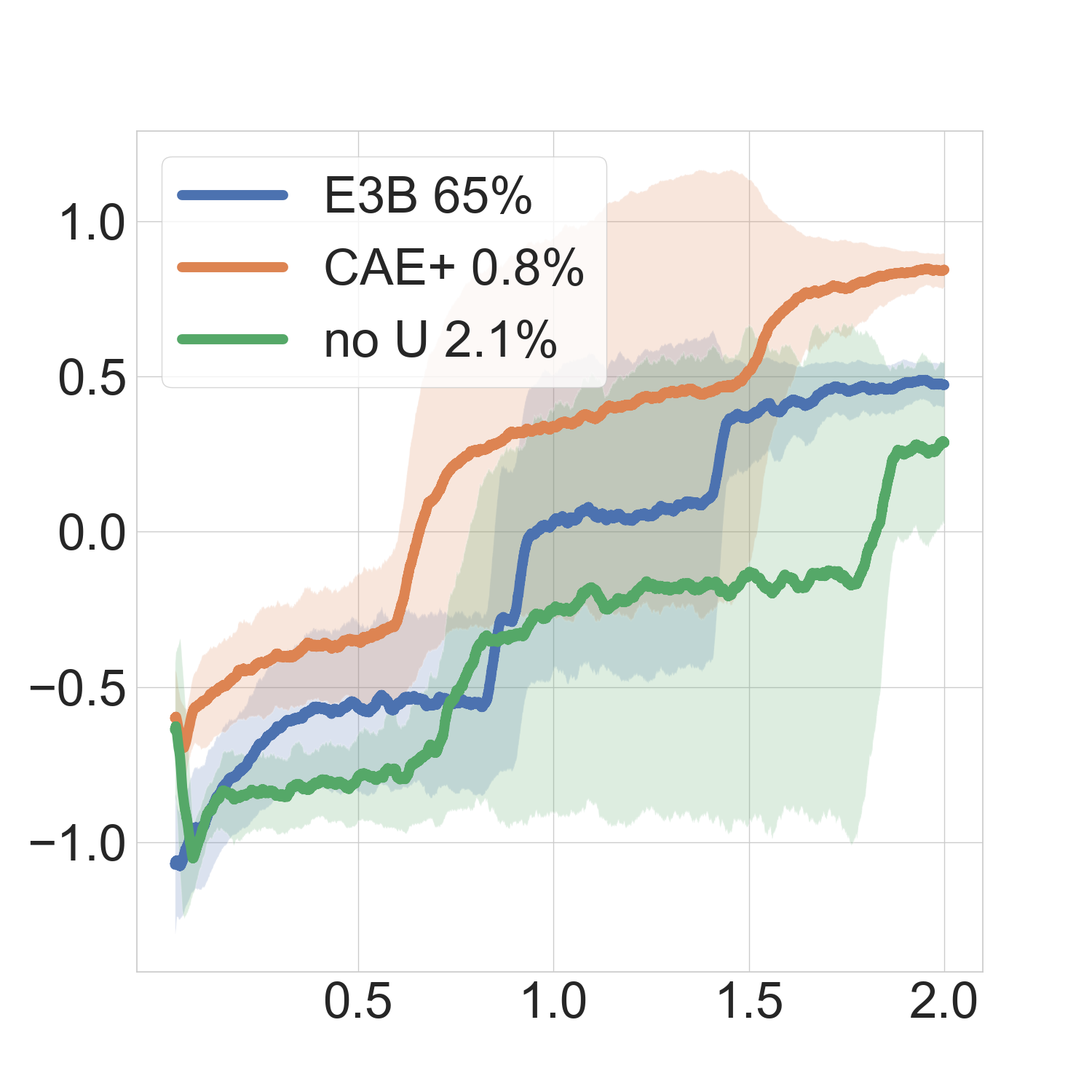}
  }
  \hfill
  \subcaptionbox{N6-Locked\label{fig:multiroom-n6-locked}}[0.45\linewidth]{
    \includegraphics[width=\linewidth, trim={0.5cm 0.5cm 3.5cm 1.5cm}, clip]{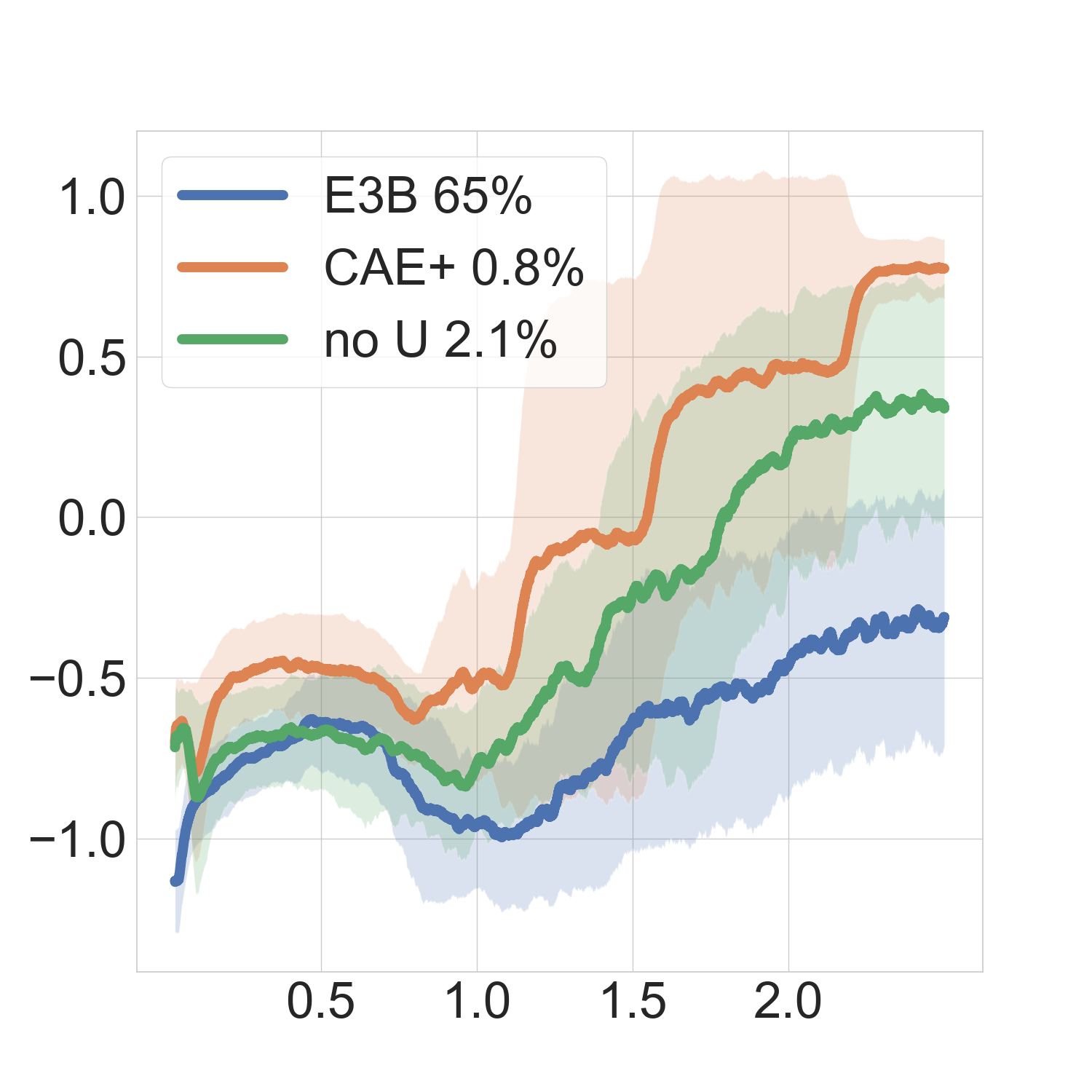}
  }
  \caption{Ablation study to $\boldsymbol{U}$ on MiniHack. Horizontal axis denotes the steps in multiples of $1e7$.}
  \label{fig:ablation_study_minihack}
\end{wrapfigure}
MiniHack~\citep{Minihack} is built on the NetHack Learning Environment~\citep{nethack}, a challenging video game where an agent navigates procedurally generated dungeons to retrieve a magical amulet. MiniHack tasks present a diverse set of challenges, such as locating and utilizing magical objects, traversing hazardous environments like lava and monsters. These tasks are characterized by sparse rewards, and the state provides a wealth of information, including images, texts, and more, though only a subset is relevant to each specific task.

As shown in Table~\ref{tab:empirical_algorithm_comparison}, \caep introduces only a $0.8\%$ increase in parameters compared to the base RL algorithm, IMPALA. In contrast, other exploration baselines, such as RIDE and E3B, require $60\% - 80\%$ additional parameters, underscoring the lightweight design of \caep. Experimental results for E3B and \caep, averaged over seeds $\left\{ 1, 2, 3 \right\}$, are summarized in Table~\ref{tab:minihack_results}. Performance is evaluated on nine representative tasks, \ie, five \textit{navigation} tasks and four \textit{skill} tasks. The results demonstrate that \caep consistently outperforms E3B across tasks. In particular, on challenging tasks such as \textit{N6-Locked} and \textit{LavaCross}, \caep achieves substantial performance improvements of $348\%$ and $282\%$, respectively, highlighting its strong exploration capabilities.

\paragraph{Ablation study to \cae on MiniHack.} Results of \cae on MiniHack tasks are provided in Table~\ref{tab:minihack_results}. As shown, \cae successfully solves a subset of tasks and even surpasses \caep in certain cases. However, it struggles to achieve positive performance in others, such as \textit{N6-Locked}, \etc, which pose significant exploration challenges. This limitation stems from the difficulty of training effective value networks in complex environments, adversely affecting exploration reliant on them.

\paragraph{Ablation study to the transformation $\boldsymbol{U}$.} Additionally, we present experimental results for \caep without the transformation matrix $\boldsymbol{U}$ in the auxiliary network $f$. As shown in Figure~\ref{fig:ablation_study_minihack}, \caep without $\boldsymbol{U}$ occasionally outperforms E3B, though there are instances where it does not. Importantly, it consistently underperforms the full \caep method with $\boldsymbol{U}$. Moreover, \caep introduces more additional parameters without matrix $\boldsymbol{U}$, specifically $2.1 \%$.

\paragraph{Wall-clock running time analysis.} RL experiments involve both network training and environment interaction on CPUs; therefore, wall-clock time is typically not regarded as a standard evaluation metric. Nevertheless, to provide a clear picture of the practical efficiency of \cae, we report wall-clock running times on MiniHack tasks. \cae requires approximately 17 hours to complete training per task, whereas E3B takes around 22 hours, demonstrating the superior efficiency of \cae.

\vspace{-5pt}
\begin{table}[H]
\caption{Experimental results after $2e7-3e7$ interaction steps on nine MiniHack tasks, whose detailed descriptions are in Appendix~\ref{app:minihack}. \textit{RPI} quantifies the improvement of \caep compared to E3B. Since E3B is the state-of-the-art on MiniHack~\citep{E3B}, and \caep consistently outperforms it, we conclude that \textbf{\caep $\succ$ E3B, ICM, RND, RIDE, NovelD}. Refer to Appendix~\ref{app:experiment} for experimental figures.}
\label{tab:minihack_results}

\begin{center}
\resizebox{\textwidth}{!}{
\begin{tabular}{c|ccccccccc}
        \toprule
        \diagbox{Alg.}{Env} & N4 & N4-Locked & N6 & N6-Locked & N10-OD & Horn & Random & Wand & LavaCross 
        \\
        \midrule
        E3B  & $0.86 \pm 0.010$ & $0.72 \pm 0.090$ & $0.75 \pm 0.019$ & $-0.31 \pm 0.403$ & $0.71 \pm 0.042$ & $0.47 \pm 0.071$ & $0.57 \pm 0.071$ & $0.49 \pm 0.201$ & $0.22 \pm 0.412$ \\
        \cae  & $0.93 \pm 0.014$ & $0.84 \pm 0.041$ & $-0.46 \pm 0.193$ & $-0.37 \pm 0.310$ & $-0.84 \pm 0.216$ & $\mathbf{0.92 \pm 0.022}$ & $\mathbf{0.93 \pm 0.022}$ & $\mathbf{0.93 \pm 0.030}$ & $0.16 \pm 0.394$ \\
        \caep  & $\mathbf{0.97 \pm 0.006}$ & $\mathbf{0.87 \pm 0.017}$ & $\mathbf{0.94 \pm 0.014}$ & $\mathbf{0.77 \pm 0.093}$ & $\mathbf{0.86 \pm 0.023}$ & $0.84 \pm 0.055$ & $0.80 \pm 0.040$ & $0.65 \pm 0.131$ & $\mathbf{0.84 \pm 0.024}$ \\
        \midrule
        \textit{RPI \%} & $12.79 $ & $20.83 $ & $25.3 $ & $348.39 $ & $21.13 $ & $78.72 $ & $40.35 $ & $32.65 $ & $281.82 $ \\
        \bottomrule
    \end{tabular}
}
\vspace{-10pt}
\end{center}
\end{table}

\subsection{Reward-free Habitat task}
\label{subsec:habitat}

Habitat~\citep{Habitat, Habitat_2, Habitat_3} is a platform for embodied AI research that supports agent navigation and interaction within simulations of real-world indoor environments. The experiments in this subsection are designed to evaluate exploration capabilities in visually rich, realistic settings. We employ the HM3D dataset~\citep{HM3D}, which comprises high-quality reconstructions of $1,000$ diverse indoor spaces. Agents are trained \textbf{solely} with the generated exploration bonuses on the \textit{PointNav} task. Evaluation is performed in unseen test environments by measuring the proportion of the environment revealed.

\begin{wrapfigure}{r}{0.5\textwidth}
\vspace{-1.1\baselineskip}
  \captionsetup{type=table}
  \caption{Experimental results after $1e8$ interaction steps on Habitat.}
  \label{tab:habitat_results}
  \centering
  \scalebox{0.9}{
  \begin{tabular}{l|cc|c}
        \toprule
        E3B & \cae & \caep & \textit{RPI \%} \\
        \midrule
        $0.51 \pm 0.097$ & $0.49 \pm 0.102$ & $\mathbf{0.69 \pm 0.074}$ & $35.29 $  \\
        \bottomrule
    \end{tabular}
}
\end{wrapfigure}
Results for E3B, \cae, and \caep, averaged over three seeds $\left\{ 1, 2, 3 \right\}$, are summarized in Table~\ref{tab:habitat_results}, highlighting the superior performance of \caep. Since E3B is the state-of-the-art on Habitat, and \caep further outperforms E3B, it follows that \textbf{\caep $\succ$ E3B, ICM, RND, RIDE, NovelD}. Figure~\ref{fig:habitat} provides an illustration of a specific case, where the trained policy of \caep explores a larger portion of the environment compared to that of E3B.

\begin{figure}[H]
  \centering
  \begin{subfigure}[t]{0.3\textwidth}
    \includegraphics[width=\textwidth, trim={0 0 0 0}, clip]{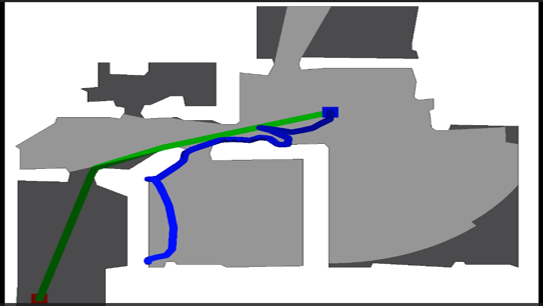}
    \caption{\caep reveals $0.67$ of the map}
    \label{fig:habitat_CAEP-n4}
  \end{subfigure}
  \hfill
  \begin{subfigure}[t]{0.3\textwidth}
    \includegraphics[width=\textwidth, trim={0 0 0 0}, clip]{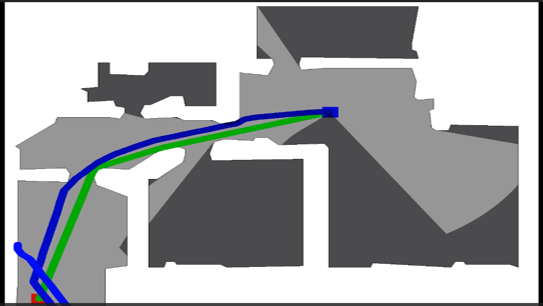}
    \caption{E3B reveals $0.53$ of the map}
    \label{fig:habitat_E3B}
  \end{subfigure}
  \hfill
  \begin{subfigure}[t]{0.3\textwidth}
    \includegraphics[width=\textwidth, trim={0 0 0 0}, clip]{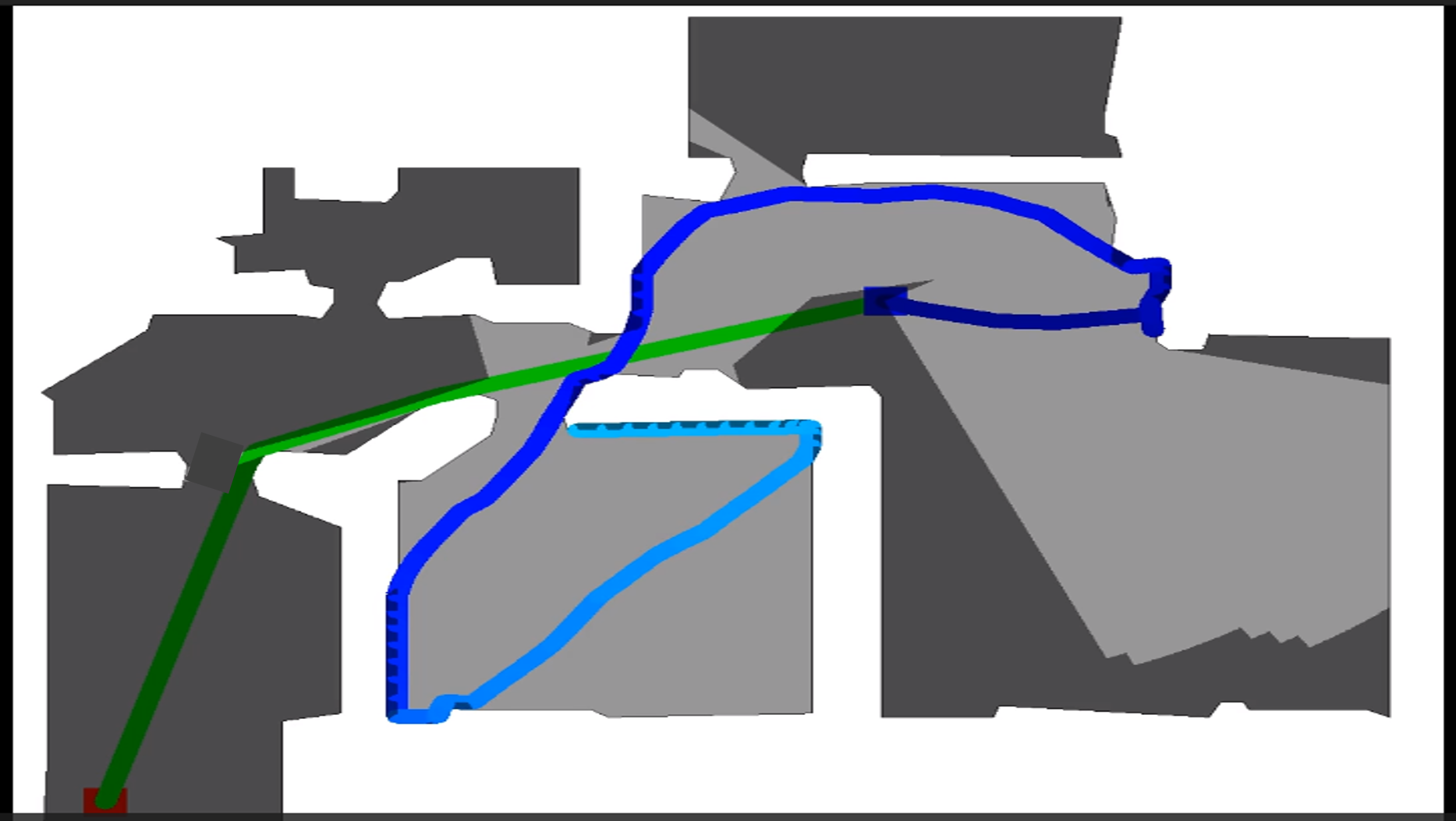}
    \caption{\cae reveals $0.51$ of the map}
    \label{fig:habitat_CAE}
  \end{subfigure}
  \caption{Trajectories of the learned policies on a Habitat environment unseen during training.}
  \label{fig:habitat}
\vspace{-5pt}
\end{figure}

\section{Conclusion}
\label{sec:conclusion}

In this paper, we propose \cae, a lightweight exploration method that seamlessly integrates with existing RL algorithms without adding parameters. \cae exploits the value network’s embedding layers to guide exploration, requiring no changes to the rest of the algorithm. A simple scaling strategy ensures stable learning. For sparse-reward tasks, we extend \cae to \caep by adding a small auxiliary network, accelerating learning with minimal overhead. We provide theoretical guarantees in the form of sub-linear regret bounds. Extensive experiments show that \cae and \caep significantly outperform baseline methods across dense-reward, sparse-reward, and reward-free settings.

\subsection*{Broader Impact Statement}
\label{sec:broader_impact}

\cae and \caep bridge the gap between provably efficient and practically successful exploration methods. While theoretically grounded approaches often suffer from scalability and applicability issues, empirically driven methods typically lack theoretical guarantees and require extensive parameter training. Inspired by the \textit{deep representation and shallow exploration} paradigm, this work proposes a novel framework that decomposes value networks in deep RL and repurposes their embedding layers for exploration, thereby achieving theoretical guarantees with minimal or no additional parameter training.

This framework contributes to the development of sample-efficient and interpretable exploration methods in deep RL, potentially accelerating progress in applications such as robotics, recommendation systems, autonomous decision-making systems, and large language models. Moreover, the lightweight nature of the method makes it suitable for deployment in resource-constrained settings. Future research could focus on integrating a broader spectrum of MAB techniques, evaluating the robustness across a wide range of tasks, and improving the computational efficiency of uncertainty estimation.

\subsubsection*{Acknowledgments}
The first author thanks Ping Huang and Hanfang Zhang for their assistance with configuring several open-source baselines, and Shuo Chen and Siyuan Qi for their substantial help in revising and improving the manuscript.

This work is supported by the National Natural Science Foundation of China (No. 62506041).

\bibliography{main}

@inproceedings{LinUCB,
  author = "Li, Lihong and Chu, Wei and Langford, John and Schapire, Robert E.",
  title = "A Contextual-Bandit Approach to Personalized News Article Recommendation",
  year = "2010",
  pages     = {661--670},
  booktitle = "Proceedings of the 19th International Conference on World Wide Web",
}

@inproceedings{LinUCB_proof,
  author = "Chu, Wei and Li, Lihong and Reyzin, Lev and Schapire, Robert E.",
  title = "Contextual Bandits with Linear Payoff Functions",
  year = "2011",
  pages = {208--214},
  booktitle = "Proceedings of the 14th International Conference on Artiﬁcial Intelligence and Statistics",
}

@Article{KernelUCB,
  author = "Valko, Michal and Korda, Nathaniel and Munos, Remi and Flaounas, Ilias and Cristianini, Nelo",
  title = "Finite-Time Analysis of Kernelised Contextual Bandits",
  journal= "arXiv preprint arXiv: 1309.6869.",
  year = "2013"
}

@inproceedings{NeuralUCB,
  author = "Zhou, Dongruo and Li, Lihong and Gu, Quanquan",
  title = "Neural Contextual Bandits with UCB-based Exploration",
  booktitle = "Proceedings of the 37th International Conference on Machine Learning",
  year = "2020",
  pages = {11492--11502},
  organization={PMLR}
}

@inproceedings{ComLinUCB,
  author = "Wen, Zheng and Kveton, Branislav and Ashkan, Azin",
  title = "Efficient Learning in Large-Scale Combinatorial Semi-Bandits",
  booktitle = "Proceedings of the 32nd International Conference on Machine Learning",
  year = "2015",
  pages = {1113--1122},
  organization={PMLR}
}

@inproceedings{NeuralTS,
  author = "Zhang, Weitong and Zhou, Dongruo and Li, Lihong and Gu, Quanquan",
  title = "Neural Thompson Sampling",
  booktitle = "Proceedings of the 9th International Conference on Learning Representations",
  year = "2021"
}

@inproceedings{LinTS,
  author = "Agrawal, Shipra and Goyal, Navin",
  title = "Thompson Sampling for Contextual Bandits with Linear Payoffs",
  booktitle = "Proceedings of the 30th International Conference on Machine Learning",
  year = "2013",
  pages = {127--135},
  organization={PMLR}
}

@Article{DQN,
  author = "Mnih, Volodymyr and Kavukcuoglu, Koray and Silver, David and Rusu, Andrei A. and Veness, Joel and Bellemare, Marc G. and Graves, Alex and Riedmiller, Martin and Fidjeland, Andreas K. and Ostrovski, Georg and Petersen, Stig and Beattie, Charles and Sadik, Amir and Antonoglou, Ioannis and King, Helen and Kumaran, Dharshan and Wierstra1, Daan and Legg, Shane and Hassabis, Demis",
  title = "Human-level Control through Deep Reinforcement Learning",
  journal = "nature",
  year = "2015",
  pages = {529--533},
  publisher={Nature Publishing Group}
}

@inproceedings{shallow_exploration,
  author = "Xu, Pan and Wen, Zheng and Zhao, Handong and Gu, Quanquan",
  title = "Neural Contextual Bandits with Deep Representation and Shallow Exploration",
  booktitle = "Proceedings of the 10th International Conference on Learning Representations",
  year = "2022"
}

@inproceedings{GP_UCB,
  author = "Chowdhury, Sayak Ray and Gopalan, Aditya",
  title = "On Kernelized Multi-armed Bandits",
  booktitle = "Proceedings of the 34th International Conference on Machine Learning",
  year = "2017",
  organization={PMLR}
}

@inproceedings{IMPALA,
  author = "Espeholt, Lasse and Soyer, Hubert and Munos, Remi and Simonyan, Karen and Mnih, Volodymyr and Ward, Tom and Doron, Yotam and Firoiu, Vlad and Harley, Tim and Dunning, Iain and Legg, Shane and Kavukcuoglu, Koray",
  title = "IMPALA: Scalable Distributed Deep-RL with Importance Weighted Actor-Learner Architectures",
  booktitle = "Proceedings of the 35th International Conference on Machine Learning",
  year = "2018",
  organization={PMLR}
}

@inproceedings{E3B,
  author = "Henaff, Mikael and Raileanu, Roberta and Jiang, Minqi and Rocktäschel, Tim",
  title = "Exploration via Elliptical Episodic Bonuses",
  booktitle = "Proceedings of the 36th Conference on Neural Information Processing System",
  year = "2022"
}

@inproceedings{RIDE,
  author = "Raileanu, Roberta and Rocktäschel, Tim",
  title = "RIDE: Rewarding Impact-Driven Exploration for Procedurally-Generated Environments",
  booktitle = "Proceedings of the 8th International Conference on Learning Representations",
  year = "2020"
}

@inproceedings{NovelD,
  author = "Zhang, Tianjun and Xu, Huazhe and Wang, Xiaolong and Wu, Yi and Keutzer, Kurt and Gonzalez, Joseph E. and Tian, Yuandong",
  title = "NovelD: A Simple yet Effective Exploration Criterion",
  booktitle = "Proceedings of the 35th Conference on Neural Information Processing System",
  year = "2021"
}

@inproceedings{ICM,
  author = "Pathak, Deepak and Agrawal, Pulkit and Efros, Alexei A. and Darrell, Trevor",
  title = "Curiosity-driven Exploration by Self-supervised Prediction",
  booktitle = "Proceedings of the 34th International Conference on Machine Learning",
  year = "2017",
  organization={PMLR}
}

@inproceedings{RND,
  author = "Burda, Yuri and Edwards, Harrison and Storkey, Amos and Klimov, Oleg",
  title = "Exploration by Random Network Distillation",
  booktitle = "Proceedings of the 7th International Conference on Learning Representations",
  year = "2019"
}

@inproceedings{AGAC,
  author = "Flet-Berliac, Yannis and Ferret, Johan and Pietquin, Olivier and Preux, Philippe and Geist, Matthieu",
  title = "Adversarially Guided Actor-Critic",
  booktitle = "Proceedings of the 9th International Conference on Learning Representations",
  year = "2021"
}

@inproceedings{E3B_study,
  author = "Henaff, Mikael and Jiang, Minqi and Raileanu, Roberta",
  title = "A Study of Global and Episodic Bonuses for Exploration in Contextual MDPs",
  booktitle = "Proceedings of the 40th International Conference on Machine Learning",
  year = "2023",
  organization={PMLR}
}

@inproceedings{PC-PG,
  author = "Agarwal, Alekh and Kakade, Sham and Henaff, Mikael and Sun, Wen",
  title = "PC-PG: Policy Cover Directed Exploration for Provable Policy Gradient Learning",
  booktitle = "Proceedings of the 34th Conference on Neural Information Processing Systems",
  year = "2020"
}

@inproceedings{ACB,
  author = "Ash, Jordan T. and Zhang, Cyril and Goel, Surbhi and Krishnamurthy, Akshay and Kakade, Sham",
  title = "Anti-Concentrated Confidence Bonuses for Scalable Exploration",
  booktitle = "Proceedings of the 10th International Conference on Learning Representations",
  year = "2022"
}

@inproceedings{provable_kernel_nn,
  author = "Yang, Zhuoran and Jin, Chi and Wang, Zhaoran and Wang, Mengdi and Jordan, Michael I.",
  title = "On Function Approximation in Reinforcement Learning: Optimism in the Face of Large State Spaces",
  booktitle = "Proceedings of the 34th Conference on Neural Information Processing System",
  year = "2020"
}

@inproceedings{provable_q_learning,
  author = "Jin, Chi and Allen-Zhu, Zeyuan and Bubeck, Sebastien and Jordan, Michael I.",
  title = "Is Q-learning Provably Efficient?",
  booktitle = "Proceedings of the 32nd Conference on Neural Information Processing System",
  year = "2018"
}

@inproceedings{provable_random_value_function,
  author = "Osband, Ian and Roy, Van Benjamin and Wen, Zheng",
  title = "Generalization and Exploration via Randomized Value Functions",
  booktitle = "Proceedings of the 33rd International Conference on Machine Learning",
  year = "2016",
  organization={PMLR}
}

@Article{provable_random_value_function_,
  author = "Osband, Ian and Roy, Van Benjamin and Russo, J. Daniel and Wen, Zheng",
  title = "Deep Exploration via Randomized Value Functions",
  journal = "Journal of Machine Learning Research",
  year = "2019",
  pages = {1--61}
}

@inproceedings{provable_linear_function,
  author = "Jin, Chi and Yang, Zhuoran and Wang, Zhaoran and Jordan, Michael I.",
  title = "Provably Efficient Reinforcement Learning with Linear Function Approximation",
  booktitle = "In Conference on Learning Theory",
  year = "2020",
  organization={PMLR}
}

@inproceedings{pseudocount,
  author = "Bellemare, Marc G. and Srinivasan, Sriram and Ostrovski, Georg and Schaul, Tom and Saxton, David and Munos, Remi",
  title = "Unifying Count-based Exploration and Intrinsic Motivation",
  booktitle = "Proceedings of the 30th Conference on Neural Information Processing System",
  year = "2016"
}

@inproceedings{Minihack,
  author = "Samvelyan, Mikayel and Kirk, Robert and Kurin, Vitaly and Parker-Holder, Jack and Jiang, Minqi and Hambro, Eric and Petroni, Fabio and Kuttler, Heinrich and Grefenstette, Edward and Rocktäschel, Tim",
  title = "MiniHack the Planet: A Sandbox for Open-Ended Reinforcement Learning Research",
  booktitle = "Proceedings of the 35th Conference on Neural Information Processing Systems",
  year = "2021"
}

@inproceedings{Habitat,
  author = "Savva, Manolis and Kadian, Abhishek and Maksymets, Oleksandr and Zhao, Yili and Wijmans, Erik and Jain, Bhavana and Straub, Julian and Liu, Jia and Koltun, Vladlen and Malik, Jitendra and Parikh, Devi and Batra, Dhruv",
  title = "Habitat: A Platform for Embodied AI Research",
  booktitle = "In Proceedings of the IEEE/CVF International Conference on Computer Vision",
  year = "2019"
}

@Article{tabular_rl,
  author = "Kearns, Michael and Singh, Satinder",
  title = "Near-Optimal Reinforcement Learning in Polynomial Time",
  journal= "Machine Learning",
  year = "2002"
}

@Article{PPO,
  author = "Schulman, John and Wolski, Filip and Dhariwal, Prafulla and Radford, Alec and Klimov, Oleg",
  title = "Proximal Policy Optimization Algorithms",
  journal= "arXiv preprint arXiv: 1707.06347v2",
  year = "2017"
}

@inproceedings{SAC,
  author = "Haarnoja, Tuomas and Zhou, Aurick and Abbeel, Pieter and Levine, Sergey",
  title = "Soft Actor-Critic: Off-Policy Maximum Entropy Deep Reinforcement Learning with a Stochastic Actor",
  booktitle = "Proceedings of the 35th International Conference on Machine Learning",
  year = "2018",
  organization={PMLR}
}

@inproceedings{DDPG,
  author = "Lillicrap, Timothy P. and Hunt, Jonathan J. and Pritzel, Alexander and Heess, Nicolas and Erez, Tom and Tassa, Yuval and Silver, David and Wierstra, Daan",
  title = "Continuous Control with Deep Reinforcement Learning",
  booktitle = "Proceedings of the 4th International Conference on Learning Representations",
  year = "2016"
}

@inproceedings{TD3,
  author = "Fujimoto, Scott and Hoof, Herke van and Meger, David",
  title = "Addressing Function Approximation Error in Actor-Critic Methods",
  booktitle = "Proceedings of the 35th International Conference on Machine Learning",
  year = "2018",
  organization={PMLR}
}

@Article{StarCraft,
  author = "Vinyals, Oriol and Babuschkin, Igor and Czarnecki, Wojciech M. and Mathieu, Michaël and Dudzik, Andrew and Chung, Junyoung and Choi, David H. and Powell, Richard and Ewalds, Timo and Georgiev, Petko et al.",
  title = "Grandmaster Level in StarCraft II Using Multi-Agent Reinforcement Learning",
  journal= "Nature",
  year = "2019"
}

@Article{Go,
  author = "Silver, David and Schrittwieser, Julian and Simonyan, Karen and Antonoglou, Ioannis and Huang, Aja and Guez, Arthur and Hubert, Thomas and Baker, Lucas and Lai, Matthew and Bolton, Adrian and Chen, Yutian and Lillicrap, Timothy et al.",
  title = "Mastering the Game of Go without Human Knowledge",
  journal= "Nature",
  year = "2017"
}

@inproceedings{Neural-LinTS,
  author = "Riquelme, Carlos and Tucker, George and Snoek, Jasper",
  title = "Deep Bayesian Bandits Showdown: An Empirical Comparison of Bayesian Deep Networks for Thompson Sampling",
  booktitle = "Proceedings of the 6th International Conference on Learning Representations",
  year = "2018"
}

@Article{deep_neural_lin_bandit,
  author = "Zahavy, Tom and Mannor, Shie",
  title = "Neural Linear Bandits: Overcoming Catastrophic Forgetting through Likelihood Matching",
  journal= "arXiv preprint arXiv: 1901.08612v2",
  year = "2019"
}

@inproceedings{NegUCB,
  author = "Li, Yexin and Mu, Zhancun and Qi, Siyuan",
  title = "A Contextual Combinatorial Bandit Approach to Negotiation",
  booktitle = "Proceedings of the 41st International Conference on Machine Learning",
 year = "2024"
}

@inproceedings{OPPO,
  author = "Cai, Qi and Yang, Zhuoran and Jin, Chi and Wang, Zhaoran",
  title = "Provably Efficient Exploration in Policy Optimization",
  booktitle = "Proceedings of the 37th International Conference on Machine Learning",
 year = "2020"
}

@inproceedings{neural_tangent_kernel,
  author = "Jacot, Arthur and Gabriel, Franck and Hongler, Clément",
  title = "Neural Tangent Kernel: Convergence and Generalization in Neural Networks",
  booktitle = "Proceedings of the 32nd Conference on Neural Information Processing System",
  year = "2018"
}

@inproceedings{wide_net,
  author = "Arora, Sanjeev and Du, Simon S. and Hu, Wei and Li, Zhiyuan and Salakhutdinov, Ruslan and Wang, Ruosong",
  title = "On Exact Computation with an Infinitely Wide Neural Net",
  booktitle = "Proceedings of the 33rd Conference on Neural Information Processing Systems ",
  year = "2019"
}

@inproceedings{nethack,
  author = "Küttler, Heinrich and Nardelli, Nantas and Miller, Alexander H. and Raileanu, Roberta and Selvatici, Marco and Grefenstette, Edward and Rocktäschel, Tim",
  title = "The NetHack Learning Environment",
  booktitle = "Proceedings of the 34th Conference on Neural Information Processing Systems",
  year = "2020"
}

@Article{cleanrl,
  author  = "Huang, Shengyi and Dossa, Rousslan Fernand Julien and Ye, Chang and Braga, Jeff and Chakraborty, Dipam and Mehta, Kinal and Araújo, João G.M.",
  title   = "CleanRL: High-quality Single-file Implementations of Deep Reinforcement Learning Algorithms",
  journal = "Journal of Machine Learning Research",
  pages   = {1--18},
  year    = 2022
}

@Article{atari_dqn,
  author = "Mnih, Volodymyr and Kavukcuoglu, Koray and Silver, David and Graves, Alex and Antonoglou, Ioannis and Wierstra, Daan and Riedmiller, Martin",
  title = "Playing Atari with Deep Reinforcement Learning",
  journal= "arXiv preprint arXiv: 1312.5602v1",
  year = "2013"
}

@Article{DSAC,
  author = "Duan, Jingliang and Guan, Yang and Li, Shengbo and Ren, Yangang and Sun, Qi and Cheng, Bo",
  title = "Distributional Soft Actor-Critic: Off-Policy Reinforcement Learning for Addressing Value Estimation Errors",
  journal = "IEEE Transactions on Neural Networks and Learning Systems",
  year = "2021",
  pages = {6584 - 6598}
}

@Article{DSAC-T,
  author = "Duan, Jingliang and Wang, Wenxuan and Xiao, Liming and Gao, Jiaxin and Li, Shengbo",
  title = "DSAC-T: Distributional Soft Actor-Critic with Three Refinements",
  journal = "arXiv:2310.05858v4",
  year = "2023",
}

@inproceedings{LMCDQN,
  author = "Ishfaq, Haque and Lan, Qingfeng and Xu, Pan and Mahmood, A. Rupam and Precup, Doina and Anandkumar, Anima and Azizzadenesheli, Kamyar",
  title = "Provable and Practical: Efficient Exploration in Reinforcement Learning via Langevin Monte Carlo",
  booktitle = "Proceedings of the 12nd International Conference on Learning Representations",
  year = "2024"
}

@inproceedings{LSVI-PHE,
  author = "Ishfaq, Haque and Cui, Qiwen and Nguyen, Viet and Ayoub, Alex and Yang Zhuoran and Wang, Zhaoran and Precup, Doina and Yang, F. Lin",
  title = "Randomized Exploration for Reinforcement Learning with General Value Function Approximation",
  booktitle = "Proceedings of the 38th International Conference on Machine Learning",
  year = "2021"
}

@inproceedings{study_curiosity,
  author = "Burda, Yuri and Edwards, Harri and Pathak, Deepak and Storkey, Amos and Darrell, Trevor and Efros, A. Alexei",
  title = "Large-Scale Study of Curiosity-Driven Learning",
  booktitle = "Proceedings of the 7th International Conference on Learning Representations",
  year = "2019"
}

@inproceedings{Automatic_ir,
  author = "Yuan, Mingqi and Li, Bo and Jin, Xin and Zeng, Wenjun",
  title = "Automatic Intrinsic Reward Shaping for Exploration in Deep Reinforcement Learning",
  booktitle = "Proceedings of the 40th International Conference on Machine Learning",
  year = "2023",
  pages = {40531--40554},
  organization={PMLR}
}

@Article{streaming-rl,
  author = "Elsayed, Mohamed and Vasan, Gautham and Mahmood, A. Rupam",
  title = "Deep Reinforcement Learning Without Experience Replay, Target Networks, or Batch Updates",
  journal = "38th Workshop on Fine-Tuning in Machine Learning, NeurIPS",
  year = "2024",
}

@Article{sample-mean-var,
  author = "Welford, B. P.",
  title = "Note on A Method for Calculating Corrected Sums of Squares and Products",
  journal = "Technometrics",
  year = "1962",
}

@Article{rank-1,
  author = "Sherman, Jack and Morrison, J. Winifred",
  title = "Adjustment of an Inverse Matrix Corresponding to a Change in One Element of a Given Matrix",
  journal = "The Annals of Mathematical Statistics",
  year = "1950",
}

@inproceedings{General_LMCDQN,
  author = "Ishfaq, Haque and Tan, Yixin and Yang, Yu and Lan, Qingfeng and Lu, Jianfeng and Mahmood, A. Rupam and Precup, Doina and Xu, Pan",
  title = "More Efficient Randomized Exploration for Reinforcement Learning via Approximate Sampling",
  booktitle = "Proceedings of the 1st Reinforcement Learning Conference",
  year = "2024"
}

@inproceedings{curiosity_in_hindsight,
  author = "Jarrett, Daniel and Tallec, Corentin and Altché, Florent and Mesnard, Thomas and Munos, Rémi and Valko, Michal",
  title = "Curiosity in Hindsight: Intrinsic Exploration in Stochastic Environments",
  booktitle = "Proceedings of the 40th International Conference on Machine Learning",
  year = "2023",
  organization={PMLR}
}

@Article{bayesian_q_net,
  author = "Kamyar, Azizzadenesheli and Animashree, Anandkumar",
  title = "Efficient Exploration through Bayesian Deep Q-Networks",
  journal = "Information Theory and Applications Workshop",
  year = "2018",
}

@inproceedings{Bayes-UCBVI,
  author = "Daniil, Tiapkin and Denis, Belomestny and Éric, Moulines and Alexey, Naumov and Sergey, Samsonov and Yunhao, Tang and Michal, Valko and Pierre Ménard",
  title = "From Dirichlet to Rubin: Optimistic Exploration in RL without Bonuses",
  booktitle = "Proceedings of the 39th International Conference on Machine Learning",
  year = "2022",
  organization={PMLR}
}

@inproceedings{RandQL,
  author = "Daniil, Tiapkin and Denis, Belomestny and Daniele, Calandriello and Éric, Moulines and Remi, Munos and Alexey, Naumov and Pierre, Perrault and Michal, Valko and Pierre, Ménard",
  title = "Model-free Posterior Sampling via Learning Rate Randomization",
  booktitle = "Proceedings of the 37th Conference on Neural Information Processing System",
  year = "2023",
}

@inproceedings{MR-NaS,
  author = "Alessio, Russo and Filippo, Vannella",
  title = "Multi-Reward Best Policy Identification",
  booktitle = "Proceedings of the 38th Conference on Neural Information Processing System",
  year = "2024",
}

@inproceedings{Liberty,
  author = "Yiming, Wang and Ming, Yang and Renzhi, Dong and Binbin, Sun and Furui, Liu and Leong, Hou U",
  title = "Efficient Potential-based Exploration in Reinforcement Learning using Inverse Dynamic Bisimulation Metric",
  booktitle = "Proceedings of the 37th Conference on Neural Information Processing System",
  year = "2023",
}

@inproceedings{EME,
  author = "Yiming, Wang and Kaiyan, Zhao and Furui, Liu and Leong, Hou U",
  title = "Rethinking Exploration in Reinforcement Learning with Effective Metric-Based Exploration Bonus",
  booktitle = "Proceedings of the 38th Conference on Neural Information Processing System",
  year = "2024",
}

@inproceedings{OPT-RLSVI,
  author = "Andrea, Zanette and David, Brandfonbrener and Emma, Brunskill and Matteo, Pirotta and Alessandro, Lazaric",
  title = "Frequentist Regret Bounds for Randomized Least-Squares Value Iteration",
  booktitle = "Proceedings of the 23rd International Conference on Artificial Intelligence and Statistics",
  year = "2020",
  organization={PMLR}
}

@inproceedings{Habitat_2,
  author = "Andrew, Szot and Alex, Clegg and Eric, Undersander and Erik, Wijmans and Yili, Zhao and John, Turner and Noah, Maestre and Mustafa, Mukadam and Devendra, Chaplot and Oleksandr, Maksymets and Aaron, Gokaslan and Vladimir, Vondrus and Sameer, Dharur and Franziska, Meier and Wojciech, Galuba and Angel, Chang and Zsolt, Kira and Vladlen, Koltun and Jitendra, Malik and Manolis, Savva and Dhruv, Batra",
  title = "Habitat 2.0: Training home assistants to rearrange their habitat",
  booktitle = "Proceedings of the 35th Conference on Neural Information Processing System",
  year = "2021",
}

@Article{HM3D,
  author = "Santhosh, K. Ramakrishnan and Aaron, Gokaslan and Erik, Wijmans and Oleksandr, Maksymets and Alex, Clegg and John, Turner and Eric, Undersander and Wojciech, Galuba and Andrew, Westbury and Angel, X. Chang and Manolis, Savva and Yili, Zhao and Dhruv, Batra",
  title = "Habitat-Matterport 3D Dataset (HM3D): 1000 Large-scale 3D Environments for Embodied AI",
  journal= "arXiv preprint arXiv: 2109.08238v1.",
  year = "2021"
}

@Article{Habitat_3,
  author = "Xavi, Puig and Eric, Undersander and Andrew, Szot and Mikael, Dallaire Cote and Ruslan, Partsey and Jimmy, Yang and Ruta, Desai and Alexander, William Clegg and Michal, Hlavac and Tiffany, Min and Theo, Gervet and Vladim\'{i}r, Vondru\v{s} and Vincent-Pierre, Berges and John, Turner and Oleksandr, Maksymets and Zsolt, Kira and Mrinal, Kalakrishnan and Jitendra, Malik and Devendra, Singh Chaplot and Unnat, Jain and Dhruv, Batra and Akshara, Rai and Roozbeh, Mottaghi",
  title = "Habitat 3.0: A Co-Habitat for Humans, Avatars, and Robots",
  journal = "arXiv preprint arXiv: 2310.13724v1",
  year = "2023",
}

@inproceedings{PBRS,
  author = "Andrew, Y. Ng and Daishi, Harada and Stuart, J. Russell",
  title = "Policy Invariance Under Reward Transformations: Theory and Application to Reward Shaping",
  booktitle = "Proceedings of the 16th International Conference on Machine Learning",
  year = "1999",
  organization={PMLR}
}
\bibliographystyle{tmlr}

\newpage
\appendix

\section{Implementation}
\label{app:code}

Our experiments are based on the following open-source codebases:

\begin{itemize}[leftmargin=1em]
    \item E3B~\citep{E3B}: \href{https://github.com/facebookresearch/e3b}{https://github.com/facebookresearch/e3b}
    \item CleanRL~\citep{cleanrl}: \href{https://github.com/vwxyzjn/cleanrl}{https://github.com/vwxyzjn/cleanrl}
    \item DSAC~\citep{DSAC, DSAC-T}: \href{https://github.com/Jingliang-Duan/DSAC-v2}{https://github.com/Jingliang-Duan/DSAC-v2}
    \item Habitat-lab~\citep{Habitat_2, Habitat_3}: \href{https://github.com/facebookresearch/habitat-lab}{https://github.com/facebookresearch/habitat-lab}
\end{itemize}

\vspace{10pt}
\lstset{
    language=Python,
    basicstyle=\scriptsize\ttfamily,        
    keywordstyle=\color{blue},   
    commentstyle=\color{blue},   
    stringstyle=\color{red},     
    emph={hello_world},          
    emphstyle=\color{magenta},   
    numbers=left,               
    numberstyle=\tiny,           
    stepnumber=1,                
    frame=single,                
    breaklines=true,
    abovecaptionskip=10pt,
    belowcaptionskip=10pt
}

In Listing~\ref{code:core_code}, we present the core code of \cae, while the rest of the deep RL algorithm remains unchanged. As shown, \cae is simple to implement, integrates seamlessly with any existing RL algorithm, and requires no additional parameter learning beyond what is already in the RL algorithm.

\begin{lstlisting}[caption={\cae core code}, label={code:core_code}, numbers=none]
A_inverse = torch.inverse(A)
phi = q_net.get_emb(torch.Tensor(obs), torch.Tensor(action)).squeeze().detach()
bouns = np.sqrt(torch.matmul(phi.T, torch.matmul(A_inverse, phi)).item())
reward += scaled_bonus
A += torch.outer(phi, phi)
\end{lstlisting}

In Listing~\ref{code:core_code_caep}, we present the additional code of \caep alongside that of \cae. As shown, \caep minimizes an additional loss, specifically the Inverse Dynamics Network (IDN) loss, in addition to the losses from the original RL algorithm.

\begin{lstlisting}[caption={Additional core code of \caep}, label={code:core_code_caep}, numbers=none]
phi = q_net.get_emb(torch.Tensor(batch['obs']), torch.Tensor(default_actions))
predict_action = inverse_dynamic_net(phi[:-1], phi[1:])
idn_loss = compute_inverse_dynamics_loss(predict_action, batch['action'][:-1])

def compute_inverse_dynamics_loss(action, true_action):
    loss=F.nll_loss(F.log_softmax(action, dim=-1), true_action, reduction='none')
    return loss

\end{lstlisting}

\newpage
\section{Long version of \cae}
\label{app:long_alg}

For a more comprehensive theoretical analysis, we present the theoretical version of \cae in Algorithm~\ref{alg:theoretical_algorithm}, where, for conciseness, we denote the embedding of the state-action pair at time step $h$ in episode $m$ by $\boldsymbol{\phi}_{h}^{m} = \phi(s_{h}^{m}, a_{h}^{m} | \boldsymbol{W}_{h}^{m})$. Following the standard notation in the literature on provable algorithms~\citep{provable_linear_function, provable_kernel_nn}, function parameters are assumed to be independent across different time steps $h \in \left[ H \right]$, and this convention is also adopted in Algorithm~\ref{alg:theoretical_algorithm}. As shown in the algorithm, the parameters $\boldsymbol{\theta}_{h}$ and $\boldsymbol{W}_{h}$ are updated iteratively to learn the two decomposed components of the action-value function in Equation~\ref{eq:q_fun} via the Bellman equation. Specifically, $\boldsymbol{\theta}_{h}$ is updated in Line~9 using its closed-form solution~\citep{LinUCB}, while the embedding network $\phi(\cdot, \cdot \mid \boldsymbol{W}_{h})$ is kept fixed. Subsequently, in Line~10, the embedding network $\phi(s, a \mid \boldsymbol{W}_{h})$ is updated with $\boldsymbol{\theta}_{h}$ held fixed. In this step, $\eta$ denotes the learning rate, $L_{h}^{m}$ is the Bellman loss, and $s_{h}^{t}, a_{h}^{t}, r_{h}^{t}$ for $\forall t \in \left[ m \right]$ denote the collected historical experiences.

\begin{algorithm}[h]
   \caption{DQN~\citep{DQN} enhanced with \cae}
   \label{alg:theoretical_algorithm}
\begin{algorithmic}[1]
\footnotesize
   \State {\bfseries Input:} Ridge parameter $\lambda > 0$, exploration parameter $\alpha \geq 0$, episode length $H$, episode number $M$, step size $\eta$, and the discount factor $\gamma$
   
   \State {\bfseries Initialize:} Gram matrix $\boldsymbol{A}_{h}^{1} = \lambda \boldsymbol{I}$, $\boldsymbol{b}_{h}^{1} = \boldsymbol{0}$, parameters $\boldsymbol{\theta}_{h}^{1} \sim \frac{1}{D} N(\boldsymbol{0}, \boldsymbol{I})$, networks $\phi (\cdot, \cdot | \boldsymbol{W}_{h}^{1})$~\citep{shallow_exploration}, $Q_{h}^{1} = (\boldsymbol{\theta}_{h}^{1})^{\mathsf{T}} \phi(\cdot, \cdot | \boldsymbol{W}_{h}^{1}) $, and the target value-networks $\bar{Q}_{h}^{1} = Q_{h}^{1}$, for $\forall h \in \left[ H \right]$

   \vspace{5pt}
   \For{episode $m=1$ {\bfseries to} $M$}

   \State Sample the initial state of the episode $s_{1}^{m}$
   
   \For{step $h=1, 2, ..., H$}
   
   \State Execute action $a_{h}^{m} = \arg \max_{a} Q^{m}_{h} (s_{h}^{m}, a)$ and get the next state $s_{h+1}^{m}$ and reward $r_{h}^{m}$

   \vspace{5pt}
   \State Compute the target value $q_{h}^{m} = r_{h}^{m} + \gamma \cdot \max_{a} \bar{Q}_{h+1}^{m} (s_{h+1}^{m}, a)$

   \vspace{5pt}
   \State Update $\boldsymbol{A}_{h}^{m+1} = \boldsymbol{A}_{h}^{m} + \boldsymbol{\phi}_{h}^{m} (\boldsymbol{\phi}_{h}^{m})^{\mathsf{T}}$ and $\boldsymbol{b}_{h}^{m+1} = \boldsymbol{b}_{h}^{m} + q_{h}^{m} \boldsymbol{\phi}_{h}^{m} $

   \vspace{5pt}
   \State Update parameter $\boldsymbol{\theta}_{h}^{m+1} = (\boldsymbol{A}_{h}^{m+1})^{-1} \boldsymbol{b}_{h}^{m+1}$

   \vspace{2pt}
   \State Update the embedding network to $\phi (\cdot, \cdot | \boldsymbol{W}_{h}^{m+1})$ by performing a gradient step $\boldsymbol{W}_{h}^{m+1} = \boldsymbol{W}_{h}^{m} - \eta \nabla _{\boldsymbol{W}_{h}^{m}} L_{h}^{m}$ where
    
    \begin{align}
    L_{h}^{m} = \sum_{t=1}^{m} \left| (\boldsymbol{\theta}_{h}^{m+1})^{\mathsf{T}} \phi (s_{h}^{t}, a_{h}^{t} | \boldsymbol{W}_{h}^{m}) - r_{h}^{t} - \gamma \cdot \max_{a} \bar{Q}_{h+1}^{m} (s_{h+1}^{t}, a) \right|^{2} \notag
    \end{align}

    \vspace{5pt}
    \State \textcolor{red}{Obtain UCB-based uncertainty}
    \begin{align}
        \textcolor{red}{\beta_{h}^{m+1} (\cdot, \cdot) = \sqrt{(\phi (\cdot, \cdot | \boldsymbol{W}_{h}^{m+1}))^{\mathsf{T}} (\boldsymbol{A}_{h}^{m+1})^{-1} \phi(\cdot, \cdot | \boldsymbol{W}_{h}^{m+1})}} \notag
    \end{align}

    \State \textcolor{blue}{Obtain Thompson Sampling-based uncertainty}
    \begin{align}
        \textcolor{blue}{\boldsymbol{\Delta \theta}_{h}^{m+1} \sim N(0, (\boldsymbol{A}_{h}^{m+1})^{-1}) \Longrightarrow \beta_{h}^{m+1} (\cdot, \cdot) = (\boldsymbol{\Delta \theta}_{h}^{m+1})^{\mathsf{T}} \phi(\cdot, \cdot | \boldsymbol{W}_{h}^{m+1})} \notag
    \end{align}

    \State Approximate the action-value function
    \begin{align}
        Q_{h}^{m+1} (\cdot, \cdot) = (\boldsymbol{\theta}_{h}^{m+1})^{\mathsf{T}} \phi(\cdot, \cdot | \boldsymbol{W}_{h}^{m+1}) + \alpha \beta_{h}^{m+1} (\cdot, \cdot)\notag
    \end{align}

   \EndFor
   \State Update the target network $\bar{Q}_{h}^{m+1} (\cdot, \cdot) = Q_{h}^{m+1} (\cdot, \cdot), h \in \left[ H \right]$
   \EndFor

\end{algorithmic}
\end{algorithm}

Notably, $Q_{h}^{m} (s_{h}^{m}, a_{h}^{m})$ denotes the estimated value, whereas $q_{h}^{m}$ represents the corresponding target value at each time step $h$ in episode $m$. By concatenating the values over all time steps $h \in \left[ H \right]$ and episodes $m \in \left[ M \right]$, we construct the estimated value vector $\tilde{\boldsymbol{q}}$ and the target value vector $\boldsymbol{q}$ of the action-value network, as referenced in Theorem~\ref{theorem:cae_regret}.

\newpage
\section{Proof of Theorem~\ref{theorem:cae_regret}}
\label{app:proof}

In this section, we analyze the cumulative regret bound of Algorithm~\ref{alg:theoretical_algorithm}. As $\boldsymbol{U}$ is a straightforward linear transformation of the embeddings, it approximately preserves the theoretical guarantees of UCB-based and Thompson Sampling-based exploration strategies, thereby maintaining the rigor of \caep.

Before delving into the detailed theory, we first review the notation used in this appendix. Let $\pi^{*}$ denote the true optimal policy, and let $\pi_{m}$ denote the policy executed in episode $m \in \left[ M \right]$. More generally, for any policy—illustrated here with $\pi_{m}$—the relationship between the action-value function $Q^{\pi_{m}}$ and the corresponding maximum return $V^{\pi_{m}}$ can be expressed by:

\begin{align}
& V_{h}^{\pi_{m}} (s) = \max_{a} Q_{h}^{\pi_{m}} (s, a) \\[5pt]
& Q_{h}^{\pi_{m}} (s, a) = r_{h}(s, a) + \mathbb{E}_{s_{h+1} \sim \mathbb{P}_{h} (\cdot | s, a)} V_{h+1}^{\pi_{m}} (s_{h+1})
\end{align}

In Algorithm~\ref{alg:theoretical_algorithm}, the estimated action-value function at step $h$ in episode $m$ is denoted by $Q_{h}^{m}(s, a)$, with the corresponding state-value function represented as $V_{h}^{m}(s)$. For clarity of presentation, we introduce the following additional notations.

\begin{align}
    & (\mathbb{P}_{h} V_{h+1}^{m})(s_{h}^{m}, a_{h}^{m}) = \mathbb{E}_{s_{h+1}^{m} \sim \mathbb{P}_{h} (\cdot | s_{h}^{m}, a_{h}^{m})} V_{h+1}^{m} (s_{h+1}^{m}) \\[5pt]
    & \delta_{h}^{m} (s_{h}^{m}, a_{h}^{m}) = r_{h}^{m} + (\mathbb{P}_{h} V_{h+1}^{m})(s_{h}^{m}, a_{h}^{m}) - Q_{h}^{m} (s_{h}^{m}, a_{h}^{m}) \\[5pt]
    & \zeta _{h}^{m} = \Big[ V_{h}^{m} (s_{h}^{m}) - V_{h}^{\pi_{m}} (s_{h}^{m}) \Big] - \Big[ Q_{h}^{m} (s_{h}^{m}, a_{h}^{m}) - Q_{h}^{\pi_{m}} (s_{h}^{m}, a_{h}^{m}) \Big] \\[5pt]
    & \varepsilon_{h}^{m} = \Big[ (\mathbb{P}_{h} V_{h+1}^{m}) (s_{h}^{m}, a_{h}^{m}) - (\mathbb{P}_{h} V_{h+1}^{\pi_{m}}) (s_{h}^{m}, a_{h}^{m}) \Big] - \Big[ V_{h+1}^{m} (s_{h+1}^{m}) - V_{h+1}^{\pi_{m}} (s_{h+1}^{m}) \Big]
\end{align}

\vspace{5pt}
Specifically, $\delta_{h}^{m} (s_{h}^{m}, a_{h}^{m})$ represents the temporal-difference error for the state-action pair $(s_{h}^{m}, a_{h}^{m})$. The notations $\zeta _{h}^{m}$ and $\varepsilon_{h}^{m}$ capture two sources of randomness, \ie, the selection of action $a_{h}^{m} \sim \pi_{m} (\cdot | s_{h}^{m})$ and the generation of the next state $s_{h+1}^{m} \sim \mathbb{P}_{h} (\cdot | s_{h}^{m}, a_{h}^{m})$ from the environment.

\begin{proof} \textbf{Theorem~\ref{theorem:cae_regret}}.

Based on Lemma~\ref{lemma:org_regret}, the cumulative regret in Equation~\ref{eq:regret} can be decomposed into three terms as follows, where $\left\langle \cdot, \cdot \right\rangle$ means the inner product of two vectors.

\begin{align}
    \notag
    \textup{R}_{M} = & \sum_{m=1}^{M} Q_{1}^{*} (s_{1}^{m}, u_{1}^{m}) - Q_{1}^{\pi_{m}} (s_{1}^{m}, a_{1}^{m}) \\ \notag
    = & \sum_{m=1}^{M} \sum_{h=1}^{H} \Big [ \mathbb{E}_{\pi^{*}} \left [ \delta_{h}^{m} (s_{h}, a_{h}) | s_{1} = s_{1}^{m} \right ] - \delta_{h}^{m} (s_{h}^{m}, a_{h}^{m}) \Big ] + \sum_{m=1}^{M} \sum_{h=1}^{H} (\zeta_{h}^{m} + \varepsilon _{h}^{m}) \\ \notag 
    & + \sum_{m=1}^{M} \sum_{h=1}^{H} \mathbb{E}_{\pi^{*}} \left [ \left \langle Q_{h}^{m} (s_{h}, \cdot), \pi_{h}^{*}(\cdot | s_{h}) - \pi_{m} (\cdot| s_{h}) \right \rangle | s_{1} = s_{1}^{m} \right ]
\end{align}

According to the definition of $\pi_{m}$, there is Equation~\ref{eq:alg_optimal}.

\begin{align}
    \left \langle Q_{h}^{m} (s_{h}, \cdot), \pi_{h}^{*}(\cdot | s_{h}) - \pi_{m} (\cdot| s_{h}) \right \rangle \leq 0
    \label{eq:alg_optimal}
\end{align}

Consequently, with probability at least $1 - \sigma$ for $ \sigma \in (0, 1)$, the cumulative regret in Equation~\ref{eq:regret} can be bounded as follows.

\begin{small}
\begin{align}
\label{eq:decompose_regret}
    \textup{R}_{M} \leq & \sum_{m=1}^{M} \sum_{h=1}^{H} \Big [ \mathbb{E}_{\pi^{*}} \left [ \delta_{h}^{m} (s_{h}, a_{h}) | s_{1} = s_{1}^{m} \right ] - \delta_{h}^{m} (s_{h}^{m}, a_{h}^{m}) \Big ] + \sum_{m=1}^{M} \sum_{h=1}^{H} (\zeta_{h}^{m} + \varepsilon _{h}^{m}) \\ \notag
    \leq & \sum_{m=1}^{M} \sum_{h=1}^{H} \Big [ \mathbb{E}_{\pi^{*}} \left [ \delta_{h}^{m} (s_{h}, a_{h}) | s_{1} = s_{1}^{m} \right ] - \delta_{h}^{m} (s_{h}^{m}, a_{h}^{m}) \Big ] + \sqrt{16MH^{3} \textup{log} \frac{2}{\sigma_{1}}} \\ \notag
    \leq & 4H \sqrt{ MH \log \frac{2}{\sigma_{2}}} + C_{2} \alpha H \sqrt{M D \cdot \textup{log} (1 + \frac{M}{\lambda D})} \\[5pt] \notag
    & + \frac{C_{3} \cdot H L^{3} D^{\frac{5}{2}} M \sqrt{\log \iota \log \left( \frac{1}{\sigma_2} \right) \log \left( \frac{M \left| \mathcal{A} \right|}{\sigma_2} \right)} \left \| \boldsymbol{q} - \tilde{\boldsymbol{q}} \right \|_{\boldsymbol{H}^{-1}}}{\iota^{\frac{1}{6}}} + \sqrt{16MH^{3} \textup{log} \frac{2}{\sigma_{1}}} \\ \notag
    \leq & C_{4} H \sqrt{MH \log \frac{1}{\sigma}} + C_{2} \alpha H \sqrt{M D \cdot \textup{log} (1 + \frac{M}{\lambda D})} + \frac{C_{3} H L^{3} D^{\frac{5}{2}} M \sqrt{\log \iota \log \left( \frac{1}{\sigma} \right) \log \left( \frac{M \left| \mathcal{A} \right|}{\sigma} \right)} \left \| \boldsymbol{q} - \tilde{\boldsymbol{q}} \right \|_{\boldsymbol{H}^{-1}}}{\iota^{\frac{1}{6}}}
\end{align}
\end{small}

Here, $\sigma_1, \sigma_2 \in (0,1)$. Specifically, the second inequality follows from Lemma~\ref{lemma:2_term}, and the third inequality follows from Lemma~\ref{lemma:1_term}. By setting $\sigma_1=\sigma_2=\frac{\sigma}{2}$ and applying a union bound, the result holds after absorbing constant factors into $C_2$, $C_3$, and $C_4$.

\end{proof}

\section{Lemmas}
\label{app:lemma}
\begin{lemma}
\label{lemma:org_regret}
    Adapted from Lemma 5.1 of~\citet{provable_kernel_nn}, the regret in Equation~\ref{eq:regret} can be decomposed as Equation~\ref{eq:org_regret}, where $\left\langle \cdot, \cdot \right\rangle$ means the inner product of two vectors.

\begin{align}
    \notag
    \textup{R}_{M} = & \sum_{m=1}^{M} Q_{1}^{*} (s_{1}^{m}, u_{1}^{m}) - Q_{1}^{\pi_{m}} (s_{1}^{m}, a_{1}^{m}) \\ \notag
    = & \sum_{m=1}^{M} V_{1}^{*} (s_{1}^{m}) - V_{1}^{\pi_{m}} (s_{1}^{m}) \\
    \label{eq:org_regret}
    = & \sum_{m=1}^{M} \sum_{h=1}^{H} \Big [ \mathbb{E}_{\pi^{*}} \left [ \delta_{h}^{m} (s_{h}, a_{h}) | s_{1} = s_{1}^{m} \right ] - \delta_{h}^{m} (s_{h}^{m}, a_{h}^{m}) \Big ] + \sum_{m=1}^{M} \sum_{h=1}^{H} (\zeta_{h}^{m} + \varepsilon _{h}^{m}) \\ \notag 
    & + \sum_{m=1}^{M} \sum_{h=1}^{H} \mathbb{E}_{\pi^{*}} \left [ \left \langle Q_{h}^{m} (s_{h}, \cdot), \pi_{h}^{*}(\cdot | s_{h}) - \pi_{m} (\cdot| s_{h}) \right \rangle | s_{1} = s_{1}^{m} \right ]
\end{align}
\end{lemma}



\vspace{5pt}
\begin{lemma}
\label{lemma:2_term} 
    Adapted from Lemma 5.3 of~\citet{provable_kernel_nn}, with probability at least $1 - \sigma_{1}$, the second term in Equation~\ref{eq:decompose_regret} can be bounded as follows:

    \begin{equation}
    \sum_{m=1}^{M} \sum_{h=1}^{H} (\zeta_{h}^{m} + \varepsilon _{h}^{m}) \leq \sqrt{16MH^{3} \textup{log} \frac{2}{\sigma_{1}}}
\end{equation}

\end{lemma}


\begin{lemma}
\label{lemma:1_term}
For any $\sigma_{2} \in (0, 1)$, assume the width of the action-value network satisfies:

\begin{align}
    \iota = \textup{poly} (L, D, \frac{1}{\sigma_{2}}, \log \frac{M |\mathcal{A}|}{\sigma_{2}})
\end{align}

where $L$ is the number of layers in the action-value network, and $\textup{poly} (\cdot)$ means a polynomial function depending on the incorporated variables, and let:
    
    \begin{align}
        & \alpha = \sqrt{2 (D \cdot \log (1 + \frac{M \cdot \textup{log} |\mathcal{A}|}{\lambda}) - \log \sigma_{2})} + \sqrt{\lambda} \\
        & \eta \leq C_{1} (\iota \cdot D^{2} M^{\frac{11}{2}} L^{6} \cdot \log \frac{M |\mathcal{A}|}{\sigma_{2}})^{-1}
    \end{align}

then with probability at least $1 - \sigma_{2}$, the first term in Equation~\ref{eq:decompose_regret} is bounded as:

\begin{small}
\begin{align}
    & \sum_{m=1}^{M} \sum_{h=1}^{H} \Big [ \mathbb{E}_{\pi^{*}} \left [ \delta_{h}^{m} (s_{h}, a_{h}) | s_{1} = s_{1}^{m} \right ] - \delta_{h}^{m} (s_{h}^{m}, a_{h}^{m}) \Big ] \\[5pt] \notag
    \leq & 4H \sqrt{ MH \log \frac{2}{\sigma_{2}}} + C_{2} \alpha H \sqrt{M D \cdot \textup{log} (1 + \frac{M}{\lambda D})} + \frac{C_{3} \cdot H L^{3} D^{\frac{5}{2}} M \sqrt{\log \iota \log \left( \frac{1}{\sigma_2} \right) \log \left( \frac{M \left| \mathcal{A} \right|}{\sigma_2} \right)} \left \| \boldsymbol{q} - \tilde{\boldsymbol{q}} \right \|_{\boldsymbol{H}^{-1}}}{\iota^{\frac{1}{6}}}
\end{align}
\end{small}
\end{lemma}

\begin{proof}
According to~\citet{provable_kernel_nn}, there is:

\begin{align}
    \sum_{m=1}^{M} \sum_{h=1}^{H} \Big [ \mathbb{E}_{\pi^{*}} \left [ \delta_{h}^{m} (s_{h}, a_{h}) | s_{1} = s_{1}^{m} \right ] - \delta_{h}^{m} (s_{h}^{m}, a_{h}^{m}) \Big ] \leq \sum_{m=1}^{M} \sum_{h=1}^{H} - \delta_{h}^{m} (s_{h}^{m}, a_{h}^{m})
\end{align}

Considering $\delta_{h}^{m} (s_{h}^{m}, a_{h}^{m})$, it can be decomposed as:

\begin{align}
    \delta_{h}^{m} (s_{h}^{m}, a_{h}^{m}) = & r_{h}^{m} + (\mathbb{P}_{h} V_{h+1}^{m}) (s_{h}^{m}, a_{h}^{m}) - Q_{h}^{m} (s_{h}^{m}, a_{h}^{m}) \\[5pt] \notag
    = & r_{h}^{m} + (\mathbb{P}_{h} V_{h+1}^{m}) (s_{h}^{m}, a_{h}^{m}) - Q_{h}^{*} (s_{h}^{m}, a_{h}^{m}) + Q_{h}^{*} (s_{h}^{m}, a_{h}^{m}) - Q_{h}^{m} (s_{h}^{m}, a_{h}^{m}) \\[5pt] \notag
    = & \mathbb{P}_{h} (V_{h+1}^{m} - V_{h+1}^{*})(s_{h}^{m}, a_{h}^{m}) + (Q_{h}^{*} - Q_{h}^{m}) (s_{h}^{m}, a_{h}^{m}) \\[5pt] \notag
    = & \underbrace{\mathbb{P}_{h} (V_{h+1}^{m} - V_{h+1}^{*})(s_{h}^{m}, a_{h}^{m}) - (V_{h+1}^{m} - V_{h+1}^{*})(s_{h+1}^{m})}_{\omega_{h}^{m}} \\[5pt] \notag 
    & + \underbrace{(V_{h+1}^{m} - V_{h+1}^{*})(s_{h+1}^{m})}_{\rho_{h+1}^{m}} + \underbrace{(Q_{h}^{*} - Q_{h}^{m}) (s_{h}^{m}, a_{h}^{m})}_{\varphi_{h}^{m}}
\end{align}

By Azuma-Hoeffding inequality, we can bound $\sum_{m=1}^{M} \sum_{h=1}^{H} \omega_{h}^{m}$ as Equation~\ref{eq:hoeffding} with probability at least $1 - \sigma_{3}$ where $\sigma_{3} \in (0, 1)$.

\begin{align}
\label{eq:hoeffding}
    - 2H \sqrt{ MH \log \frac{2}{\sigma_{3}}} \leq \sum_{m=1}^{M} \sum_{h=1}^{H} \omega_{h}^{m} \leq 2H \sqrt{ MH \log \frac{2}{\sigma_{3}}}
\end{align}

As $\rho_{h+1}^{m}$ can be decomposed as Equation~\ref{eq:rho} where $u_{h+1}^{m} \sim \pi_{h+1}^{*} (\cdot| s_{h+1}^{m})$, there is Equation~\ref{eq:sub_delta}.

\begin{align}
\label{eq:rho}
    & \rho_{h+1}^{m} = (V_{h+1}^{m} - V_{h+1}^{*}) (s_{h+1}^{m}) = Q_{h+1}^{m}(s_{h+1}^{m}, a_{h+1}^{m}) - Q_{h+1}^{*} (s_{h+1}^{m}, u_{h+1}^{m}) \\[5pt]
    \label{eq:sub_delta}
    \Rightarrow  & \sum_{m=1}^{M} \sum_{h=1}^{H} (\rho_{h+1}^{m} + \varphi_{h}^{m}) \\[5pt] \notag
    = & \sum_{m=1}^{M} \sum_{h=1}^{H-1} \left[ Q_{h+1}^{m}(s_{h+1}^{m}, a_{h+1}^{m}) - Q_{h+1}^{*} (s_{h+1}^{m}, u_{h+1}^{m}) \right] + \sum_{m=1}^{M} \sum_{h=1}^{H} (Q_{h}^{*} - Q_{h}^{m}) (s_{h}^{m}, a_{h}^{m}) \\[5pt] 
    \label{eq:MAB_loss}
    = & \underbrace{\sum_{m=1}^{M} \sum_{h=2}^{H} Q_{h}^{*} (s_{h}^{m}, a_{h}^{m}) - Q_{h}^{*} (s_{h}^{m}, u_{h}^{m})}_{\textup{R}_{\textup{MAB}}} + \sum_{m=1}^{M} (Q_{1}^{*} - Q_{1}^{m}) (s_{1}^{m}, a_{1}^{m})
\end{align}

Specifically, the second equation follows from $Q_{H+1}^{*} (s_{H+1}^{m}, a_{H+1}^{m}) = 0$ and $Q_{H+1}^{m} (s_{H+1}^{m}, a_{H+1}^{m}) = 0$. The second term in Equation~\ref{eq:MAB_loss} originates from the estimation error of the action-value function, which is constrained by the convergence properties of DQN. Consequently, to complete the proof of Lemma~\ref{lemma:1_term}, it suffices to establish a bound for the $\textup{R}_{\textup{MAB}}$ term, while the second term is omitted for conciseness in the remaining discussion. Bounds of $\textup{R}_{\textup{MAB}}$ under UCB- and Thompson Sampling-based exploration strategies are proved in Lemma~\ref{lemma:ucb_regret} and Lemma~\ref{lemma:ts_regret}, respectively. 

By setting $C_{2} = \max\{C_{\text{ucb}}, C_{\text{ts}}\}$, $\sigma_3 = \sigma_4 = \frac{\sigma_2}{2}$, applying a union bound, and absorbing constant factors into $C_2$ and $C_3$, the proof is completed.

\end{proof}

\vspace{10pt}
\subsection{Regret Bound of UCB-based Exploration}
\label{app:ucb_bound}

In this subsection, we first introduce the third assumption from \textit{deep representation and shallow exploration}~\citep{shallow_exploration}, which, with Assumptions~\ref{ass:ass_1} and~\ref{ass:ass_3}, forms the standard assumptions. We then present Lemma~\ref{lemma:ucb_regret}.
\begin{assumption}
    \label{ass:ass_2}
    For $\forall s_{1}, s_{2} \in \mathcal{S}$ and $\forall a_{1}, a_{2} \in \mathcal{A}$, there is a constant $l_{Lip} > 0$, such that:
    \begin{align}
    \label{eq:lip}
        \left\| \nabla _{\boldsymbol{W}} \phi (s_{1}, a_{1}|\boldsymbol{W}_{h}^{1}) - \nabla _{\boldsymbol{W}} \phi (s_{2}, a_{2}|\boldsymbol{W}_{h}^{1}) \right\|_{2} \leq l_{Lip} \left\| (s_{1}; a_{1}) - (s_{2}; a_{2}) \right\|_{2}
    \end{align}
\end{assumption}

\vspace{5pt}

\begin{lemma}
\label{lemma:ucb_regret}
Adapted from Theorem 4.4 of~\citet{shallow_exploration}, suppose the standard initializations and assumptions hold. Additionally, assume without loss of generality that $\left \| \boldsymbol{\theta}^{*} \right \|_{2} \leq 1$ 
and $\left \| \phi(s_{h}, a_{h}) \right \|_{2} \leq 1$. If with the UCB-based exploration strategy, then for any $\sigma_{4} \in (0, 1)$, let the width of the action-value network satisfies:

\begin{align}
    \iota = \textup{poly} (L, D, \frac{1}{\sigma_{4}}, \log \frac{M |\mathcal{A}|}{\sigma_{4}})
\end{align}

where $L$ is the number of layers in the action-value network, and $\textup{poly} (\cdot)$ means a polynomial function depending on the incorporated variables, and let:
    
    \begin{align}
    \label{eq:alpha}
        & \alpha = \sqrt{2 (D \cdot \log (1 + \frac{M \cdot \textup{log} |\mathcal{A}|}{\lambda}) - \log \sigma_{4})} + \sqrt{\lambda} \\
        & \eta \leq C_{1} (\iota \cdot D^{2} M^{\frac{11}{2}} L^{6} \cdot \log \frac{M |\mathcal{A}|}{\sigma_{4}})^{-1}
    \end{align}

    then with probability at least $1-\sigma_{4}$, the term $\textup{R}_{\textup{UCB}}$ in Equation~\ref{eq:MAB_loss} can be bounded as follows:

    \begin{small}
    \begin{align}
        \textup{R}_{\textup{UCB}} \leq & C_{\textup{ucb}} \cdot \alpha H \sqrt{M D \cdot \textup{log} (1 + \frac{M}{\lambda D})} + \frac{C_{3} H L^{3} D^{\frac{5}{2}} M \sqrt{\log \iota \log \left( \frac{1}{\sigma_4} \right) \log \left( \frac{M \left| \mathcal{A} \right|}{\sigma_4} \right)} \left \| \boldsymbol{q} - \tilde{\boldsymbol{q}} \right \|_{\boldsymbol{H}^{-1}}}{\iota^{\frac{1}{6}}}
    \end{align}
    \end{small}

    where $C_{1}, C_{\textup{ucb}}, C_{3}$ are constants independent of the problem parameters; $\boldsymbol{q} = (q_{1}^{1}; q_{2}^{1}; \ldots; q_{H}^{1}; \ldots; q_{H}^{M})$ and $\tilde{\boldsymbol{q}} = (Q_{1}^{1} (s_{1}^{1}, a_{1}^{1}); \ldots; Q_{H}^{1} (s_{H}^{1}, a_{H}^{1}); ...; Q_{H}^{M} (s_{H}^{M}, a_{H}^{M}))$ are the target and estimated value vectors, respectively.
    
\end{lemma}

Notably, the proof of the above lemma relies on the concentration of self-normalized stochastic processes. However, since $Q_{h}^{m}$ is not independent of historical data, this result cannot be directly applied. Instead, a similar approach to that in~\citet{provable_kernel_nn} is adopted by leveraging Uniform Convergence across all possible inputs within the value function class $\mathcal{Q}$. This ensures that the maximum deviation between the true and learned values over all time steps remains small, \ie, $\sup_{Q_{h}^{m} \in \mathcal{Q}} \left| Q_h^m - Q_h^{*} \right| \leq \epsilon_m$ for $\forall h \in \left[ H \right]$. Crucially, the error $\epsilon_m$ decreases as the number of episodes increases. Applying triangle inequality, we decompose the error bound $\textup{R}_{\textup{UCB}}$ into two terms to manage the dependency:

\begin{itemize}[leftmargin=1em]
    \item A true optimal value function component, to which the concentration of self-normalized stochastic processes applies.
    \item A cumulative error term dependent on $\epsilon_m$, which can be systematically bounded.
\end{itemize}

\vspace{10pt}
\subsection{Regret Bound of Thompson Sampling-based Exploration}
\label{app:ts_bound}

\begin{lemma}
\label{lemma:ts_regret}
Under the same settings as those of Lemma~\ref{lemma:ucb_regret}, if with the Thompson Sampling-based exploration strategy, the term $\textup{R}_{\textup{Thompson Sampling}}$ in Equation~\ref{eq:MAB_loss} can be bounded as Equation~\ref{eq:ts_bound}, where $C_{\textup{ts}}$ is a constant.
    \begin{small}
    \begin{align}
    \label{eq:ts_bound}
        \textup{R}_{\textup{Thompson Sampling}} \leq & C_{\textup{ts}} \cdot \alpha H \sqrt{M D \cdot \textup{log} (1 + \frac{M}{\lambda D})} + \frac{C_{3} H L^{3} D^{\frac{5}{2}} M \sqrt{\log \iota \log \left( \frac{1}{\sigma_4} \right) \log \left( \frac{M \left| \mathcal{A} \right|}{\sigma_4} \right)} \left \| \boldsymbol{q} - \tilde{\boldsymbol{q}} \right \|_{\boldsymbol{H}^{-1}}}{\iota^{\frac{1}{6}}}
    \end{align}
    \end{small}

\end{lemma}

\begin{proof}
According to Lemma A.1 of~\citet{shallow_exploration}, there exits $\boldsymbol{W}_h^{\#}$ such that $Q_{h}^{*}(s, u) - Q_{h}^{*} (s, a)$ can be decomposed as Equation~\ref{eq:ts_bound_}, where $g (s, a; \boldsymbol{W}) = \nabla _{\boldsymbol{W}} \phi(s, a; \boldsymbol{W})$.

\begin{align}
\label{eq:ts_bound_}
    & Q_{h}^{*}(s, u) - Q_{h}^{*} (s, a) \\[5pt] \notag
    = & (\boldsymbol{\theta}_{h}^{*})^{\mathsf{T}} \left[ \phi (s, u; \boldsymbol{W}_{h}^{m}) - \phi (s, a; \boldsymbol{W}_{h}^{m}) \right] + (\boldsymbol{\theta}_{h}^{1})^{\mathsf{T}} \left[ g (s, u; \boldsymbol{W}_{h}^{1}) - g (s, a; \boldsymbol{W}_{h}^{1}) \right] (\boldsymbol{W}_{h}^{\#} - \boldsymbol{W}_{h}^{m}) \\[5pt] \notag
    = & (\boldsymbol{\theta}_{h}^{1})^{\mathsf{T}} \left[ g (s, u; \boldsymbol{W}_{h}^{1}) - g (s, a; \boldsymbol{W}_{h}^{1}) \right] (\boldsymbol{W}_{h}^{\#} - \boldsymbol{W}_{h}^{m}) \\[5pt] \notag
    & + \underbrace{(\boldsymbol{\theta}_{h}^{m})^{\mathsf{T}} \left[ \phi (s, u; \boldsymbol{W}_{h}^{m}) - \phi (s, a; \boldsymbol{W}_{h}^{m}) \right]}_{\vartheta_{h}^{m}} - (\boldsymbol{\theta}_{h}^{m} - \boldsymbol{\theta}_{h}^{*})^{\mathsf{T}} \left[ \phi (s, u; \boldsymbol{W}_{h}^{m}) - \phi (s, a; \boldsymbol{W}_{h}^{m}) \right]
\end{align}

Based on the action selection process using Thompson Sampling-based exploration strategy in Algorithm~\ref{alg:theoretical_algorithm}, we derive Equation~\ref{eq:ts_selection}.

\begin{align}
\label{eq:ts_selection}
    (\boldsymbol{\theta}_{h}^{m} + \alpha \Delta \boldsymbol{\theta}_{h}^{m})^{\mathsf{T}} \phi (s, u; \boldsymbol{W}_{h}^{m}) \leq (\boldsymbol{\theta}_{h}^{m} + \alpha \Delta \boldsymbol{\theta}_{h}^{m})^{\mathsf{T}} \phi (s, a; \boldsymbol{W}_{h}^{m})
\end{align}

Consequently, for any $\sigma_{5} \in (0, 1)$, with probability at least $1-\sigma_5$, the term $\vartheta_{h}^{m}$ is bounded as Equation~\ref{eq:vartheta_bound}.

\begin{align}
\label{eq:vartheta_bound}
\vartheta_{h}^{m} \leq & \left\| \Delta \boldsymbol{\theta}_{h}^{m} \right\|_{\boldsymbol{A}_{h}^{m}} \left\| \phi (s, a; \boldsymbol{W}_{h}^{m}) - \phi (s, u; \boldsymbol{W}_{h}^{m}) \right\|_{(\boldsymbol{A}_{h}^{m})^{-1}} \\[5pt] \notag
\leq & (\sqrt{D} + \sqrt{2 \log \frac{1}{\sigma_{5}}}) \left\| \phi (s, a; \boldsymbol{W}_{h}^{m}) - \phi (s, u; \boldsymbol{W}_{h}^{m}) \right\|_{(\boldsymbol{A}_{h}^{m})^{-1}}
\end{align}

Specifically, the last inequality above is because $\boldsymbol{\Delta \theta}_{h}^{m} \sim N(0, (\boldsymbol{A}_{h}^{m})^{-1})$. Substituting the bound of $\vartheta_{h}^{m}$ back into Equation~\ref{eq:ts_bound_} further yields:

\begin{align}
\label{eq:ts_bound__}
     Q_{h}^{*}(s, u) - Q_{h}^{*} (s, a) \leq & (\boldsymbol{\theta}_{h}^{1})^{\mathsf{T}} \left[ g (s, u; \boldsymbol{W}_{h}^{1}) - g (s, a; \boldsymbol{W}_{h}^{1}) \right] (\boldsymbol{W}_{h}^{\#} - \boldsymbol{W}_{h}^{m}) \\[5pt] \notag
    & + (\sqrt{D} + \sqrt{2 \log \frac{1}{\sigma_{5}}}) \left\| \phi (s, a; \boldsymbol{W}_{h}^{m}) - \phi (s, u; \boldsymbol{W}_{h}^{m}) \right\|_{(\boldsymbol{A}_{h}^{m})^{-1}} \\[5pt] \notag 
    & - (\boldsymbol{\theta}_{h}^{m} - \boldsymbol{\theta}_{h}^{*})^{\mathsf{T}} \left[ \phi (s, u; \boldsymbol{W}_{h}^{m}) - \phi (s, a; \boldsymbol{W}_{h}^{m}) \right]
\end{align}

Comparing Equation~\ref{eq:ts_bound__} with Equation A.7 of~\citet{shallow_exploration}, the difference between the regrets of the Thompson Sampling-based and UCB-based exploration strategies is bounded as in Equation~\ref{eq:ts_ucb}.

\begin{align}
\label{eq:ts_ucb}
    \notag
    & \left| \textup{R}_{\textup{Thompson Sampling}} - \textup{R}_{\textup{UCB}} \right| \\[5pt] \notag
    \leq & \sum_{m=1}^{M} \sum_{h=1}^{H} (\sqrt{D} + \sqrt{2 \log \frac{1}{\sigma_{5}}}) \left\| \phi (s_h^m, a_h^m; \boldsymbol{W}_{h}^{m}) - \phi (s_h^m, u_h^m; \boldsymbol{W}_{h}^{m}) \right\|_{(\boldsymbol{A}_{h}^{m})^{-1}} \\[5pt] \notag 
    & + \sum_{m=1}^{M} \sum_{h=1}^{H} \alpha \left\| \phi (s_h^m, a_h^m; \boldsymbol{W}_{h}^{m}) \right\|_{(\boldsymbol{A}_{h}^{m})^{-1}} + \sum_{m=1}^{M} \sum_{h=1}^{H} \alpha \left\| \phi (s_h^m, u_h^m; \boldsymbol{W}_{h}^{m}) \right\|_{(\boldsymbol{A}_{h}^{m})^{-1}} \\[5pt]
    \leq & C_{5} \alpha H \sqrt{MD \cdot \textup{log} (1 + \frac{M}{\lambda D})}
\end{align}

Setting $C_{\text{ts}} = C_{\text{ucb}} + C_{5}$ completes the proof.
\end{proof}

Notably, the above proof applies a union bound, which requires assigning a constant to $\alpha$ in Equation~\ref{eq:alpha}. For brevity, we omit the allocation of this constant and the corresponding discussion of the failure probabilities.

\vspace{5pt}

\subsection{Discussion about the Standard Assumptions}
\label{app:assumptions}
Assumption~\ref{ass:ass_1} can be readily satisfied by transforming the state–action pairs $(s; a)$ into $(s; a;s; a)$ and applying an appropriate scaling.

Assumption~\ref{ass:ass_3} is a mild condition. Specifically, prior studies have demonstrated that for two-layer ReLU networks, this assumption can be directly derived from Asumption~\ref{ass:ass_1}. A diverse input distribution and a wide neural network can largely ensure that $\boldsymbol{H}$ remains positive definite.

Assumption~\ref{ass:ass_2} holds when the gradient of the neural network with respect to certain weights does not fluctuate excessively. In other words, within a neighborhood of a given state–action pair, the gradient with respect to these weights remains well-controlled. This is a standard assumption in the non-convex optimization literature, commonly used to ensure the convergence of alternating optimization procedures in which parameters are updated iteratively.

\newpage
\section{Experiment}
\label{app:experiment}

In this section, we present additional experimental results and provide the hyperparameters used to reproduce the experiments reported in this paper.

\begin{figure}[t]
  \centering
  \begin{subfigure}[t]{0.3\textwidth}
    \includegraphics[width=\textwidth, trim={0 0 2.5cm 1.5cm}, clip]{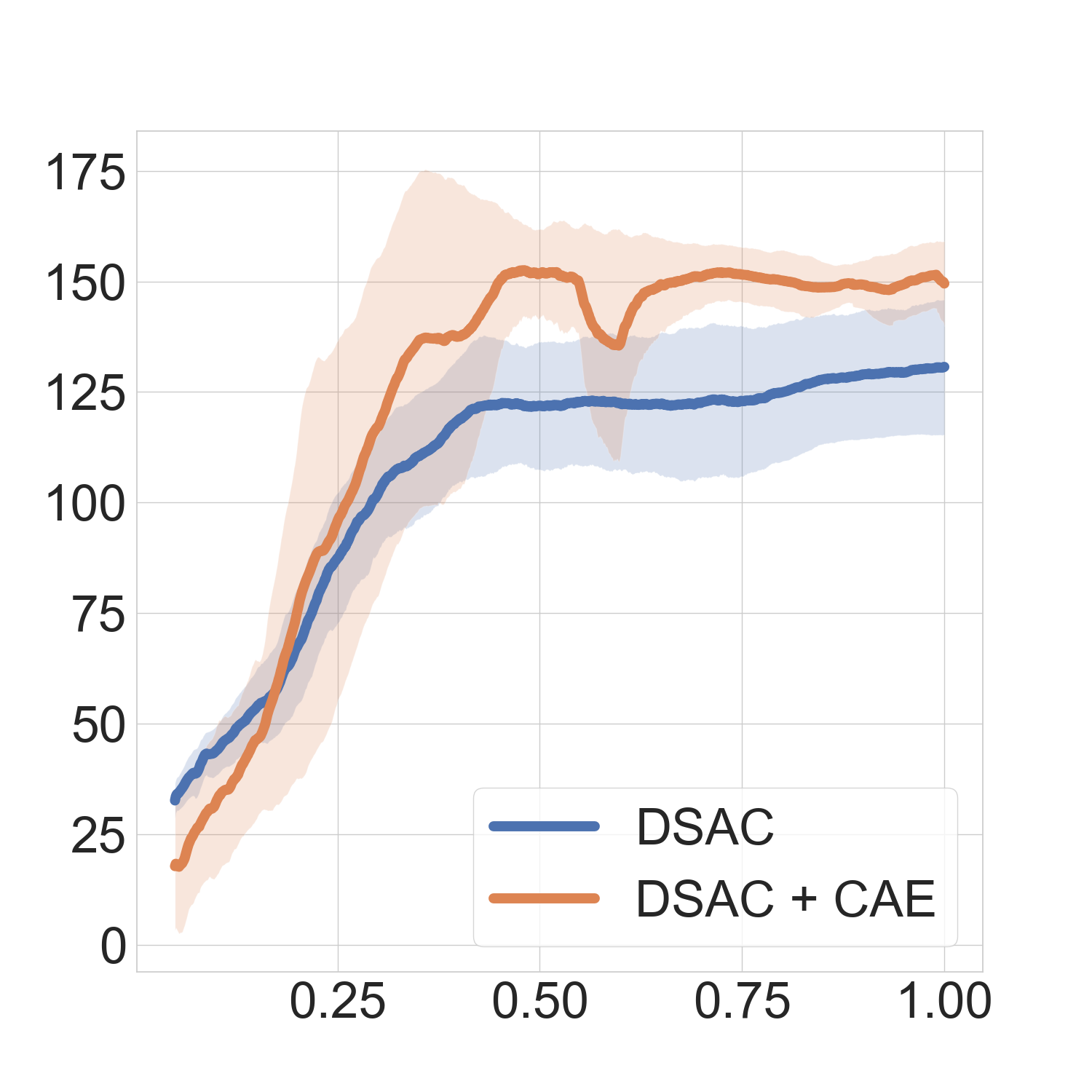}
    \caption{Swimmer}
    \label{fig:swimmer_dsac}
  \end{subfigure}
  \hfill
  \begin{subfigure}[t]{0.3\textwidth}
    \includegraphics[width=\textwidth, trim={0 0 2.5cm 1.5cm}, clip]{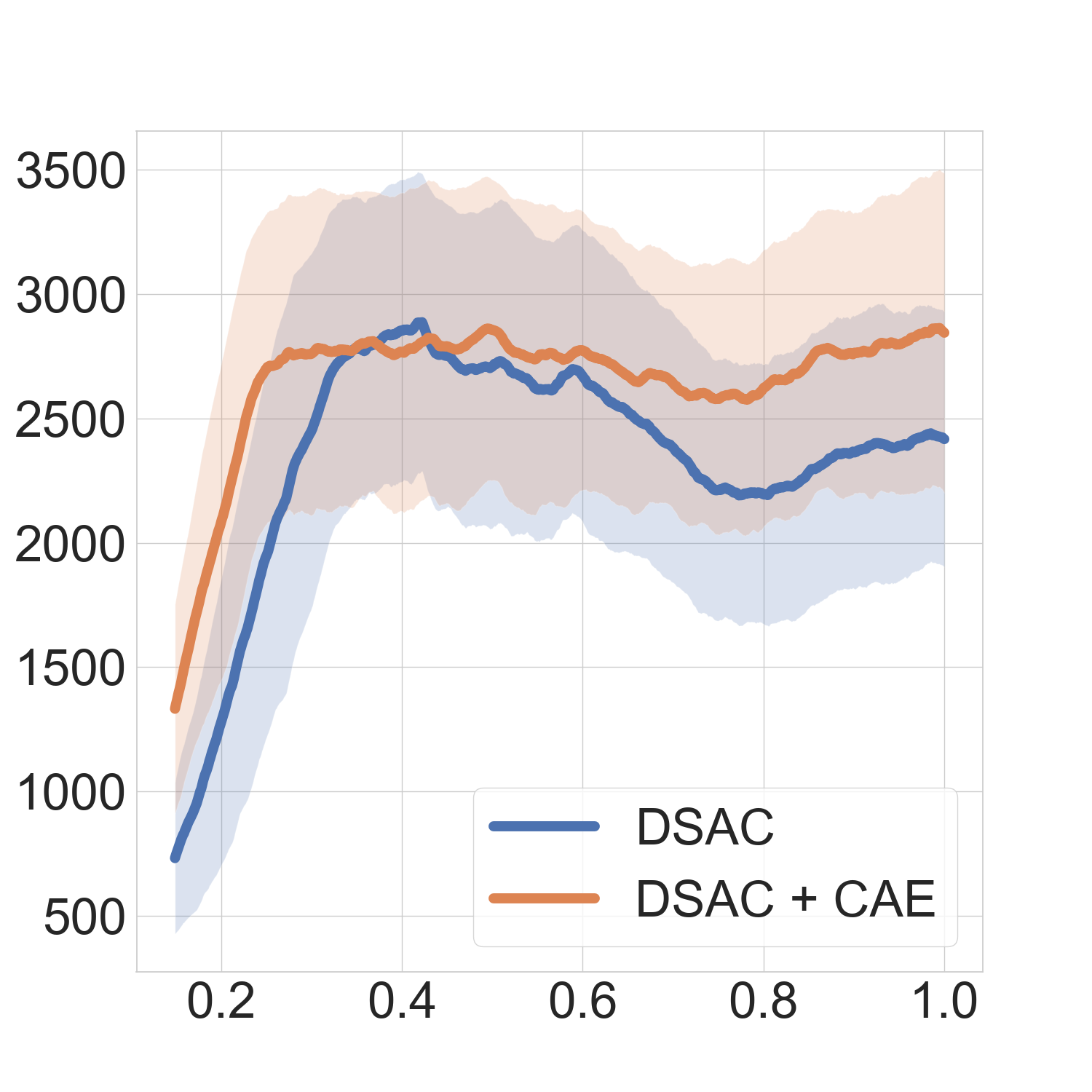}
    \caption{Hopper}
    \label{fig:hopper_dsac}
  \end{subfigure}
  \hfill
    \begin{subfigure}[t]{0.3\textwidth}
    \includegraphics[width=\textwidth, trim={0 0 2.5cm 1.5cm}, clip]{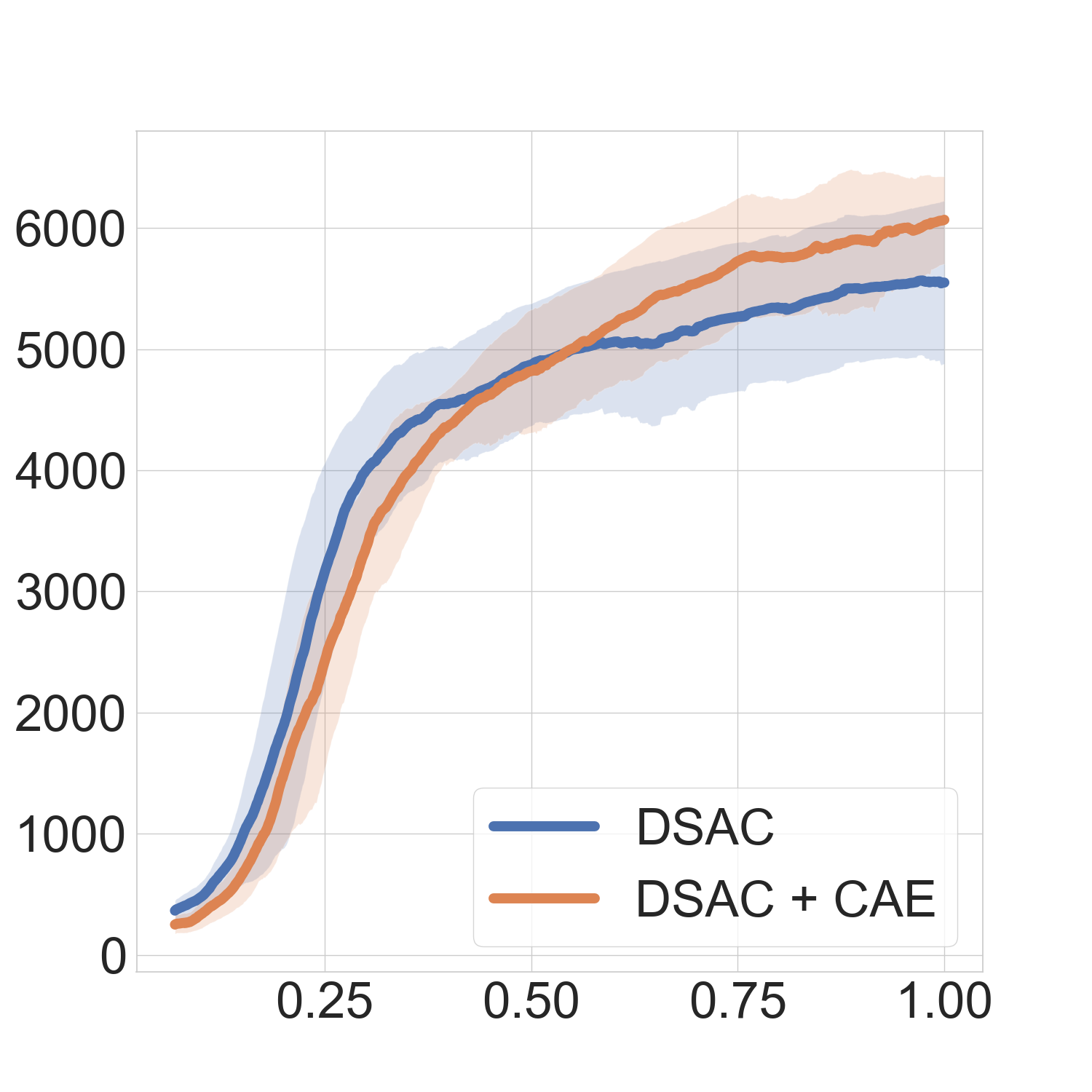}
    \caption{Walker2d}
    \label{fig:walker2d_dsac}
  \end{subfigure}
  \vskip\baselineskip
  \begin{subfigure}[t]{0.3\textwidth}
    \includegraphics[width=\textwidth, trim={0 0 2.5cm 1.5cm}, clip]{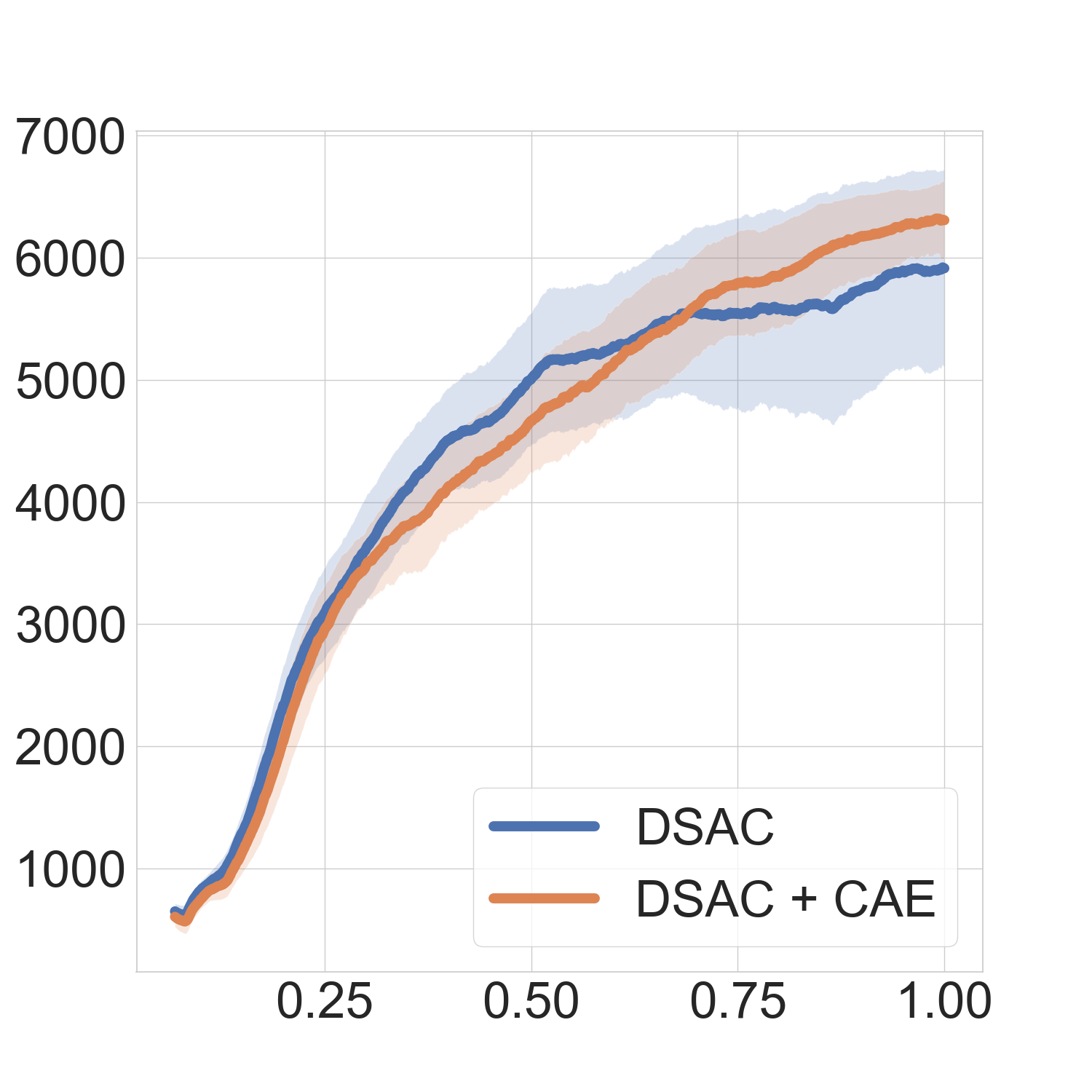}
    \caption{Ant}
    \label{fig:ant_dsac}
  \end{subfigure}
  \hfill
  \begin{subfigure}[t]{0.3\textwidth}
    \includegraphics[width=\textwidth, trim={0 0 2.5cm 1.5cm}, clip]{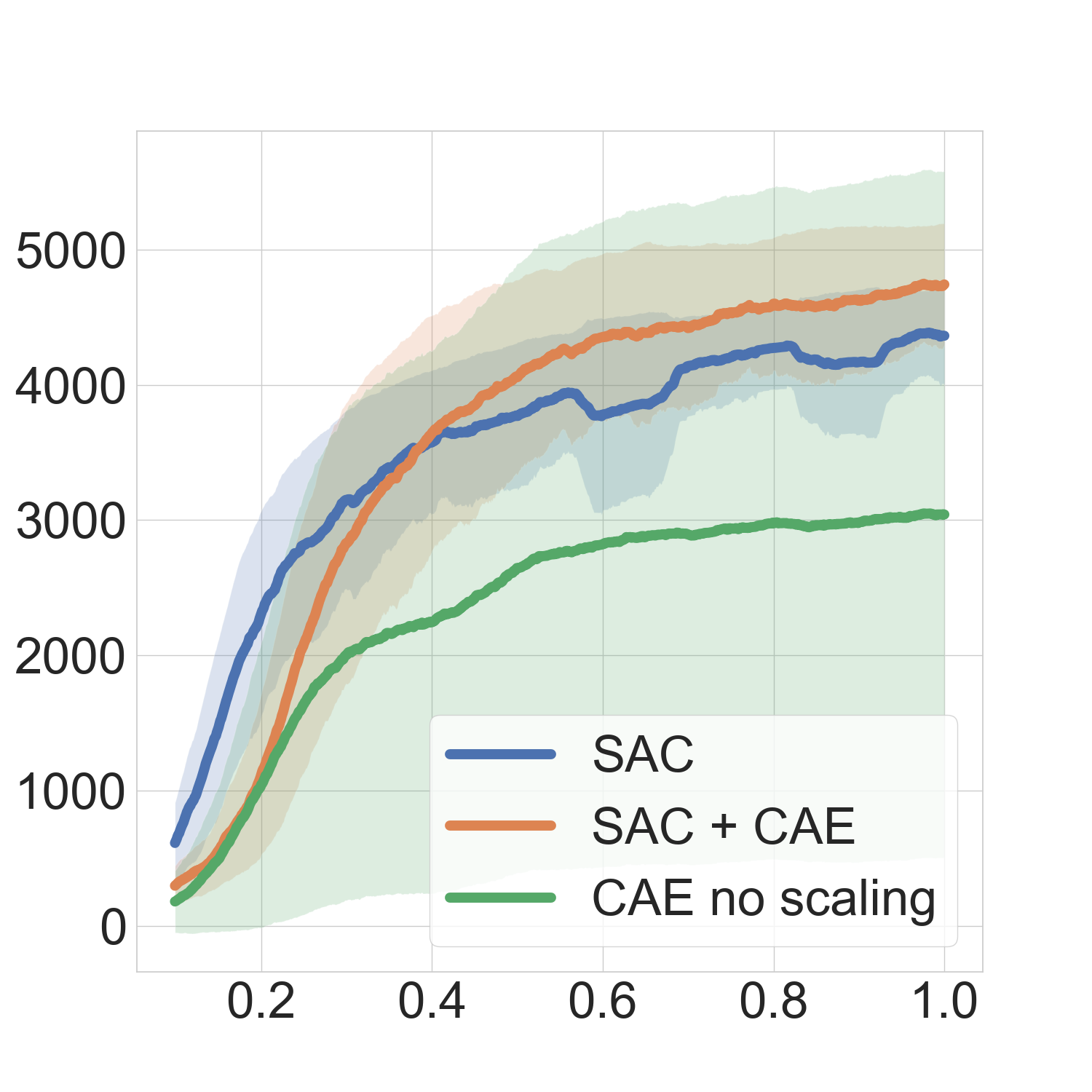}
    \caption{Walker2d}
    \label{fig:sac_walker2d_ablation}
  \end{subfigure}
  \hfill
  \begin{subfigure}[t]{0.3\textwidth}
    \includegraphics[width=\textwidth, trim={0 0 2.5cm 1.5cm}, clip]{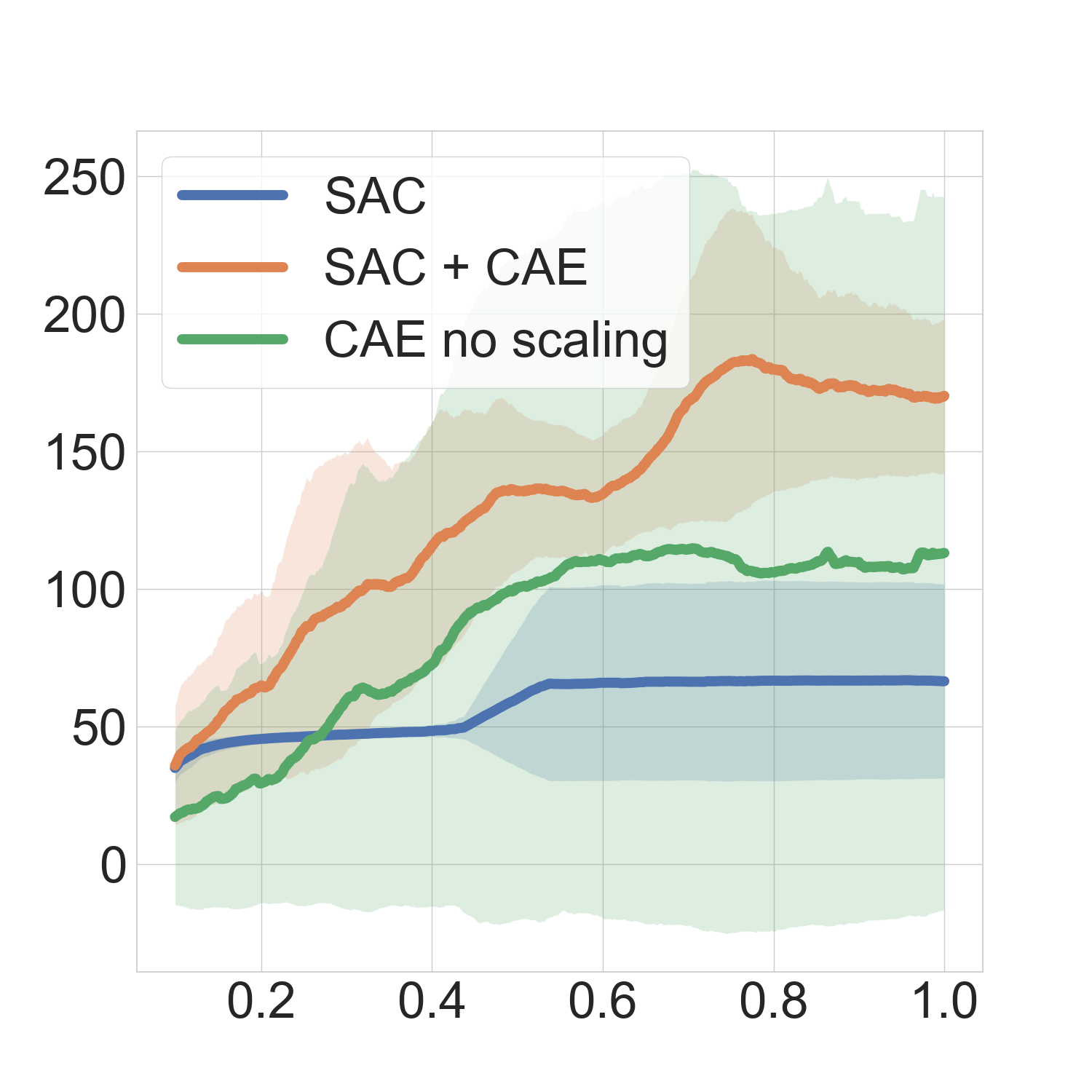}
    \caption{Swimmer}
    \label{fig:sac_swimmer_ablation}
  \end{subfigure}
  \caption{Experimental results of DSAC on MuJoCo-v3, \ie, from Figure~\ref{fig:swimmer_dsac} to Figure~\ref{fig:ant_dsac}, and those of ablation study to the \textit{scaling strategy} on MuJoCo-v4, \ie, Figure~\ref{fig:sac_walker2d_ablation} and Figure~\ref{fig:sac_swimmer_ablation}.}
  \label{fig:mujoco_v3}
\end{figure}

\subsection{Experiment on Mujoco}
\label{app:mujoco}

\textbf{Ablation study to the \textit{scaling strategy}}. We conduct experiments to assess the necessity of the scaling strategy. Experiment results, illustrated in Figure~\ref{fig:sac_walker2d_ablation} and Figure~\ref{fig:sac_swimmer_ablation}, are based on randomly selected tasks for SAC. As shown, SAC enhanced with \cae fails to deliver satisfactory performance on \textit{Walker2d} and \textit{Swimmer} tasks when the scaling strategy is not applied. It achieves high performance under certain seeds, while it performs poorly under others, leading to poor overall results and large variance. This highlights the critical role of the scaling strategy in ensuring the practical stability and effectiveness of \cae.

\renewcommand{\arraystretch}{1.2}
\begin{table}[t]
\caption{Exploration coefficient for various MuJoCo tasks and algorithms}
\label{tab:mujoco}
\vspace{5pt}

\begin{center}

\scalebox{0.8}{
\begin{tabular}{l|c|c|c|c}
    \toprule
    \diagbox{Env}{Algorithms} & SAC & PPO & TD3 & DSAC \\
    \midrule
    Swimmer & $0.2$ & $0.1$ & $0.1$ & $0.1$ \\
    \hline
    Ant & $0.7$ & $0.2$ & $0.3$ & $1.0$ \\
    \hline
    Walker2d & $1.0$ & $0.13$ & $0.8$ & $3.0$ \\
    \hline
    Hopper & $0.4$ & $0.14$ & $0.3$ & $2.0$ \\
    \hline
    HalfCheetah & $0.4$ & $0.5$ & $3.7$ & $3.0$ \\
    \hline
    Humanoid & $4.0$ & $0.1$ & $6.0$ & $4.5$ \\
    \bottomrule
\end{tabular}
}

\vspace{-5pt}
\end{center}
\end{table}
\renewcommand{\arraystretch}{1}

Hyperparameters of various RL algorithms for the experiments on MuJoCo are completely the same as those in the public codebase CleanRL. \cae introduces only two more hyperparameters, \ie, the exploration coefficient $\alpha$ and the ridge which is set as $\lambda=1$. Exploration coefficients are summarized in Table~\ref{tab:mujoco} for various tasks and algorithms, and the experimental results on various MuJoCo tasks involving different RL algorithms are in Figure~\ref{fig:mujoco_v3} and Figure~\ref{fig:mujoco_v4}. Notably, we only present results for cases where the \textit{RPI} exceeds $5.0\%$, as smaller \textit{RPI} is not clearly distinguishable in the figures.

\subsection{Experiment on MiniHack}
\label{app:minihack}

In Subsection~\ref{subsec:minihack}, we evaluate our methods \cae and \caep on nine representative MiniHack tasks, whose detailed descriptions are provided in Table~\ref{tab:minihack_task}.

\renewcommand{\arraystretch}{1.2}
\begin{table}[H]
\caption{MiniHack tasks evaluated in Subsection~\ref{subsec:minihack}.}
\label{tab:minihack_task}
\vspace{5pt}

\begin{center}

\scalebox{0.9}{
\begin{tabular}{l|c|c}
    \toprule
    Short names in Subsection~\ref{subsec:minihack}  & Full names & Task types \\
    \midrule
    N4  & MultiRoom-N4 & \multirow{5}{*}{Navigation-based} \\
    \cline{1-2}
    N4-Locked & MultiRoom-N4-Locked & \\
    \cline{1-2}
    N6 & MultiRoom-N6 & \\
    \cline{1-2}
    N6-Locked & MultiRoom-N6-Locked &  \\
    \cline{1-2}
    N10-OD & MultiRoom-N10-OpenDoor &  \\
    \hline
    Horn & Freeze-Horn-Restricted & \multirow{4}{*}{Skill-based} \\
    \cline{1-2}
    Random & Freeze-Random-Restricted &  \\
    \cline{1-2}
    Wand & Freeze-Wand-Restricted &  \\
    \cline{1-2}
    LavaCross & LavaCross-Restricted &  \\
    \bottomrule
\end{tabular}
}

\vspace{-5pt}
\end{center}
\end{table}
\renewcommand{\arraystretch}{1}

The hyperparameters for IMPALA, E3B, \cae, and \caep used in our experiments are summarized in Table~\ref{tab:impala} and Table~\ref{tab:e3b_cae}, aligning with those from the E3B experiments~\citep{E3B}. The experimental results on MiniHack are presented in Figure~\ref{fig:minihack}. 

\renewcommand{\arraystretch}{1.2}
\begin{table}[H]
\caption{IMPALA Hyperparameters for MiniHack~\citep{E3B}}
\label{tab:impala}
\vspace{5pt}

\begin{center}

\scalebox{0.9}{
\begin{tabular}{l|c}
    \toprule
    Learning rate  & $0.0001$ \\
    \hline
    RMSProp smoothing constant & $0.99$ \\
    \hline
    RMSProp momentum & $0$ \\
    \hline
    RMSProp & $10^{-5}$ \\
    \hline
    Unroll length & $80$ \\
    \hline
    Number of buffers & $80$ \\
    \hline
    Number of learner threads & $4$ \\
    \hline
    Number of actor threads & $8$ \\
    \hline
    Max gradient norm & $40$ \\
    \hline
    Entropy cost & $0.0005$ \\
    \hline
    Baseline cost & $0.5$ \\
    \hline
    Discounting factor & $0.99$ \\
    \bottomrule
\end{tabular}
}

\vspace{-5pt}
\end{center}
\end{table}
\renewcommand{\arraystretch}{1}

\renewcommand{\arraystretch}{1.2}
\begin{table}[H]
\caption{E3B, \cae, and \caep Hyperparameters for MiniHack}
\vspace{5pt}

\begin{center}

\scalebox{0.9}{
\begin{tabular}{c|c|c}
    \toprule
    \multirow{4}{*}{E3B, \cae, and \caep} & Scaling strategy & True \\
    \cline{2-3}
    & Ridge regularizer & $0.1$ \\
    \cline{2-3}
    & Entropy Cost & $0.005$ \\
    \cline{2-3}
    & Exploration coefficient & $1$ \\
    \hline
    \caep & Dimension of $\boldsymbol{U}$ & $128$ or $256$ \\
    \bottomrule
\end{tabular}
}

\vspace{-5pt}
\end{center}
\label{tab:e3b_cae}
\end{table}

\subsection{Experiment on Habitat}
\label{app:habitat}

The hyperparameters for PPO, E3B, \cae, and \caep used in these experiments are summarized in Table~\ref{tab:ppo_habitat} and Table~\ref{tab:e3b_cae_habitat}.

\renewcommand{\arraystretch}{1.2}
\begin{table}[H]
\caption{PPO Hyperparameters for Habitat are adopted from \textit{habitat-lab}~\citep{Habitat, Habitat_2, Habitat_3} 
}
\label{tab:ppo_habitat}
\vspace{5pt}

\begin{center}

\scalebox{0.9}{
\begin{tabular}{l|c}
    \toprule
    Clipping  & $0.2$ \\
    \hline
    PPO epoch  & $4$ \\
    \hline
    Value loss coefficient & $0.5$ \\
    \hline
    Entropy coefficient & $0.01$ \\
    \hline
    Learning rate  & $2.5e-4$ \\
    \hline
    $\epsilon-$greedy & 1e-5 \\
    \hline
    Max gradient norm & $0.2$ \\
    \hline
    Rollout steps & $128$ \\
    \hline
    Use GAE & True \\
    \hline
    Discounting factor & $0.99$ \\
    \hline
    Number of actor threads & $16$ \\
    \bottomrule
\end{tabular}
}

\vspace{-5pt}
\end{center}
\end{table}
\renewcommand{\arraystretch}{1}

\renewcommand{\arraystretch}{1.2}
\begin{table}[H]
\caption{E3B, \cae, and \caep Hyperparameters for Habitat}
\vspace{5pt}

\begin{center}

\scalebox{0.9}{
\begin{tabular}{c|c|c}
    \toprule
    \multirow{4}{*}{E3B, \cae, and \caep} & Scaling strategy & False \\
    \cline{2-3}
    & Ridge regularizer & $0.1$ \\
    \cline{2-3}
    & Inverse Dynamics Model updates per PPO epoch & $3$ \\
    \cline{2-3}
    & Exploration coefficient & $0.1$ \\
    \hline
    \caep & Dimension of $\boldsymbol{U}$ & $256$ \\
    \bottomrule
\end{tabular}
}

\vspace{-5pt}
\end{center}
\label{tab:e3b_cae_habitat}
\end{table}

\section{Limitations}
\label{sec:limitation}
One limitation of the proposed method is that, despite the small number of additional trainable parameters, computing the uncertainty incurs a non-negligible computational overhead due to the need for matrix inversion. This can be mitigated by using approximation techniques, such as the \textbf{Rank-1} method described in Subsection~\ref{subsec:caep}. Additionally, other linear MAB methods could be integrated into the proposed framework to avoid the need to calculate the inverse Gram matrix.

\newpage
\begin{figure}[H]
  \centering
  \begin{subfigure}[t]{0.3\textwidth}
    \includegraphics[width=\textwidth, trim={0 0 2.5cm 1.5cm}, clip]{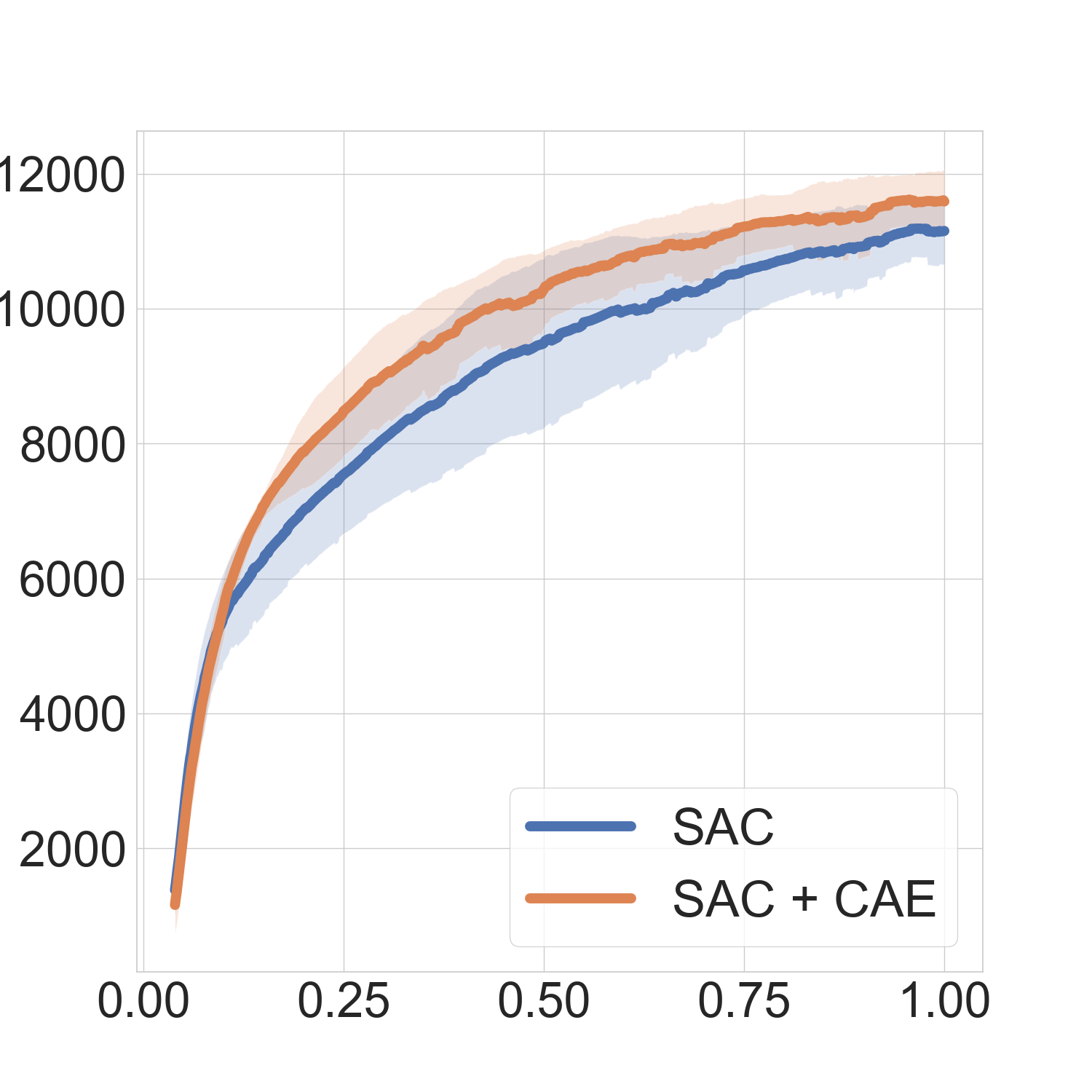}
    \caption{HalfCheetah}
    \label{fig:halfcheetah_sac}
  \end{subfigure}
  \hfill
  \begin{subfigure}[t]{0.3\textwidth}
    \includegraphics[width=\textwidth, trim={0 0 2.5cm 1.5cm}, clip]{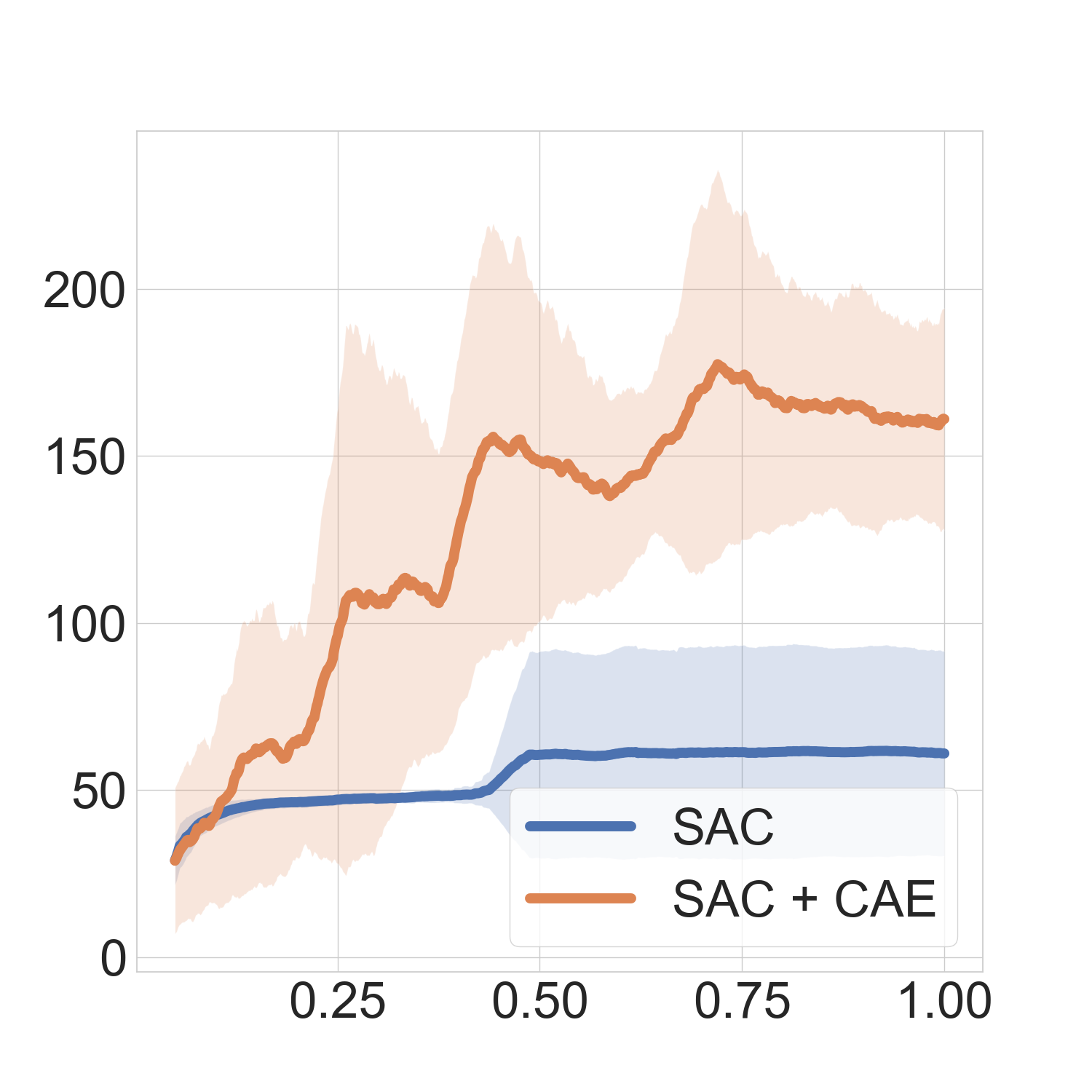}
    \caption{Swimmer}
    \label{fig:swimmer_sac}
  \end{subfigure}
  \hfill
  \begin{subfigure}[t]{0.3\textwidth}
    \includegraphics[width=\textwidth, trim={0 0 2.5cm 1.5cm}, clip]{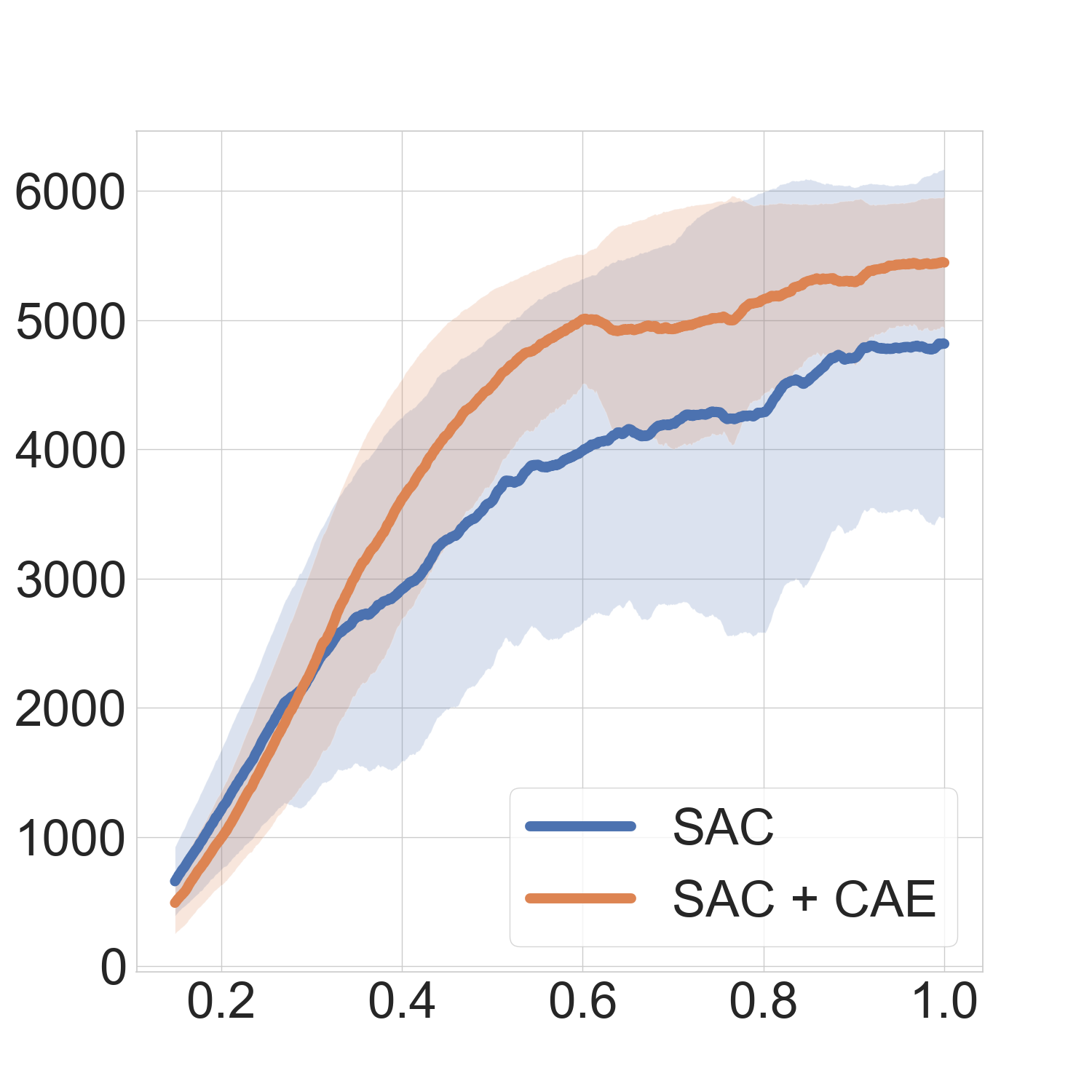}
    \caption{Ant}
    \label{fig:ant_sac}
  \end{subfigure}

  \vspace{-1mm}

  \begin{subfigure}[t]{0.3\textwidth}
    \includegraphics[width=\textwidth, trim={0 0 2.5cm 1.5cm}, clip]{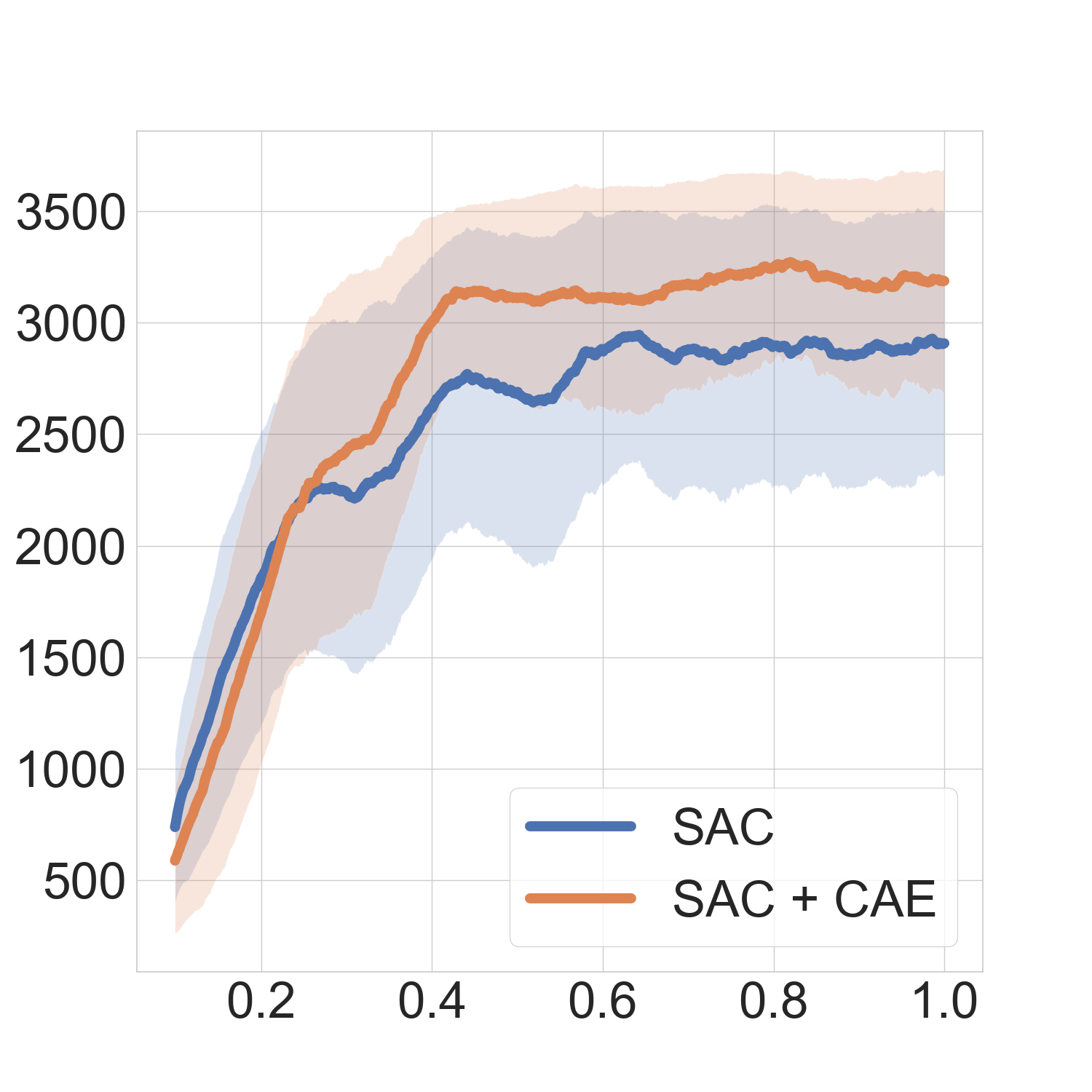}
    \caption{Hopper}
    \label{fig:hopper_sac}
  \end{subfigure}
  \hfill
  \begin{subfigure}[t]{0.3\textwidth}
    \includegraphics[width=\textwidth, trim={0 0 2.5cm 1.5cm}, clip]{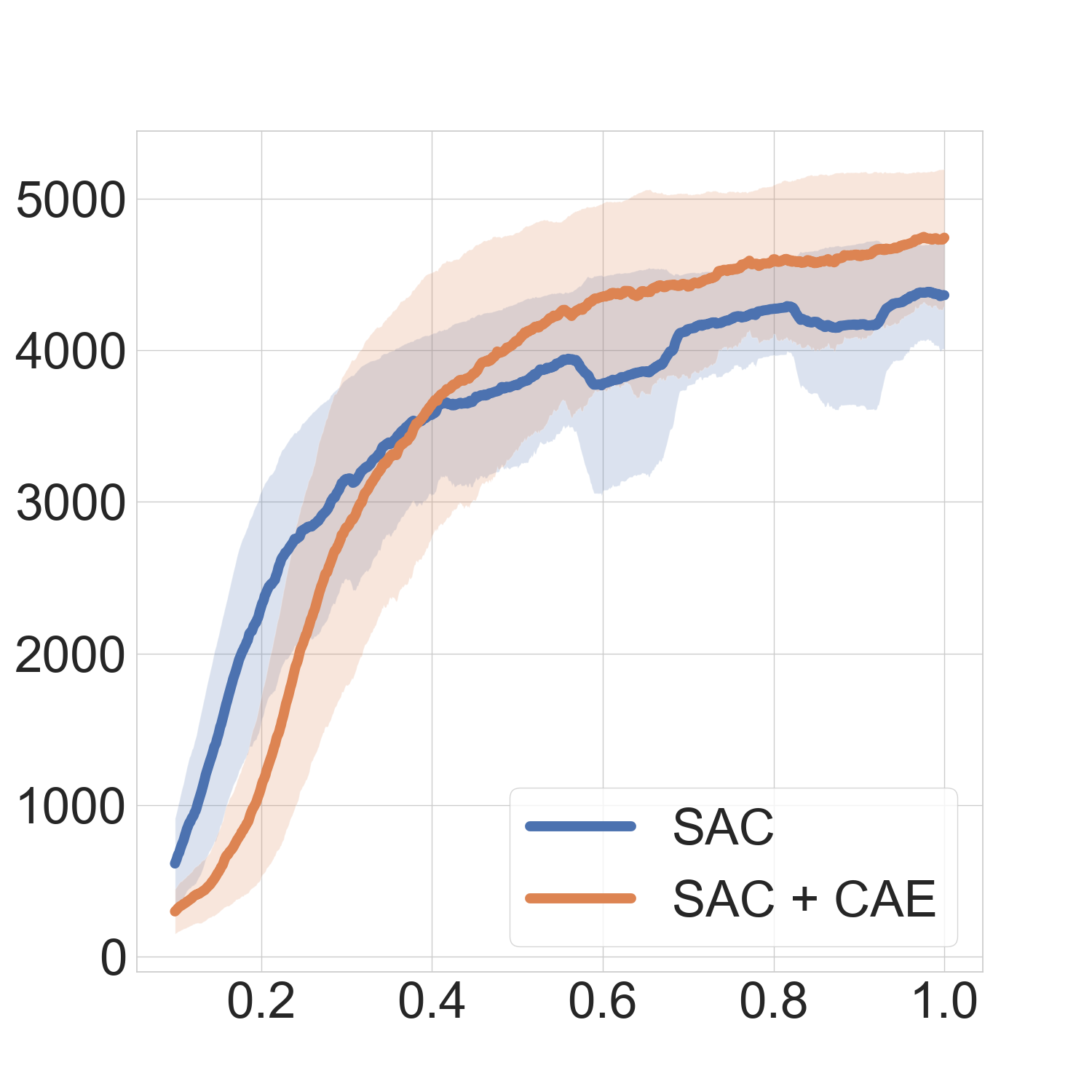}
    \caption{Walker2d}
    \label{fig:walker_sac}
  \end{subfigure}
  \hfill
  \begin{subfigure}[t]{0.3\textwidth}
    \includegraphics[width=\textwidth, trim={0 0 2.5cm 1.5cm}, clip]{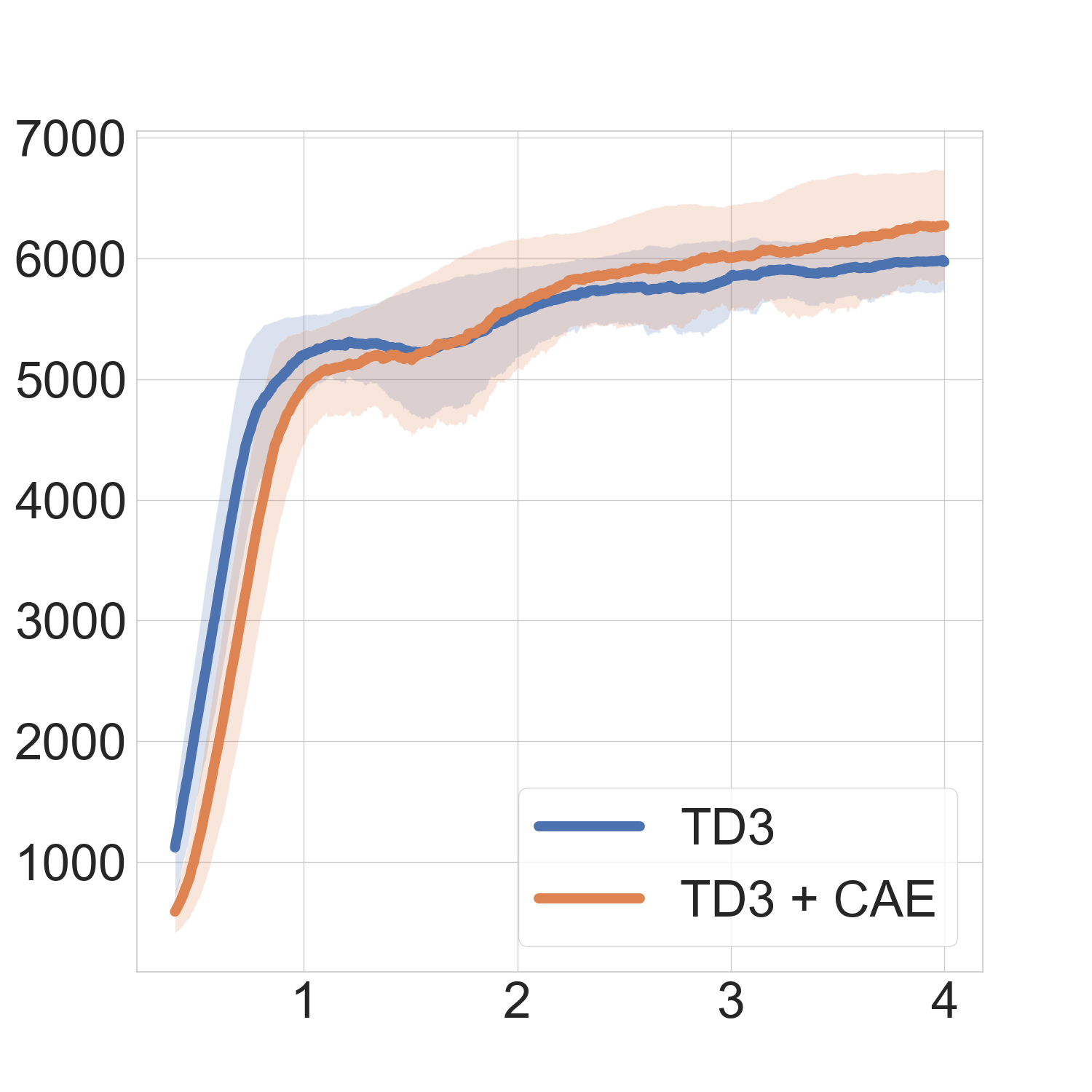}
    \caption{Humanoid}
    \label{fig:humanoid_td3}
  \end{subfigure}

  \vspace{-1mm}

  \begin{subfigure}[t]{0.3\textwidth}
    \includegraphics[width=\textwidth, trim={0 0 2.5cm 1.5cm}, clip]{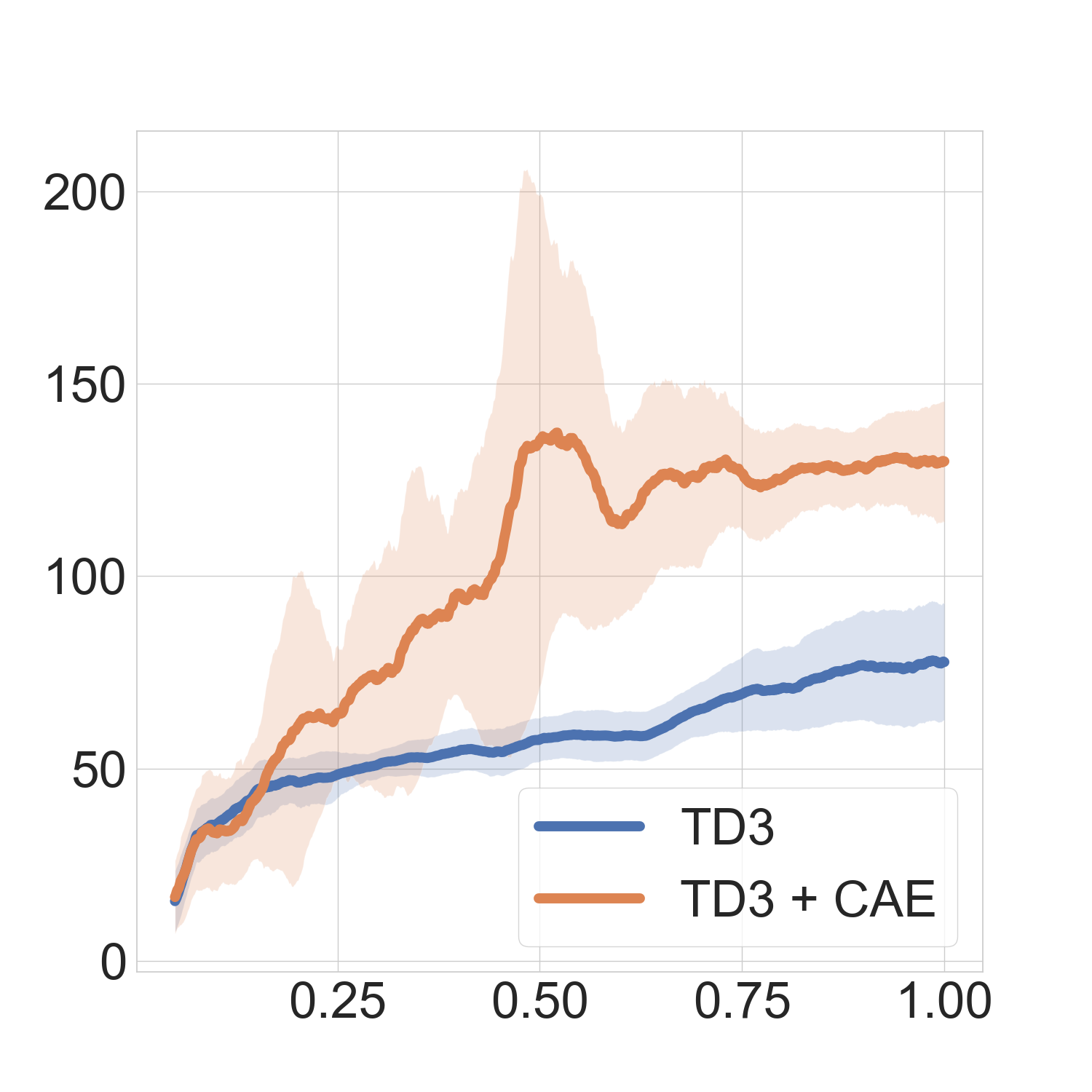}
    \caption{Swimmer}
    \label{fig:swimmer_td3}
  \end{subfigure}
  \hfill
  \begin{subfigure}[t]{0.3\textwidth}
    \includegraphics[width=\textwidth, trim={0 0 2.5cm 1.5cm}, clip]{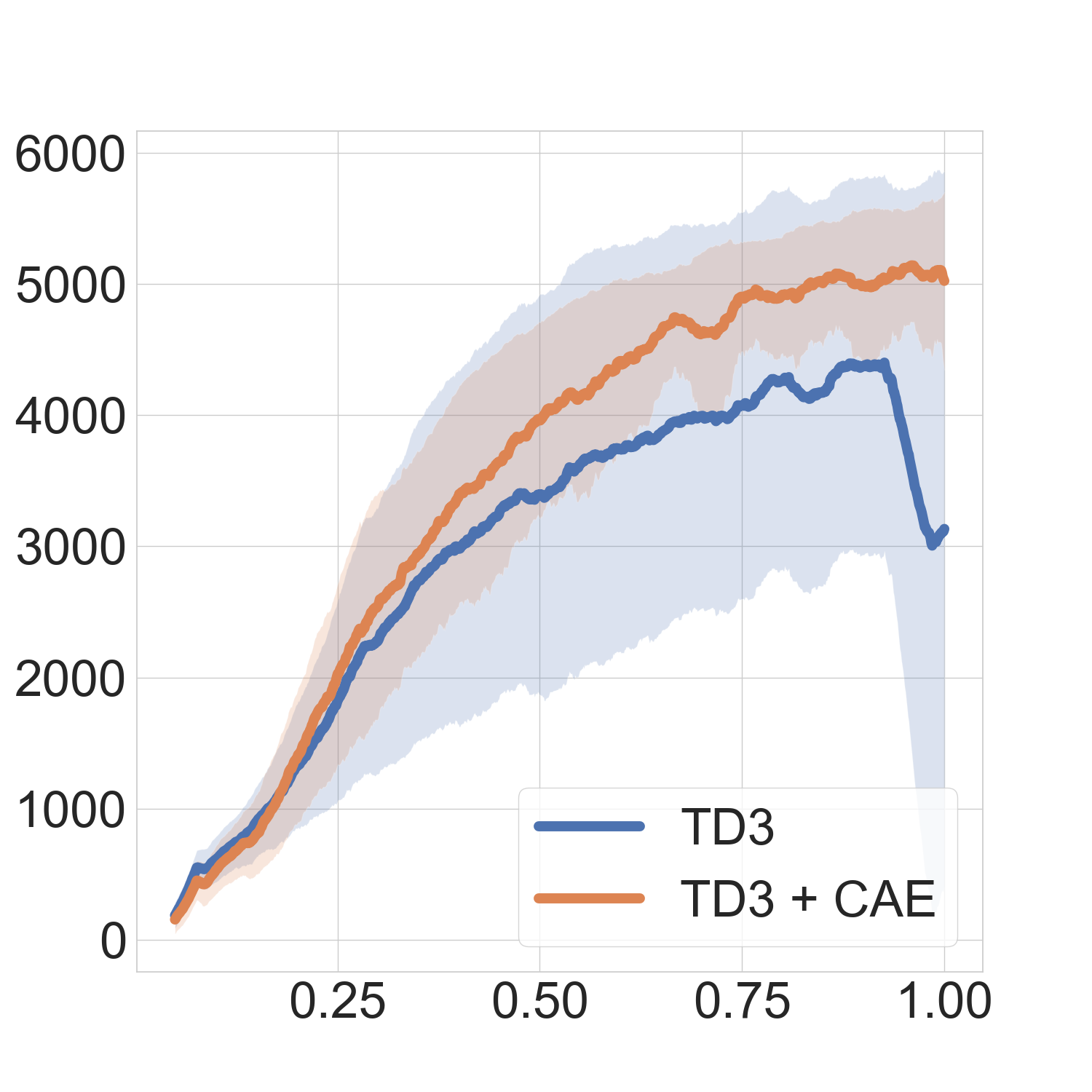}
    \caption{Ant}
    \label{fig:ant_td3}
  \end{subfigure}
  \hfill
  \begin{subfigure}[t]{0.3\textwidth}
    \includegraphics[width=\textwidth, trim={0 0 2.5cm 1.5cm}, clip]{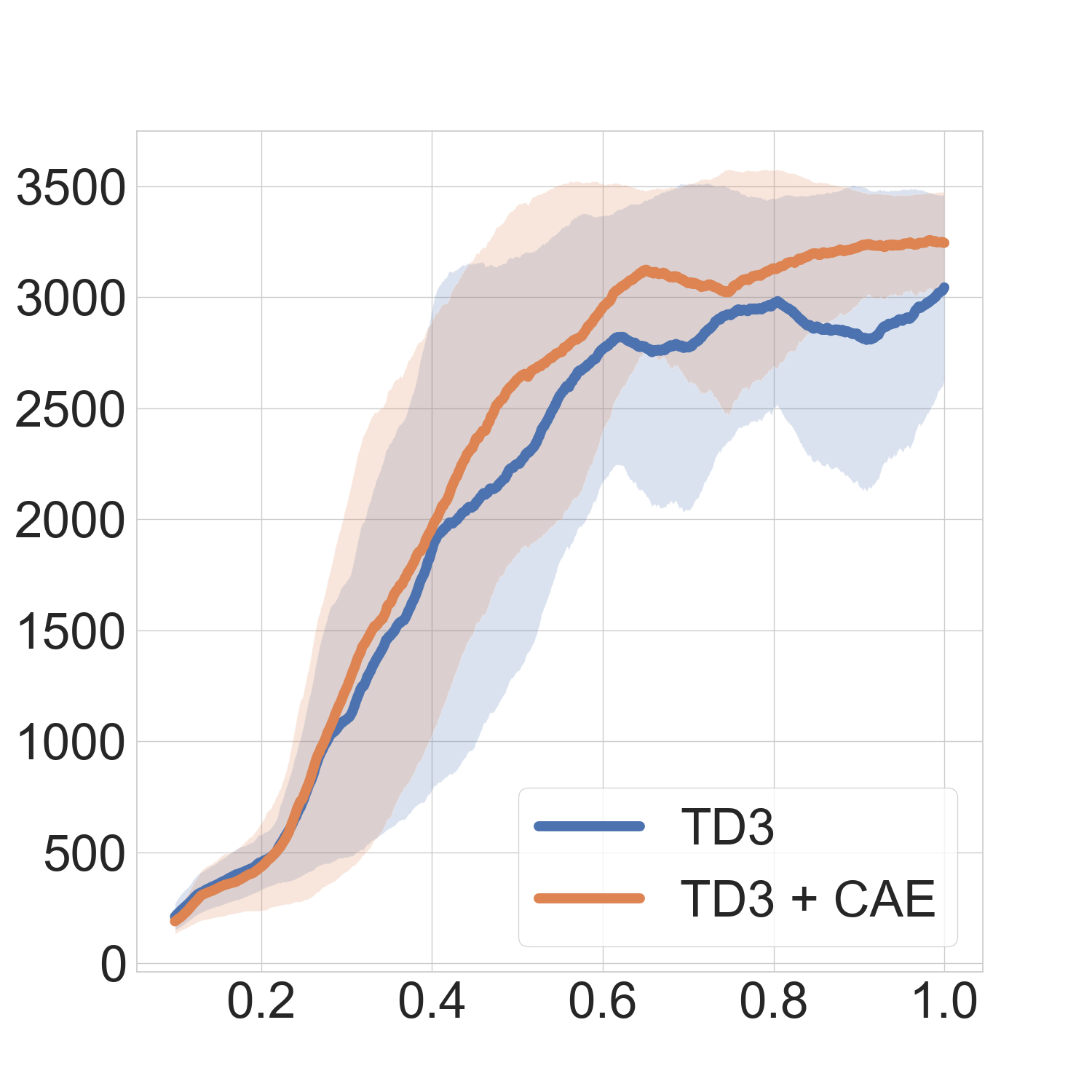}
    \caption{Hopper}
    \label{fig:hopper_td3}
  \end{subfigure}

  \vspace{-1mm}

  \begin{subfigure}[t]{0.3\textwidth}
    \includegraphics[width=\textwidth, trim={0 0 2.5cm 1.5cm}, clip]{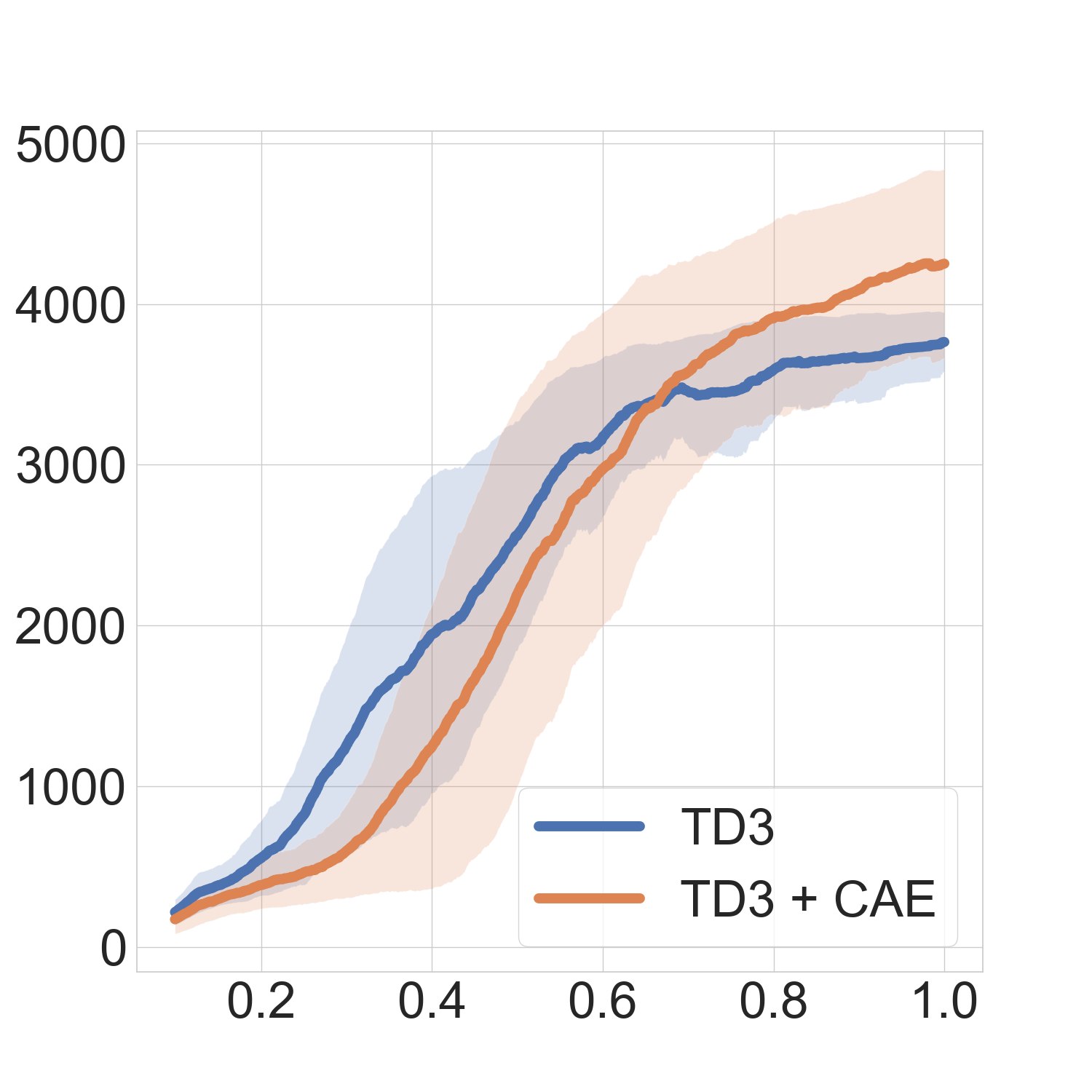}
    \caption{Walker2d}
    \label{fig:walker2d_td3}
  \end{subfigure}
  \hfill
  \begin{subfigure}[t]{0.3\textwidth}
    \includegraphics[width=\textwidth, trim={0 0 2.5cm 1.5cm}, clip]{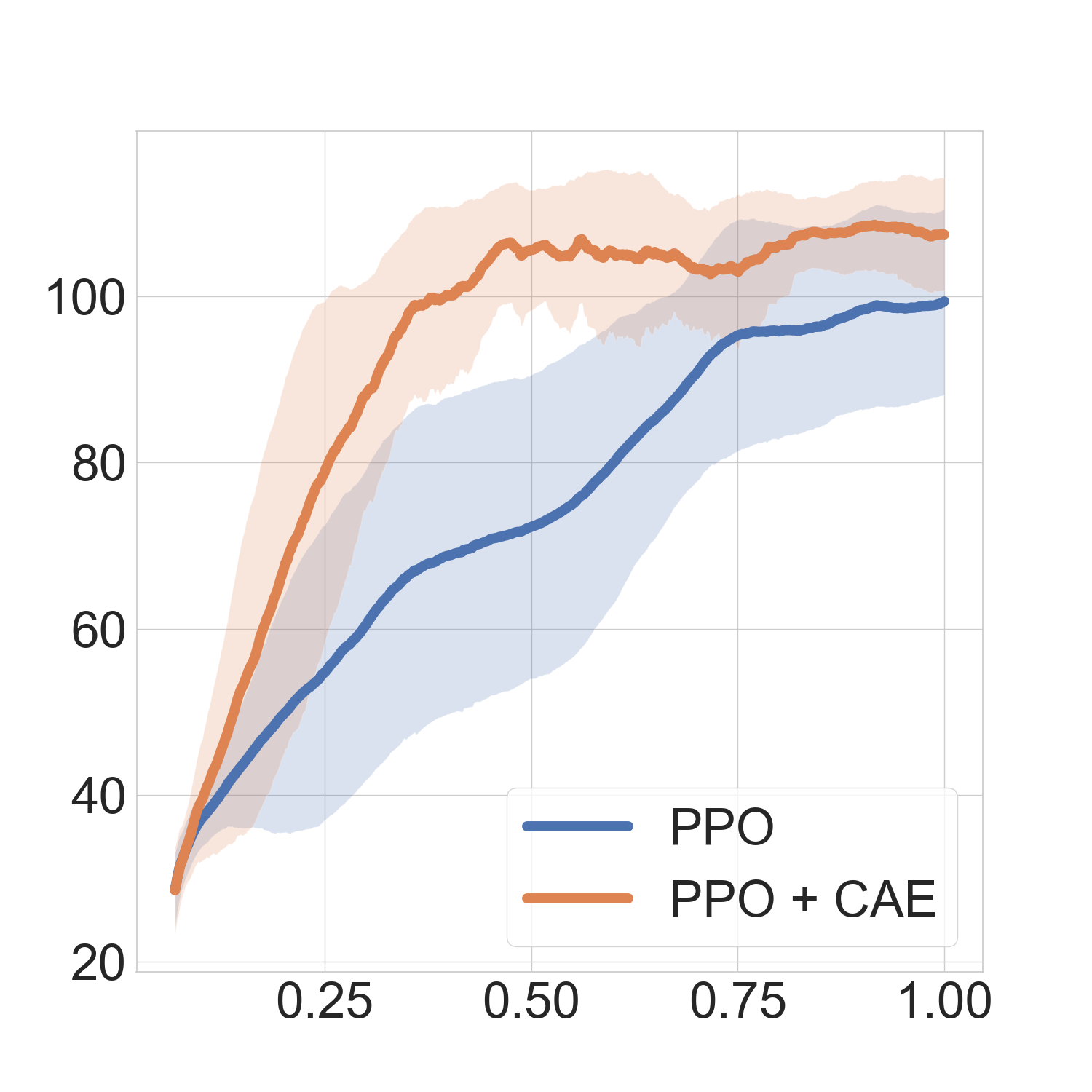}
    \caption{Swimmer}
    \label{fig:swimmer_ppo}
  \end{subfigure}
  \hfill
  \begin{subfigure}[t]{0.3\textwidth}
    \includegraphics[width=\textwidth, trim={0 0 2.5cm 1.5cm}, clip]{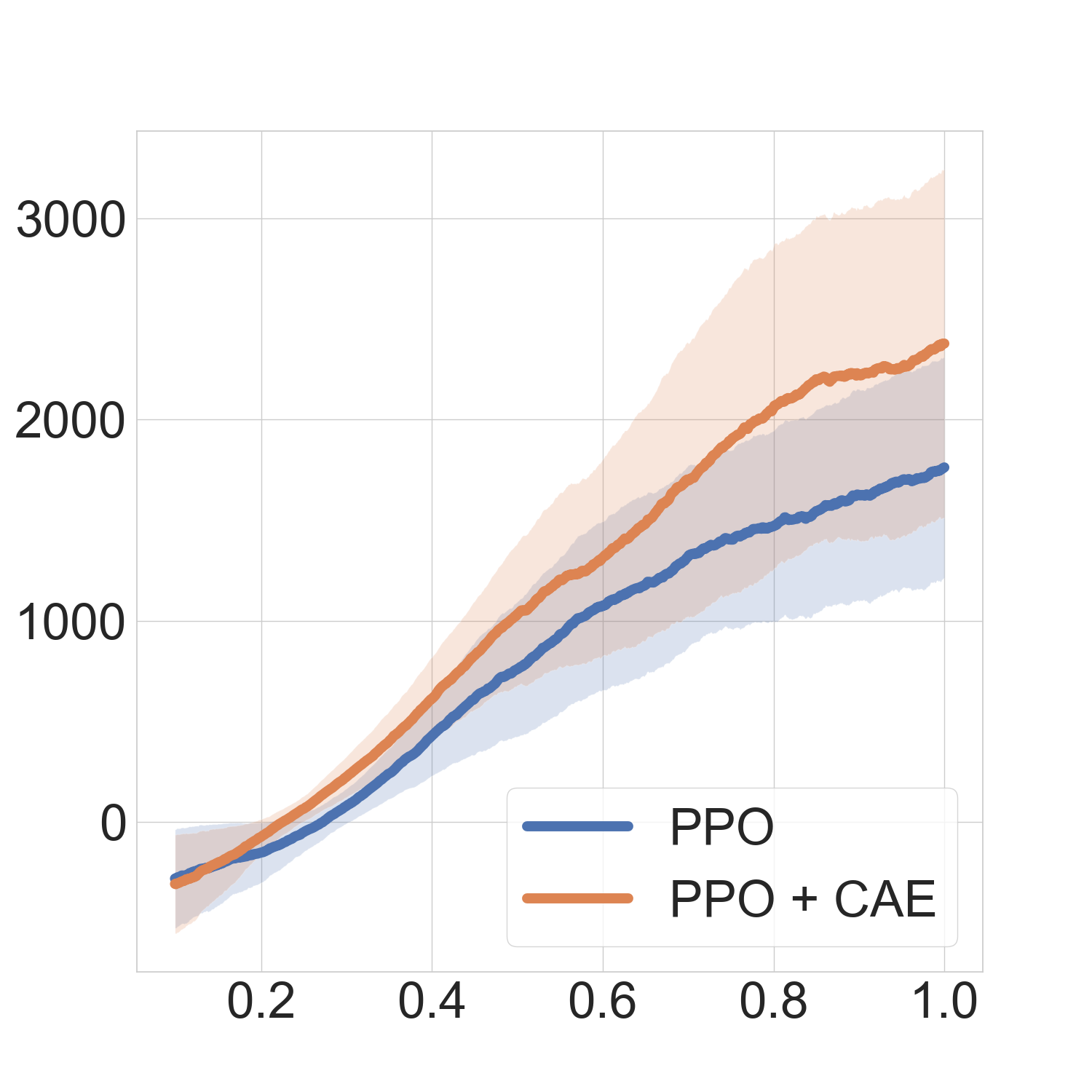}
    \caption{Ant}
    \label{fig:ant_ppo}
  \end{subfigure}

  \caption{Experimental results on MuJoCo-v4.}
  \label{fig:mujoco_v4}
\end{figure}

\newpage
\begin{figure}[H]
  \centering
  \begin{subfigure}[t]{0.3\textwidth}
    \includegraphics[width=\textwidth, trim={0 0 2.5cm 1.5cm}, clip]{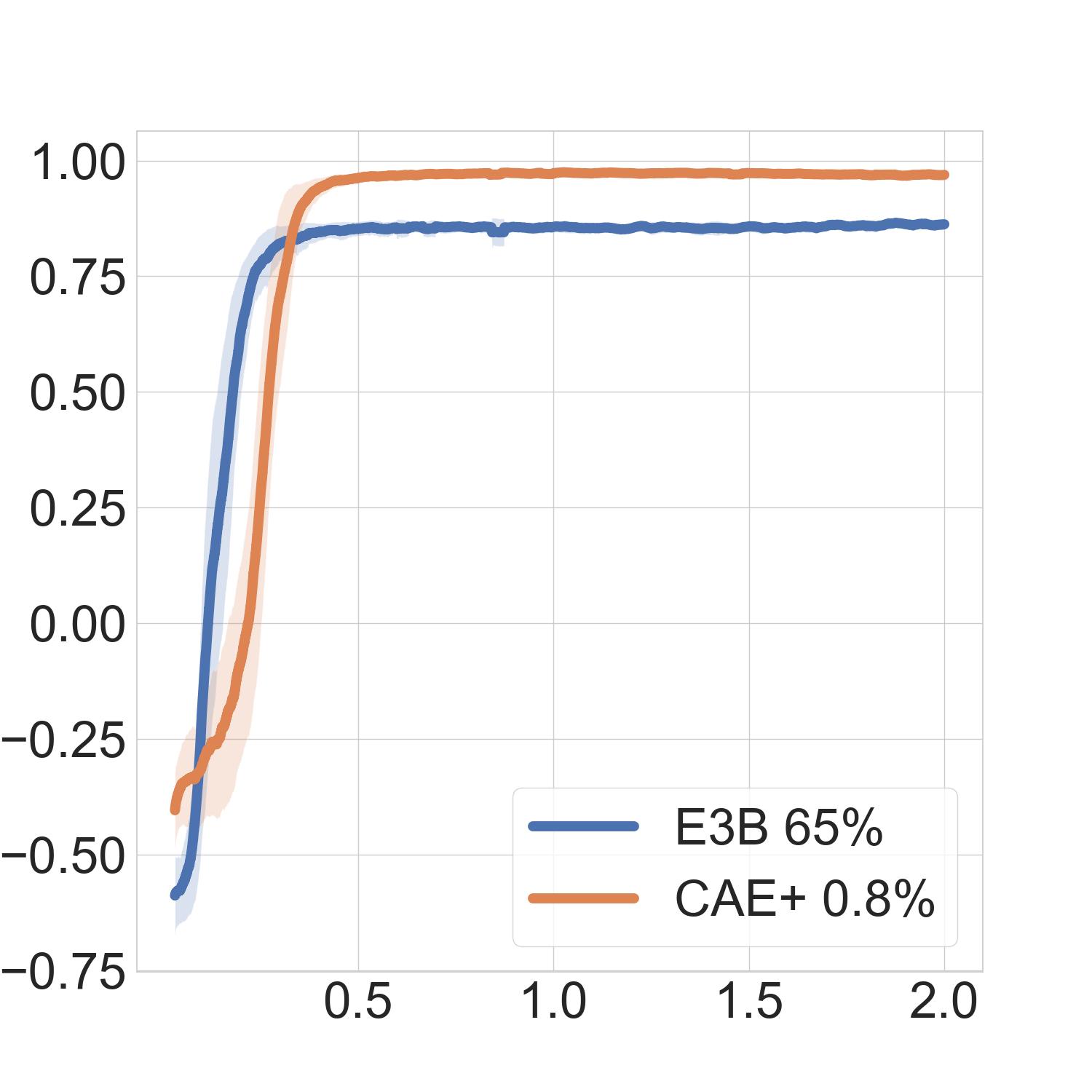}
    \caption{MultiRoom-N4-v0}
    \label{fig:multiroom-n4}
  \end{subfigure}
  \hfill
  \begin{subfigure}[t]{0.3\textwidth}
    \includegraphics[width=\textwidth, trim={0 0 2.5cm 1.5cm}, clip]{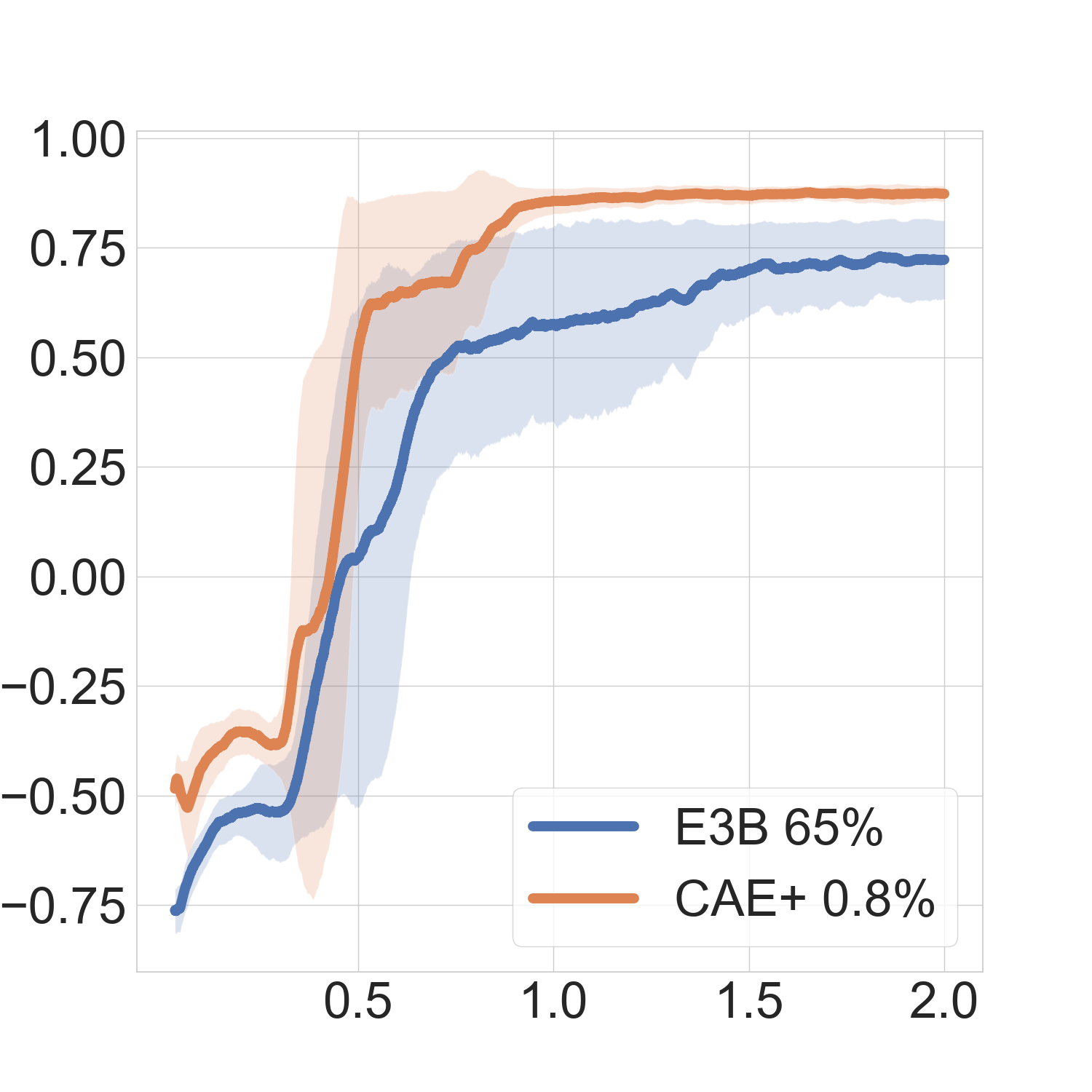}
    \caption{MultiRoom-N4-Locked-v0}
    \label{fig:multiroom-n4-locked}
  \end{subfigure}
  \hfill
  \begin{subfigure}[t]{0.3\textwidth}
    \includegraphics[width=\textwidth, trim={0 0 2.5cm 1.5cm}, clip]{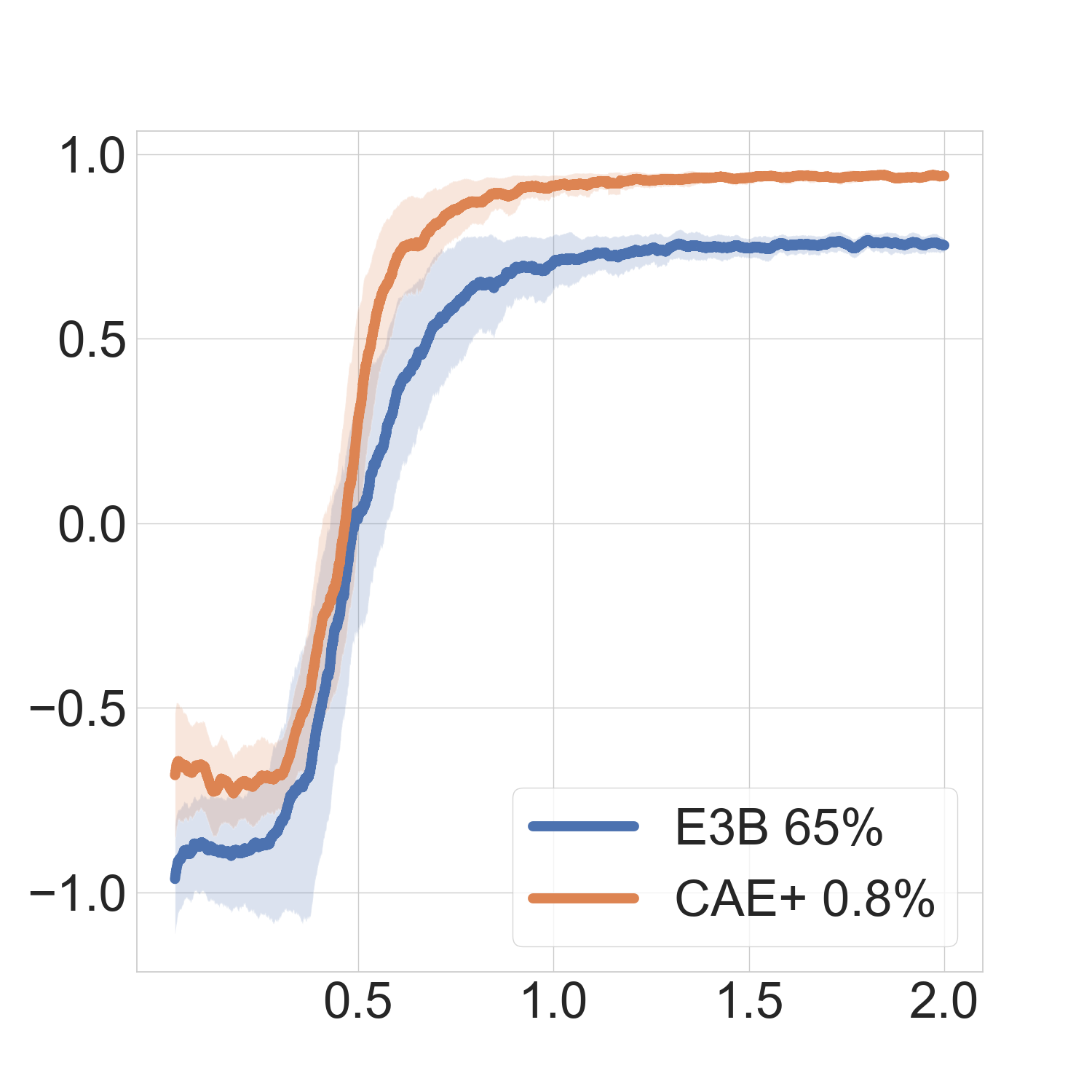}
    \caption{MultiRoom-N6-v0}
    \label{fig:multiroom-n6}
  \end{subfigure}
  \vskip\baselineskip
  \begin{subfigure}[t]{0.3\textwidth}
    \includegraphics[width=\textwidth, trim={0 0 2.5cm 1.5cm}, clip]{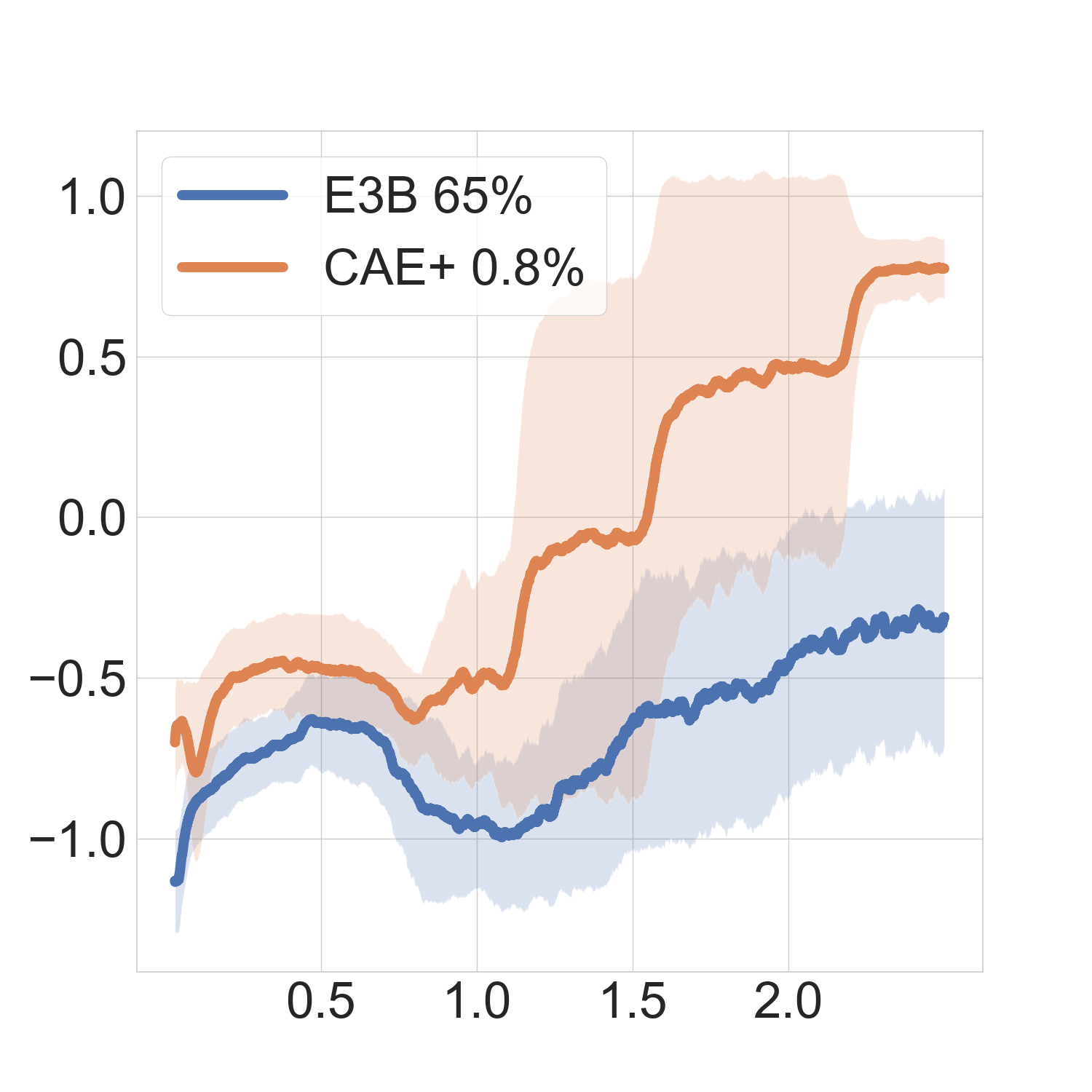}
    \caption{MultiRoom-N6-Locked-v0}
    \label{fig:multiroom-n6-locked_}
  \end{subfigure}
  \hfill
  \begin{subfigure}[t]{0.3\textwidth}
    \includegraphics[width=\textwidth, trim={0 0 2.5cm 1.5cm}, clip]{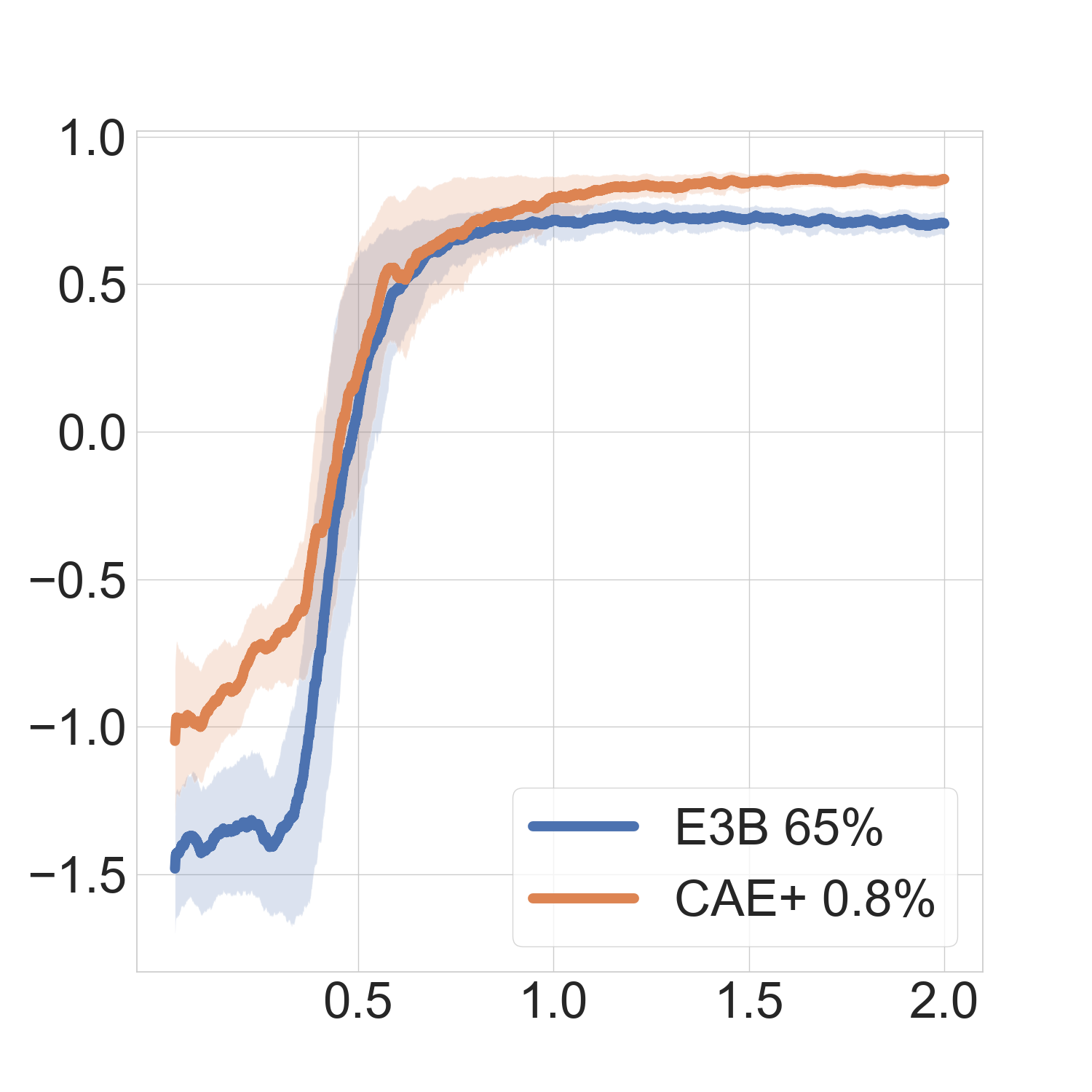}
    \caption{MultiRoom-N10-OpenDoor}
    \label{fig:multiroom-n10-opendoor}
  \end{subfigure}
  \hfill
  \begin{subfigure}[t]{0.3\textwidth}
    \includegraphics[width=\textwidth, trim={0 0 2.5cm 1.5cm}, clip]{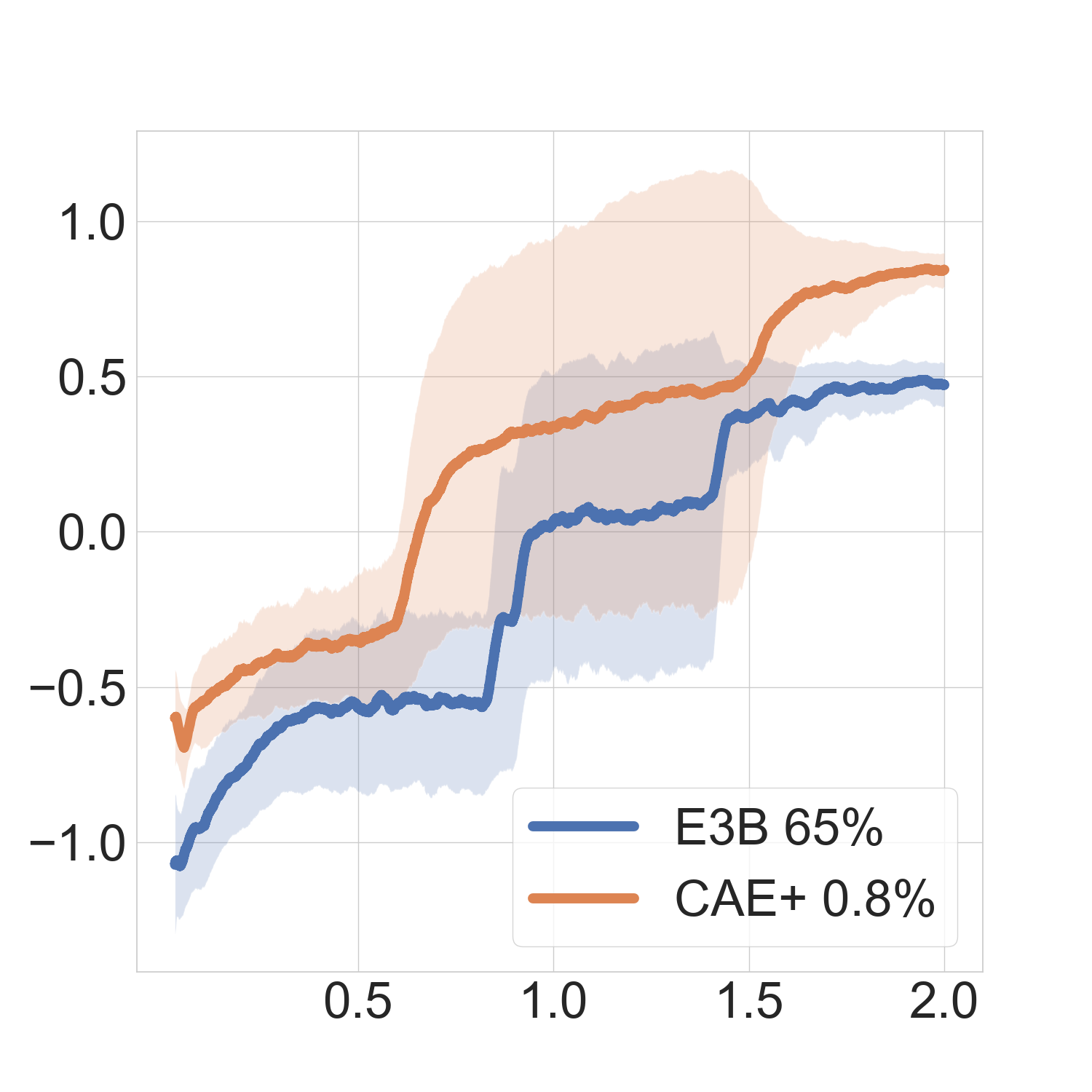}
    \caption{Freeze-Horn-Restricted-v0}
    \label{fig:freeze-horn-restricted}
  \end{subfigure}
  \vskip\baselineskip
  \begin{subfigure}[t]{0.3\textwidth}
    \includegraphics[width=\textwidth, trim={0 0 2.5cm 1.5cm}, clip]{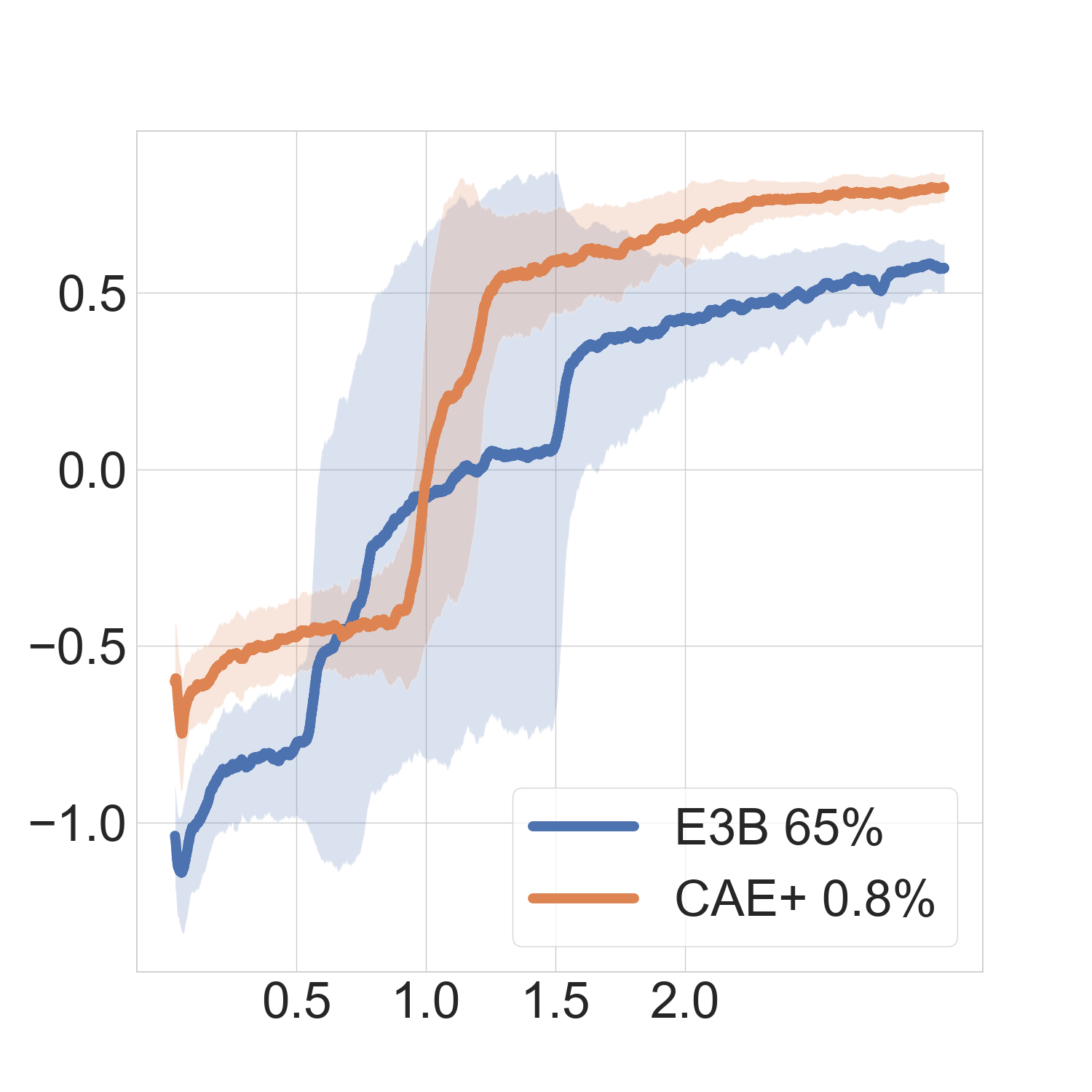}
    \caption{Freeze-Random-Restricted-v0}
    \label{fig:freeze-random-restricted}
  \end{subfigure}
  \hfill
  \begin{subfigure}[t]{0.3\textwidth}
    \includegraphics[width=\textwidth, trim={0 0 2.5cm 1.5cm}, clip]{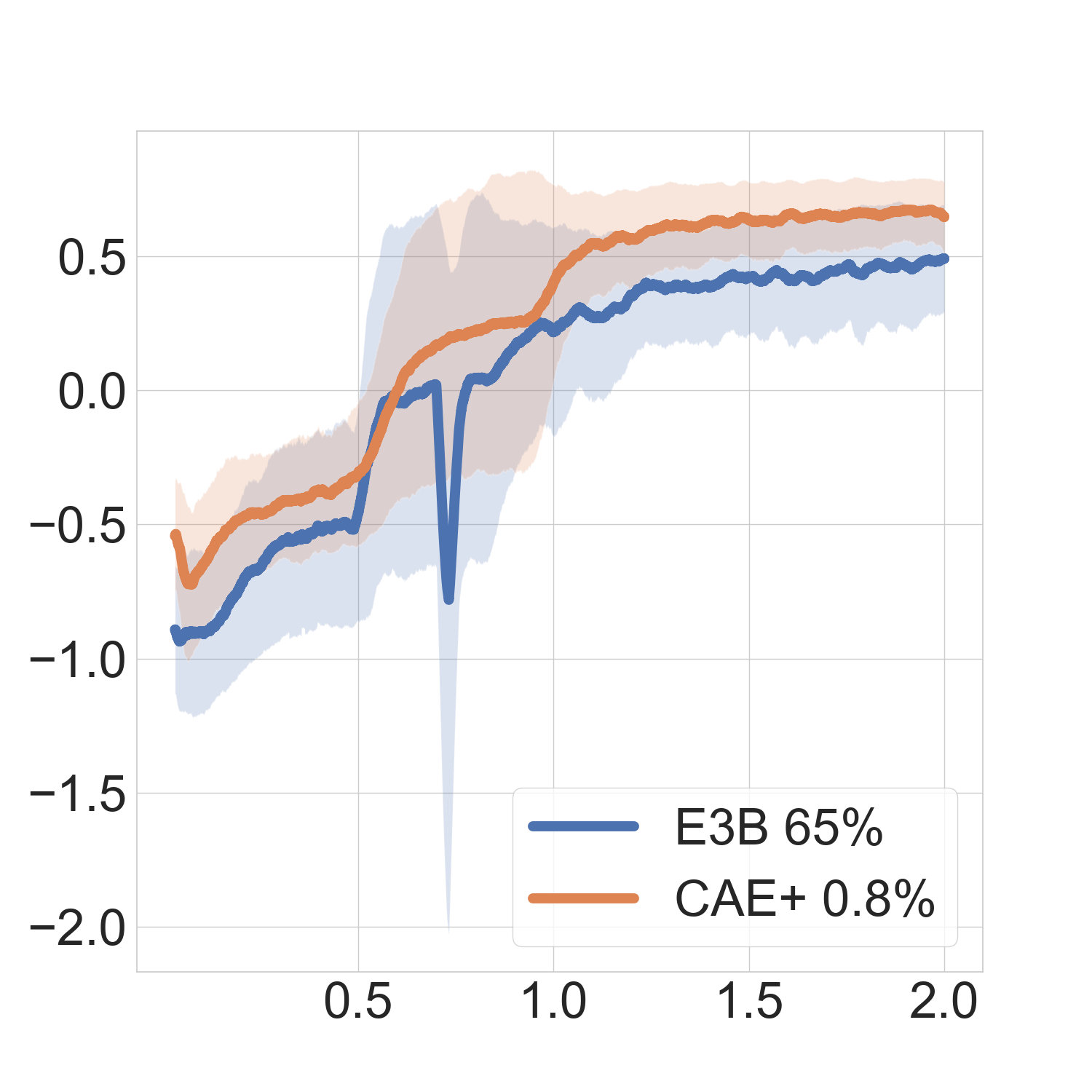}
    \caption{Freeze-Wand-Restricted-v0}
    \label{fig:freeze-wand-restricted}
  \end{subfigure}
  \hfill
  \begin{subfigure}[t]{0.3\textwidth}
    \includegraphics[width=\textwidth, trim={0 0 2.5cm 1.5cm}, clip]{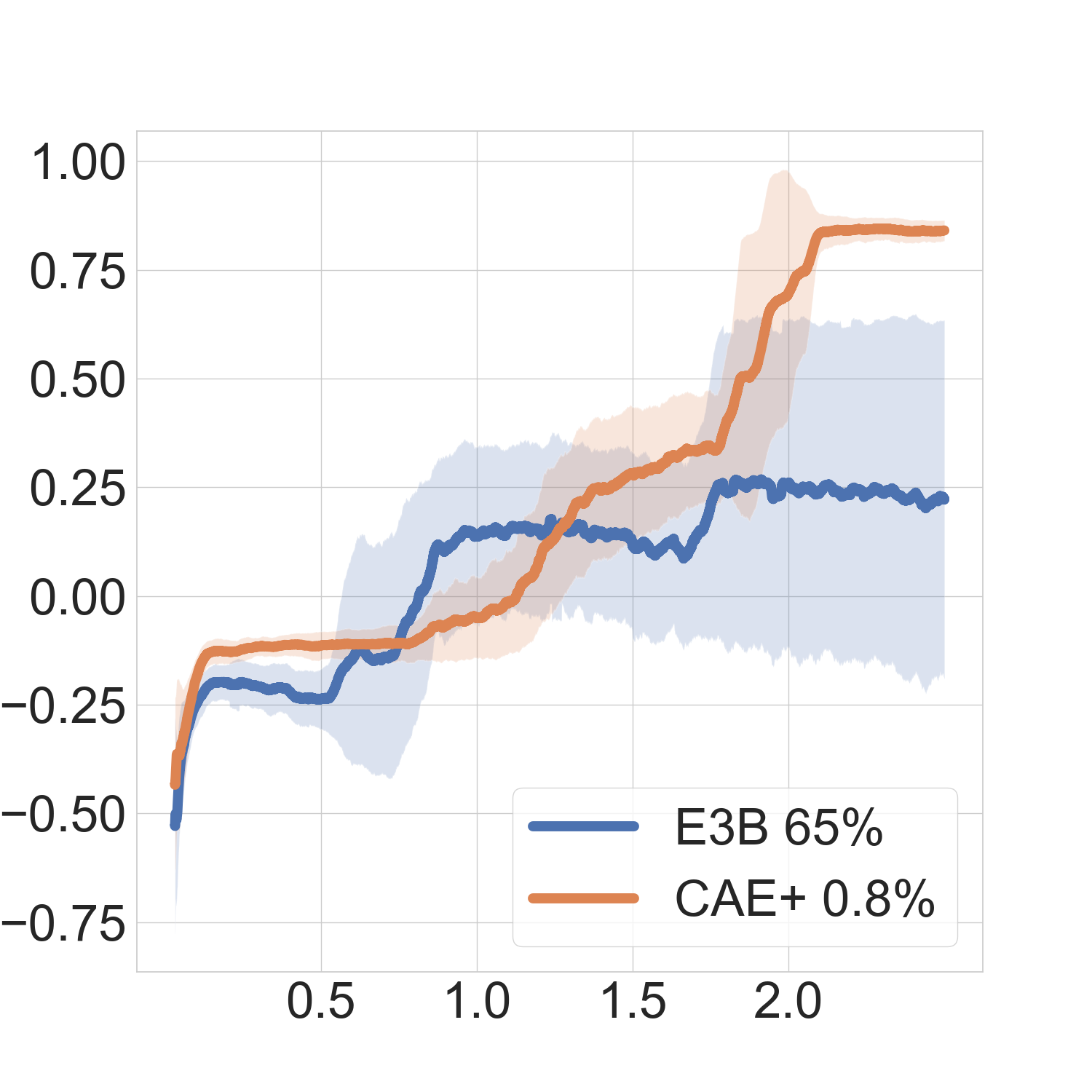}
    \caption{LavaCross-Restricted}
    \label{fig:lavacross-restricted}
  \end{subfigure}
  \caption{Experimental results on MiniHack.}
  \label{fig:minihack}
\end{figure}

\end{document}